\documentclass[nohyperref]{article}

\usepackage{microtype}
\usepackage{graphicx}
\usepackage{subfigure}
\usepackage{booktabs} 

\usepackage{hyperref}

\usepackage[accepted]{icml2022}

\usepackage{amsmath}
\usepackage{amssymb}
\usepackage{mathtools}
\usepackage{amsthm}

\usepackage[capitalize,noabbrev]{cleveref}

\theoremstyle{plain}
\newtheorem{theorem}{Theorem}[section]

\newtheorem{lemma}[theorem]{Lemma}

\theoremstyle{definition}

\theoremstyle{remark}

\usepackage[textsize=tiny]{todonotes}
\usepackage{amsfonts}
\usepackage{amssymb}
\usepackage{mathabx}
\usepackage{pdflscape}
\newcommand{\red}[1]{\textcolor{red}{\textbf{#1}}}
\newcommand{\blue}[1]{\textcolor{blue}{\underline{#1}}}
\newcommand{\bb}{\hspace{-1mm} $\bullet$}

\icmltitlerunning{GNNRank: Learning Global Rankings from Pairwise Comparisons via Directed Graph Neural Networks}

\begin{document}

\twocolumn[
\icmltitle{GNNRank: Learning Global Rankings from Pairwise Comparisons via \\ Directed Graph Neural Networks}




\icmlsetsymbol{internAWS}{*}
\begin{icmlauthorlist}
\icmlauthor{Yixuan He}{UnivOx,internAWS}
\icmlauthor{Quan Gan}{Amazon}
\icmlauthor{David Wipf}{Amazon}
\icmlauthor{Gesine Reinert}{UnivOx,turing}
\icmlauthor{Junchi Yan}{SJTU,AILab}
\icmlauthor{Mihai Cucuringu}{UnivOx,turing}
\end{icmlauthorlist}

\icmlaffiliation{UnivOx}{Department of Statistics, University of Oxford, Oxford, United Kingdom}
\icmlaffiliation{turing}{The Alan Turing Institute, London, United Kingdom}
\icmlaffiliation{Amazon}{Amazon Web Services AI Shanghai Lablet, Shanghai, China}
\icmlaffiliation{SJTU}{Department of Computer Science and Engineering and MoE Key Lab of Artificial Intelligence, Shanghai Jiao Tong University, Shanghai, China}
\icmlaffiliation{AILab}{Shanghai AI Laboratory, Shanghai, China}

\icmlcorrespondingauthor{Mihai Cucuringu}{mihai.cucuringu@stats.ox.ac.uk}

\icmlkeywords{Machine Learning, ICML}

\vskip 0.3in
]



\printAffiliationsAndNotice{\textsuperscript{*} Work was partially done during an internship at Amazon and a visit to Shanghai Jiao Tong University.}
\begin{abstract}
Recovering global rankings from pairwise comparisons has wide applications from time synchronization to sports team ranking. Pairwise comparisons corresponding to matches in a competition can be construed as edges in a directed graph (digraph), whose nodes represent e.g. competitors with an unknown rank. In this paper, we introduce neural networks into the ranking recovery problem by proposing the so-called GNNRank, a trainable GNN-based framework with digraph embedding. Morever, 
new objectives are devised to encode ranking upsets/violations. 
The framework involves a ranking score estimation approach, and adds an inductive bias by unfolding the Fiedler vector computation of the graph constructed from a learnable similarity matrix. Experimental results on extensive data sets show that our methods attain competitive and often superior performance against baselines, as well as showing promising transfer ability. Codes and preprocessed data are at: \url{https://github.com/SherylHYX/GNNRank}.
\end{abstract}

\section{Introduction}
Recovering global rankings from pairwise comparisons reflecting relative latent strengths or scores is a fundamental problem in information retrieval \cite{schutze2008introduction,liu2011learning} and beyond. 
When analyzing large-scale data sets, one often seeks various forms of rankings (i.e.~orderings) of the data for the purpose of  identifying the most important entries, efficient computation of search \& sort operations, or for extracting the main features. There is a swarm of applications employing ranking techniques ranging from 
Amazon's Mechanical Turk system for crowdsourcing
\cite{raykar2011ranking}
to the movie recommendation system provided by Netflix  \cite{cremonesi2010performance}, and 
modeling outcomes of football matches \cite{SoccerChallenge__Machine_Learning_2019}.

A very rich literature on ranking traces back to
\cite{KendallSmith1940}, who studied
recovering the ranking of a set of players from pairwise comparisons reflecting a total ordering. The last decades have seen a flurry of methods for ranking from pairwise comparisons, mostly based on spectral methods leveraging the eigenvectors of suitably defined matrix operators built directly from the data, which will be detailed in the related work. In particular, SerialRank~\cite{fogel2014SerialRank} shows promising ability in ranking for a set of totally ordered items without ties, by introducing a specific inductive bias from the so-called \emph{seriation} problem \cite{atkins1998spectral}: a Fiedler vector of a certain similarity matrix could recover true rankings given enough pairwise comparisons in the noiseless setting. The \emph{Fiedler value} of a symmetric and nonnegative 
matrix (or a graph) is defined to be the second smallest eigenvalue of the combinatorial Laplacian of the matrix
\cite{fiedler1973algebraic}, and its corresponding eigenvector is called a \emph{Fiedler vector}. 
The Fiedler vector is able to encode useful information of a graph, such as bi-partitions of a graph \cite{fiedler1973algebraic} and the seriation problem \cite{atkins1998spectral}, and can be employed for partitioning hypergraphs \cite{chen2017fiedler}.

Moreover, there has been promising progress on  combinatorial optimization 
in machine learning~\cite{WANG2021100010, maurya2021graph},
especially graph neural networks (GNNs)~\cite{Zhou, Wu},
due to their potential in data exploration.
Compared with their great success in many combinatorial tasks, the capability of GNNs 
in ranking tasks is 
not well 
developed. The few existing works are restricted 
to specific settings e.g. top-$n$ personalized recommendations~\cite{WANG2021100010}, and approximating centrality measures \cite{maurya2021graph}. Another technical gap is the inability to learn a model by directly optimizing the ranking objective, which we aim to fill in our work.

We propose GNNRank, an end-to-end ranking framework compatible with existing GNN models e.g.~\cite{he2021digrac, tong2020digraph} that is able to learn directed graph (digraph) node embeddings. We devise differentiable objectives to encode ranking upsets/violations. Following the standard protocol in \cite{d2021ranking}, an upset of an edge is a ranking violation; the inferred relative ranking is in the opposite direction to what the original measurement indicates. GNNRank consists of a family of ranking score estimation techniques, and adds an inductive bias by unfolding the Fiedler vector computation of the graph 
constructed from a \textit{learnable} similarity matrix.

Our main contributions are as follows:
\begin{itemize}
    \item To the best of our knowledge, this is \emph{the first neural network framework} (specifically GNN) to recover global rankings from pairwise comparisons whereby a direct optimization of
the ranking objectives is enabled, without supervision. Our method
differs from the learning-free methods including SerialRank because we design two novel differentiable losses tailored to neural networks, while the metric $\mathcal{L}_{upset, naive}$ used in previous works is piecewise constant and only used for post-evaluation instead of training the algorithms. Along the way, we adapt the widely used $\mathcal{L}_{upset, naive}$ metric to a more fined-grained evaluation method. 
Thus the motivation is also different from existing works. 
\item For jointly solving global ranking and GNN training, we design and introduce an inductive bias as achieved by the proximal gradient steps to our neural network, as unfolded from the Fiedler vector calculation, as an effective way of encoding latent orderings. The technique of constructing the $Q$ matrix and the transformation of the optimization problem for differentiable computation of Fiedler eigenvectors as part of a deep learning model is nontrivial and, to the best of our knowledge, novel.
\item Our methods empirically attain competitive, and often superior,  accuracy compared to state-of-the-art methods, 
and its cost-effectiveness is especially pronounced when the trained GNN model is transferred to new datasets for ranking recovery. Compared to some 
existing methods e.g. Minimum Violation Rank (MVR) \cite{MVR_1986} with expensive optimizations, our methods have 
better asymptotic time complexity. 
\item From a computational perspective, when the data may have a temporal dimension or when there are 
similarities across different data sets, one can apply an already trained model 
to new data sets that are similar to the one the model has been trained on. 
\item We provide theoretical convergence guarantees for our method, with a technically novel proof, which in our view is a considerable advantage compared to other, ad-hoc, neural solvers.
\end{itemize}

\section{Related Work} \label{sec:relatedWork}
\textbf{Ranking Methods}
One of the most popular models in the ranking literature is the Bradley-Terry-Luce (BTL) model \cite{bradley1952rank}, \cite{luce1959individual}. 
In its basic version, the probability that player $i$ beats player $j$ is given by $P_{ij} = \frac{w_i}{w_i + w_j}$, where the parameter vector $w \in \mathbb{R}_{+}^n$
is estimated from the 
data, and $w_i$ is a proxy for the strength of player $i$.  \cite{firth2005bradley} 
facilitates the specification and fitting of Bradley-Terry logit models to pairwise comparisons.  

Employing the stationary distribution of a suitably defined Markov chain for the ranking task traces its roots in early work on the topic of \textit{network centrality}. Such measures have been designed to quantify the extent to which nodes of the graph (or other network structures) are most important \cite{booknewman}; see for example
\cite{Pageetal98}, \cite{bonacich1987power}, and \cite{negahban2017rank}.

David's Score \cite{david1987ranking} computes rankings from proportions of wins and losses. 
Minimum Violation Rank (MVR) \cite{MVR_1986} encompasses a suite of methods that aim to directly minimize certain penalty functions at the level of each upset.
In our experiments, we compare against the algorithm of  \cite{gupte2011finding}, that considers a linear relaxation of an  integer program that  minimizes a so-called agony loss. 
However, MVR is computationally expensive as also will be shown in our experiments.

SyncRank \cite{cucuringu2016sync} formulates the ranking problem with incomplete noisy pairwise information as an instance of the group synchronization problem over the group SO(2) of planar rotations  \cite{sync}, which has attracted significant attention in recent years.  
The SpringRank algorithm of \cite{de2018physical}  borrows intuition from statistical physics, and proposes to infer hierarchical rankings in directed networks by solving a linear system of equations. 
\cite{d2021ranking} introduces simple algorithms for ranking and synchronization based on singular value decompositions, with theoretical guarantees. SerialRank \cite{fogel2014SerialRank} 
first computes the Laplacian from a certain similarity matrix $\mathbf{S'}$. The corresponding Fiedler vector 
of $\mathbf{S'}$
then serves as the final ranking estimate.
The intuition is that the more similar two players are  (in terms of the pattern of incoming/outgoing edges), the more similar their ranking should be; indeed, if two players defeat, and are defeated, by the same set of other players, then they are likely to have a similar ranking/strength. In this classical ordering problem (called \textit{seriation} \cite{atkins1998spectral}), one is given a similarity matrix between a set of items and assumes that the items can be ordered along a chain such that the similarity between items decreases with their distance  in the chain.

Apart from the above classical learning-free optimizers for ranking problems, there also exist learning-based (mostly neural) models for similar tasks. The early work \cite{scarselli2005graph} applies a certain form of GNN to ranking web pages, while it requires label supervision for ground-truth ranks for training. While the authors in \cite{damke2021ranking} propose the family of so-called RankGNNs, their competitors are graphs. 
\cite{rigutini2011sortnet} applies a neural network approach for preference learning, 
\cite{koppel2019pairwise} generalizes 
\cite{burges2005learning}, but these methods require queries as input, which solve a different problem from ours.
\cite{maurya2021graph} proposes the first GNN-based model to approximate betweenness and closeness centrality, facilitating locating influential nodes in the graphs in terms of information spread and connectivity. The pairwise direction is rarely considered in these works but it
is important for the problem studied in this paper. Thus, 
our task differs fundamentally
from the (abundant) learning-to-rank literature.

\textbf{Directed Graph Neural Networks}
Directed GNNs are useful in learning digraph node embeddings. 
\cite{tong2020digraph} builds aggregators based on higher-order proximity.
\cite{zhang2021magnet} constructs a complex Hermitian Laplacian matrix. 
\cite{he2021digrac} introduces imbalance objectives for digraph clustering. In GNNRank, existing digraph neural networks can be readily incorporated.

\textbf{Unfolding Techiques}
Algorithm unfolding \cite{gregor2010learning} was first introduced to unfold the iterations as a cascade of layers while adding learnable parameters. The unfolding idea has later been applied to problems such as semantic segmentation \cite{liu2017deep} and efficient power allocation \cite{chowdhury2021unfolding}. The work~\cite{yang2021graph} discusses a new family of GNN layers designed to mimic and integrate the update rules of classical iterative algorithms. The unfolding idea inspires us to add a useful inductive bias from the calculation of the Fiedler vector via our bi-level optimization pipeline \cite{talbi2013taxonomy}. Most importantly, unfolding allows us to pass gradients through the optimization of a useful function.
\begin{figure*}[tb!]
\centering
\includegraphics[width=\linewidth]{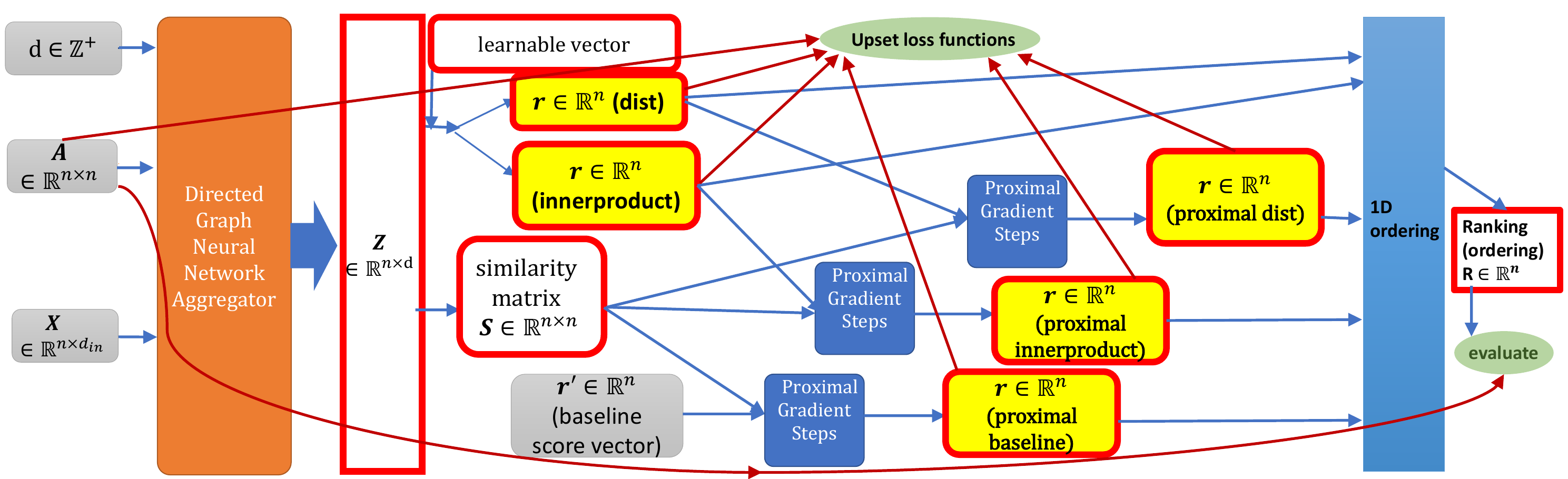}
\vspace{-10pt}
\caption{Overview of GNNRank based on directed graph neural networks and the proximal gradient steps corresponding to Algo.~\ref{algo:framework}: starting from an adjacency matrix $\mathbf{A}$ which encodes pairwise comparisons, input feature matrix $\mathbf{X}$ and embedding dimension $d,$ GNNRank first applies a directed graph neural network model to learn node embeddings $\mathbf{Z}$ for each competitor (node). Then it calculates the inner product or the similarity score with respect to a learnable vector to produce non-proximal outcomes for ranking scores (``innerproduct" or ``dist"). Proximal  variants start from a similarity matrix constructed from the learnable embeddings $\mathbf{Z}$, then utilize proximal gradient steps to output ranking scores. Depending on the initial guess score vector $\mathbf{r}'$, the proximal variants have names ``proximal innerproduct", ``proximal dist" or ``proximal baseline". Ordering the scores in the score vector $\mathbf{r}$ induces the final ranking/ordering vector $\mathbf{R}\in\mathbb{R}^n.$ The loss function is applied to a variant's output score vector $\mathbf{r}$, given the  input adjacency matrix $\mathbf{A},$ while the final evaluation is based on $\mathbf{R}$ and $\mathbf{A}.$ Red frames indicate trainable tensors/vectors/matrices. Grey squares correspond to fixed inputs.}
    \label{fig:framework_overview}
\end{figure*}
\section{Approach}
\subsection{Problem Definition}
Without loss of generality, we consider pairwise comparisons in a competition, which can be encoded in a directed graph (digraph) $\mathcal{G}=(\mathcal{V}, \mathcal{E})$, where the node set $\mathcal{V}$ denotes competitors, and the edge set $\mathcal{E}$ represents pairwise comparisons. The outcomes of the matches can be captured by the adjacency matrix $\mathbf{A}$. 
A single edge $e\in\mathcal{E}$ from node $v_i$ to node $v_j$, with edge weight $\mathbf{A}_{i,j}\geq 0$ that is not reciprocal, denotes that node $v_i$ is stronger than node $v_j$ by $\mathbf{A}_{i,j}$
For a reciprocal edge, $\mathbf{A}_{i,j}$ and $\mathbf{A}_{j,i}$ could be different; 
they could denote the (sum of) match scores for both competitors in matches between them, or the sum of absolute wins across different matches between them, where $\mathbf{A}_{i,j}$ is the sum of absolute wins for $v_i.$
Recovering global rankings from pairwise comparisons amounts to assigning an integer $R_i$ to each node $v_i\in\mathcal{V}$, denoting its position among competitors, where the lower the rank, the stronger the node is. To this end, many existing methods (including our proposed methods) first learn a real-valued ranking score $r_i$ for $v_i$, where the higher the score, the stronger the node is. The scores are then ordered to provide the final integer ranking.

\subsection{Motivation and Connection to SerialRank}
\label{subsec:overview}
The whole GNNRank framework, described in Algo.\,\ref{algo:framework}, uses steps from Algo.\,\ref{algo:proximal} for the lower-level optimization within the whole bi-level optimization pipeline; 
a diagram is provided in Fig.~\ref{fig:framework_overview}. Details are provided later in this section.

Indeed, we highlight an intrinsic connection to SerialRank \cite{fogel2014SerialRank}, which operates by  first computing the Laplacian from a certain similarity matrix $\mathbf{S}'=\frac{1}{2}(n\mathbf{1}\mathbf{1}^\top + \mathbf{C}\mathbf{C}^\top)$, where $\mathbf{C}$ is the binary comparison matrix with $\mathbf{C}_{i,j}=\text{sign}(\mathbf{A}_{i,j}-\mathbf{A}_{j,i})$,
and $\mathbf{A}$ the
digraph adjacency matrix. The corresponding Fiedler vector of $\mathbf{S}'$ 
then serves as the final ranking estimate, after a global sign reconciliation. While often effective in practice, SerialRank is heavily dependent on the quality of the underlying similarity matrix $\mathbf{S}'$.

To address this issue, we introduce a parameterized GNN model that allows us to compute trainable measures of similarity that are useful for subsequent ranking. However, for training purposes, we of course need to somehow back-propagate gradients through the computation of a Fiedler vector to update the GNN parameters.  Because it is generally difficult to directly pass gradients through eigenvector computations, we instead express the Fiedler vector as the solution of a constrained optimization problem.  We then approximate the solution of this problem using proximal gradient steps, each of which are themselves differentiable with respect to the underlying optimization variables, and ultimately by the chain rule, the GNN parameters. Note that Algo. \ref{algo:proximal} can be viewed as a \emph{differentiable} function that inputs $\mathbf{r}'$ and $\mathbf{L}$, and outputs $\mathbf{r}$.  Therefore, the gradients can backpropagate \emph{through} Algo. \ref{algo:proximal} into 
$\mathbf{L}$ and $\mathbf{r}'$, hence the similarity matrix $\mathbf{S}$ and the rest of the model parameters.

In broader contexts, this process is sometimes referred to as \textit{unfolded optimization}  \cite{gregor2010learning}, and is applicable in situations whereby a high-level loss function (in our case a ranking loss) is defined with respect to the minimization of some lower-level, parameter-dependent optimization problem (e.g., Fiedler vector computation) that has been unfolded across differentiable iterations/updates. Within our proposed framework, this process allows to combine the inductive bias of Fiedler-vector-based rankings with the flexibility of GNNs for modeling relations between entities. Also, SerialRank can be viewed as a special case of GNNRank;
$\mathbf{S}'_{i,j}$ counts the number of competitors that beat both $v_i$ and $v_j,$ plus that of competitors beaten by both $v_i$ and $v_j$, minus that of competitors beating one of $v_i$ and $v_j$ but
beaten by the other, plus half the number of nodes. 
Information from common 
neighbors could be aggregated by directed GNNs, and a 
kernel such as a Gaussian RBF kernel applied to the digraph embeddings could approximate $\mathbf{S}'$.

\subsection{Loss Functions and Objectives}

Define the skew-symmetric 
matrix $\mathbf{M'} = \mathbf{A} - \mathbf{A}^\top,$ and let $t$ be the number of nonzero elements in $\mathbf{M'}$. For
a vector $\mathbf{r}=[r_1, \dots,r_n]^\top$ with \emph{real-valued} ranking scores (often viewed as skill levels)  as entries, a naive upset
is defined as the fraction of relationships disagreeing with their expected sign, in the spirit of \cite{d2021ranking}. Formally, let $\mathbf{1}$ 
be an all-one column vector, and define the matrix $\mathbf{T'}=\mathbf{r1}^\top-\mathbf{1r}^\top\in\mathbb{R}^{n\times n}.$ Then we have $\mathbf{T'}_{i,j} = r_i-r_j, \forall i,j\in\{1,\cdots,n\}$. To not penalize entries where the initial pairwise rankings are not available, 
we only compare $\mathbf{T'}$ with $\mathbf{M'}$ at locations where $\mathbf{M'}$ has nonzero entries. A naive upset loss is 
defined as
\begin{equation} 
\vspace{-3mm}
    \mathcal{L}_\text{upset, naive} = 
    {\sum_{i,j:\mathbf{M'}_{i,j}\neq 0}(\text{sign}(\mathbf{T'}_{i,j})\neq \text{sign}(\mathbf{M'}_{i,j}))}/t.
    \label{eq:loss_upset_naive}
\end{equation}

Being piecewise constant, 
$\mathcal{L}_\text{upset, naive}$ is not useful in gradient descent. To account for the difference in the scaling between the output scores and the input adjacency matrix, 
we define element-wise divisions  $\mathbf{M} = \frac{\mathbf{A} - \mathbf{A}^\top}{\mathbf{A} + \mathbf{A}^\top},$ and $\mathbf{T}=\frac{\mathbf{r1}^\top-\mathbf{1r}^\top}{\mathbf{r1}^\top+\mathbf{1r}^\top}\in\mathbb{R}^{n\times n}.$ Then $\mathbf{T}_{i,j} = \frac{r_i-r_j}{r_i+r_j}, \forall i,j\in\{1,\cdots,n\}.$ 
Similarly, we only compare $\mathbf{T}$ with $\mathbf{M}$ at locations where $\mathbf{M}$ has nonzero entries.  

With $\Tilde{\mathbf{T}}$ where $\Tilde{\mathbf{T}}_{i,j}=\mathbf{T}_{i,j}$ if $\mathbf{M}_{i,j}\neq 0$ and $\Tilde{\mathbf{T}}_{i,j}=0$ if $\mathbf{M}_{i,j}= 0$, the differentiable upset loss is then defined as
\vspace{-1.5mm}
\begin{equation}
\vspace{-1.5mm}
    \mathcal{L}_\text{upset, ratio} = 
    \left\lVert \Tilde{\mathbf{T}}-\mathbf{M} \right\rVert^2_F/t(\mathbf{M}),
    \label{eq:loss_upset}
\end{equation}
where the subscript $F$ means Frobenius norm, 
and $t(\mathbf{M})$ is the number of nonzero elements in $\mathbf{M}.$

Note that $\mathcal{L}_\text{upset, ratio}$ requires $r_i\geq 0$ for each $r_i$ to represent the skill level. Hence we use transformations e.g. $r_i \leftarrow \text{sigmoid}(r_i)$ for general $r_i\in\mathbb{R}$,  and $r_i \leftarrow \frac{r_i+1}{2}$ when we know $r_i\geq -1,$ e.g. when the score  $\mathbf{r}=(r_i)$ has unit norm. 

Another choice is the margin loss $\mathcal{L}_\text{upset, margin}$, with a tunable nonnegative parameter $\epsilon\geq 0$ (default 0.01), defined as 
\begin{equation}
    \mathcal{L}_\text{upset, margin} = \sum_{i, j}(\mathbf{M}_{i,j}+|\mathbf{M}_{i,j}|)\cdot \text{ReLU}(r_j - r_i + \epsilon)/t(\mathbf{M}), \label{eq:loss_margin}
\end{equation}
where
ReLU is the Rectified Linear Unit.

The training loss used is either $\mathcal{L}_\text{upset, ratio}$, $\mathcal{L}_\text{upset, margin}$ or their sum, where the choice is set as a hyperparameter.

\vspace{-3mm} 
\paragraph{Evaluation.}For evaluation, in addition to $\mathcal{L}_\text{upset, naive}$, we introduce another objective similar to $\mathcal{L}_\text{upset, naive}$ but penalizing predicted signs opposite to the actual signs more than predicted signs being zero, i.e. we distinguish predictions as ties or being opposite signs, which can also take integer rankings as input
\begin{equation}
\vspace{-1mm}
    \mathcal{L}_\text{upset, simple} =
    \left\lVert \text{sign}(\mathbf{T'})- \text{sign}(\mathbf{M'})) \right\rVert^2_F/t,
    \label{eq:loss_upset_simple}
\end{equation}
where the sign function acts element-wise, 
and $t$ is the number of nonzero elements in $\mathbf{M}'.$ However, although $\mathcal{L}_\text{upset, simple}$ distinguishes ties and opposite signs, this loss 
can easily be made equal to one by assigning the same score to all nodes (i.e., making $\mathbf{r}$ a constant vector which corresponds to all ties), which means that whenever we achieve a value larger than one, the model performs even worse than trivial guess.

When ground-truth rankings are available, we use the Kendall tau \cite{brandenburg2013comparing} values for evaluation, and select the model based on the lowest $\mathcal{L}_\text{upset, simple}$.
\subsection{Obtaining Directed Graph Embeddings}
For obtaining a directed graph embedding, any GNN method which can
take into account directionality and output node embeddings could be applied, e.g. DIMPA by \cite{he2021digrac}, the inception block model (IB) \cite{tong2020digraph}, and MagNet \cite{zhang2021magnet}. In our experiments, we employ DIMPA and IB, to aid in the ranking task.
Denoting the final node embedding by $\mathbf{Z}\in \mathbb{R}^{n \times d}$, the embedding vector $\mathbf{z}_i$ for a node $v_i$ is $\mathbf{z}_i=(\mathbf{Z})_{(i,:)} \in \mathbb{R}^{d}, $ the $i^\text{th}$ row of $\mathbf{Z}$. 

\subsection{Obtaining Final Scores and Rankings}
\label{subsec:proximal}

To obtain the final ranking score, we unfold the calculation of a Fiedler vector for the graph
constructed from our symmetric similarity matrix $\mathbf{S}$ with proximal gradient steps. 
\vspace{-6mm}
\paragraph{Obtaining the similarity matrix.}From the high-dimensional embedding matrix $\mathbf{Z},$ we calculate the symmetric similarity matrix $\mathbf{S}$ with $\mathbf{S}_{i,j}=\exp(-|\mathbf{z}_j-\mathbf{z}_i|_2^2/(\sigma^2d))$  
    where $\sigma\in\mathbb{R}$ is the same trainable parameter as in ``dist". 
    Denote by $\mathbf{D}$ the diagonal matrix with $\mathbf{D}_{i,i}=\sum_j\mathbf{S}_{i,j}$. 
    We consider the unnormalized Laplacian $\mathbf{L = D - S}$, and apply proximal gradient to approximate a Fiedler vector of $\mathbf{S}$,
    which then serves as $\mathbf{r}$.  
\vspace{-3mm}
\paragraph{Transformation of the Optimization Problem.}Computing a Fiedler vector of the similarity matrix $\mathbf{S}$
is equivalent to solving the optimization problem \cite{von2007tutorial}
\vspace{-1mm} 
\begin{align}
\label{eq:problem_original}
\min_\mathbf{r} \mathbf{r}^\top \mathbf{L} \mathbf{r} 
\text{ s.t.} &\quad \left\lVert \mathbf{r} \right\rVert_2^2 = 1, \quad \mathbf{r}^\top \mathbf{1} = 0, 
\end{align} 
\noindent  where $\mathbf{L}$ is the graph Laplacian matrix.
We observe that the constraints describe an intersection of a unit sphere and a hyperplane.  By rotating the problem and the constraints so that the hyperplane becomes cardinal, we can effectively fix one dimension to zero and solve
\vspace{-1mm}
\begin{equation}
   \min_\mathbf{y} \mathbf{y}^\top
   \mathbf{Q} \mathbf{L} \mathbf{Q}^\top 
   \mathbf{y} 
\text{ s.t.} \quad \left\lVert \mathbf{y} \right\rVert_2^2 = 1, y_1=0. 
\end{equation}
Here $\mathbf{Q}$ is an orthogonal matrix. 
We choose a
$\mathbf{Q}$ such that $\mathbf{Q} \mathbf{1} = \sqrt{n} \mathbf{e}_1$, where $\mathbf{e}_1 = \begin{bmatrix} 1 & 0 & \cdots & 0 \end{bmatrix}^\top$.
Let $\mathbf{y} = \mathbf{Q} \mathbf{r}$, and thus $\mathbf{r} = \mathbf{Q}^\top \mathbf{y}$. The problem  becomes
\vspace{-1mm}
\begin{align*}
&\min_\mathbf{y} \mathbf{y}^\top \mathbf{Q} \mathbf{L} \mathbf{Q}^\top \mathbf{y} \\
\text{s.t.} &\quad
\left\lVert \mathbf{r} \right\rVert_2^2=1
\qquad  \mathbf{r}^\top \mathbf{1}  
= \sqrt{n} \mathbf{y}^\top \mathbf{e}_1 = \sqrt{n} y_1 = 0. 
\end{align*}
Since we fix $y_1 = 0$, this is equivalent to
\begin{align}
\vspace{-3mm}
\label{eq:problem_reduced}
&\min_{\mathbf{y}' \in \mathbb{R}^{n-1}} \mathbf{y}'^\top \left[ \mathbf{Q} \mathbf{L} \mathbf{Q}^\top\right]_{2:n,2:n} \mathbf{y}' \quad 
\text{s.t.} \quad \left\lVert \mathbf{y}' \right\rVert_2^2 = 1,  
\end{align}
where $\left[\cdot\right]_{2:n,2:n}$ represents the matrix with its first row and first column removed. 
To illuminate this equivalence, since the constraint $\mathbf{r}^\top \mathbf{1}=0$ is equivalent to $y_1=0,$ we need to ensure that $y_1=0$ is maintained throughout. If we start with $\mathbf{y}\in\mathbb{R}^n$ where $y_1=0,$ and let $\mathbf{y}'=[y_2, \dots,y_n]^\top \in \mathbb{R}^{n-1}$ then \begin{align*}\vspace{-3mm}
&\mathbf{y}^\top \left[ \mathbf{Q} \mathbf{L} \mathbf{Q}^\top\right]_{1:n,1:n} \mathbf{y}=\sum_{1\leq i \leq n, 1\leq j \leq n}\left[ \mathbf{Q} \mathbf{L} \mathbf{Q}^\top\right]_{i,j}y_iy_j \\ = & \sum_{2\leq i \leq n, 2\leq j \leq n}\left[ \mathbf{Q} \mathbf{L} \mathbf{Q}^\top\right]_{i,j}y_iy_j
= \mathbf{y}'^\top \left[ \mathbf{Q} \mathbf{L} \mathbf{Q}^\top\right]_{2:n,2:n} \mathbf{y}'.
\end{align*} 
One possible $\mathbf{Q}$, with details of the construction given in Appendix~\ref{appendix_sec:find_Q_details}, is
the following upper Hessenberg matrix, 
which
can be efficiently precomputed
\begin{equation*}
\mathbf{Q}_{ij} = \begin{cases}
\sqrt{\frac{1}{n}} & i=1 \\
-\sqrt{\frac{n-i+1}{n-i+2}} & i\geq 2, j=i-1 \\
\sqrt{\frac{1}{(n-i+1)(n-i+2)}} & i\geq 2, j\geq i \\
0 & \text{otherwise.}
\end{cases}
\end{equation*}
\paragraph{Proximal Gradient Steps.}  
To enforce a zero in the first entry of the initial guess, we
zero-center the input score vector, then left multiply it by $\mathbf{Q}$, so this resulting vector has 0 in its first entry. We then remove the first entry of the resulting vector, and discard the first row and first column for the matrix $\mathbf{Q L Q^\top}$. The gradient of the objective $\mathbf{y^\top}\widetilde{\mathbf{L}}\mathbf{y}$ with respect to $\mathbf{y}$ is $(\widetilde{\mathbf{L}} + \widetilde{\mathbf{L}}^\top)\mathbf{y}=2\widetilde{\mathbf{L}}\mathbf{y}$ since $\widetilde{\mathbf{L}}$ is symmetric. We set the 
number of proximal steps as $\Gamma=5,$ and initial learning rates inside proximal gradient steps $\alpha_\gamma = 1, \gamma=1, \ldots, \Gamma$.
Define the spherical projection operation $\mathcal{P}_{\mathcal{S}^{n-1}}(\cdot):\mathbb{R}^n\rightarrow\mathbb{R}^n$ by $\mathcal{P}_{\mathcal{S}^{n-1}}(\mathbf{x})=\frac{\mathbf{x}}{\left\lVert \mathbf{x} \right\rVert_2}$ if $\mathbf{x}\neq \mathbf{0}$ and $\mathcal{P}_{\mathcal{S}^{n-1}}(\mathbf{0})=\mathbf{e}_1,$ where $\mathbf{0}=[0,\dots,0]^\top.$ Algo.\,\ref{algo:proximal} details the proximal gradient steps, with proximal operator $\mathcal{P}_{\mathcal{S}^{n-2}}$, which guarantees descent of the optimization objective in~\eqref{eq:problem_original} given suitable $\alpha$'s (see Thm.~\ref{thm:convergence}).
\begin{algorithm}[tb!]
\begin{small}
\caption{Proximal Gradient Steps}
\label{alg:algorithmprox}
\textbf{Input}: Initial score $\mathbf{r'}\in\mathbb{R}^n$, Laplacian $\mathbf{L}\in\mathbb{R}^{n\times n}$ and $\mathbf{Q}\in\mathbb{R}^{n\times n}$\\
\textbf{Parameter}: (Initial) learning rate set $\{\alpha_\gamma>0\}_{\gamma=1}^\Gamma$ that could either be fixed or trainable (default: trainable).\\
\textbf{Output}: Updated score vector $\mathbf{r}=[r_1, \dots, r_n]^\top$. 

\begin{algorithmic}[1] 
\STATE Let $\mathbf{y}= \mathbf{r'} - \sum_{i=1}^n r'_i/n$;
\STATE $\mathbf{y} \leftarrow \mathbf{Q}'\mathbf{y} \in \mathbb{R}^{n-1}$ ($\mathbf{Q}'$ is $\mathbf{Q}$ with the first row removed);
\STATE $\mathbf{y} \leftarrow \mathcal{P}_{\mathcal{S}^{n-1}}(\mathbf{y})$ to have unit 2-norm;
\STATE $ \widetilde{\mathbf{L}} \leftarrow \left[\mathbf{Q L Q}^\top\right]_{2:n,2:n}.$
\FOR{$\gamma < \Gamma$}
\STATE $\mathbf{y} \leftarrow \mathbf{y} - \alpha_\gamma (2\widetilde{\mathbf{L}})\mathbf{y}/n$; 
\STATE $\mathbf{y} \leftarrow \mathcal{P}_{\mathcal{S}^{n-2}}(\mathbf{y})$ to have unit 2-norm;
\STATE $\gamma \leftarrow  \gamma + 1.$
\ENDFOR
\STATE $\mathbf{y} \leftarrow \text{CONCAT}(0, \mathbf{y})\in\mathbb{R}^n$;
\STATE $\mathbf{r} = \mathbf{Q}^\top \mathbf{y}$
\STATE \textbf{return} $\mathbf{r}$;
\end{algorithmic}
\label{algo:proximal}
\end{small}
\end{algorithm}
\begin{algorithm}[tb!]
\begin{small}
\caption{GNNRank: Proposed Ranking Framework}
\label{alg:algorithm}
\textbf{Input}: Digraph adjacency matrix $\mathbf{A}\in\mathbb{R}^{n\times n}$, node feature matrix $\mathbf{X}$, variant name $var$, (optional) initial guess from baseline $\mathbf{r'}$.\\
\textbf{Parameter}: Learnable vector $\mathbf{a}\in\mathbb{R}^n, b\in\mathbb{R}, \sigma\in\mathbb{R}$, GNN parameters, parameters from Algo.\,\ref{algo:proximal}.\\
\textbf{Output}: Score vector $\mathbf{r}$.
\begin{algorithmic}[1] 
\STATE $\mathbf{Z}\in\mathbb{R}^{n\times d} = \text{GNN}(\mathbf{A},\mathbf{X})$;
\IF{$var\in$\{``dist", ``proximal dist"\}}
\STATE Compute $\mathbf{r}: r_i = \exp(-|\mathbf{a}-\mathbf{z}_i|_2^2/(\sigma^2d))$;
\ELSIF{$var\in$\{``innerproduct", ``proximal innerproduct"\}}
\STATE Compute $\mathbf{r}: r_i = \text{sigmoid}(z_i\cdot \mathbf{a}+b)$;
\ENDIF
\IF{$var\in$\{``proximal dist", ``proximal innerproduct",  ``proximal baseline"\}}
\STATE Compute $\mathbf{S}:\mathbf{S}_{i,j}=\exp(-|\mathbf{z}_j-\mathbf{z}_i|_2^2/(\sigma^2d))$;
\STATE Compute Laplacian $\mathbf{L}$ and $\mathbf{Q}$ by Sec. \ref{subsec:proximal};
    \IF{$var==$\{``proximal baseline"\}}
    \STATE $\mathbf{r} =\mathbf{r'}$;
    \ENDIF
    \STATE $\mathbf{r}\leftarrow \text{Proximal Gradient Steps} (\mathbf{r}, \mathbf{L}, \mathbf{Q})$ from Algo.\,\ref{algo:proximal};
\ENDIF
\STATE \textbf{return} $\mathbf{r}$;
\end{algorithmic}
\label{algo:framework}
\end{small}
\end{algorithm}

\subsection{Initialization and Pretraining Considerations}
Algo.\,\ref{algo:proximal} requires an initial guess $\mathbf{r'}\in\mathbb{R}^n$.
To achieve this, we introduce two non-proximal variants, whose output score vector could serve as $\mathbf{r'}$. These non-proximal variants are also evaluated in our experiments.

\bb (1) \emph{``innerproduct"} variant: With a trainable vector $\mathbf{a}$ that has its dimension equal to the embedding dimension, we obtain the scores by the inner product of $\mathbf{z}_i$  with  $\mathbf{a}$, plus a trainable  bias $b$, followed by a sigmoid layer to force positive score values: $r_i = \text{sigmoid}(\mathbf{z}_i \cdot \mathbf{a} + b).$ 
    
\bb (2) \emph{``dist"} variant: $r_i = \exp(-|\mathbf{a}-\mathbf{z}_i|_2^2/(\sigma^2d))$ where $\mathbf{a}\in\mathbb{R}^{d}, \sigma\in\mathbb{R}$ are trainable parameters. This variant applies a trainable Gaussian RBF kernel to describe the scores.

For each non-proximal variant, the corresponding proximal variant
is called \bb (3) \emph{``proximal dist"} or \bb (4) \emph{``proximal innerproduct"}. We can also adopt the initial guess from a certain existing baseline's output, with variant name \bb (5) \emph{``proximal baseline"}.

For a reasonable similarity matrix to start with, we apply some pretraining. One option is to train the ranking model with a non-proximal variant, ``dist" or ``innerproduct", for the early training epochs. 
As the proximal variants are inspired by unfolding the Fiedler vector calculation introduced in SerialRank \cite{fogel2014SerialRank}, another option is to add a term to the loss function in early training epochs that compares the similarity matrix constructed by $\mathbf{Z}$ using a GNN with the normalized version (divided by the max entry so that the maximum is 1) of $\mathbf{S'}$, the similarity matrix from SerialRank. 
The corresponding additional term added to the loss function $\mathcal{L}_\text{upset, ratio}$ from Eq.~\eqref{eq:loss_upset} or $\mathcal{L}_\text{upset, margin}$ from Eq.~\eqref{eq:loss_margin}, or their sum,
is $\mathcal{L}_\text{similarity} = n^{-2}  {\left\lVert \mathbf{S}-\mathbf{S'}/\max(\{\mathbf{S'}_{i,j}\})\right\rVert^2_F}.$
\subsection{Convergence Analysis}
An analysis of the convergence of our proximal gradient steps is provided in Appendix~\ref{appendix_sec:convergence_dis}, along with additional theoretical and practical considerations. We summarize our main result below; the proof is available in  Appendix~\ref{appendix_sec:convergence_dis}.
\begin{theorem}
\label{thm:convergence}
Let $\{\alpha_\gamma>0\}_{\gamma=1}^\Gamma$ in Algo.~\ref{algo:proximal} be fixed (equal to $\alpha$) and let $\rho$ be the Fiedler eigenvalue of $\mathbf{S}$.  Denote a Fiedler eigenvector by $\mathbf{r}^*$. Assume that $\mathbf{r}^*$ is a strict local minimizer of problem~(\ref{eq:problem_original}).
If 
$0<\alpha<\frac{1}{4(n-1)},$ 
then with our definition of $\mathbf{S}$,  Algo.~\ref{algo:proximal} converges locally uniformly to $\mathbf{r}^*$.
\end{theorem}

\section{Experiments}
Implementation details are provided in Appendix~\ref{appendix_sec:implementation}. 
Experiments were conducted on a compute node with 8 Nvidia Tesla T4, 96 Intel Xeon Platinum 8259CL CPUs @ 2.50GHz and $378$GB RAM.
\subsection{Data sets and Protocol}
\textbf{Real-World Data.} 
We consider 10 real-world data sets (78 digraphs constructed in total). Digraphs based on pairwise comparisons are constructed as follows in most cases: for each match, an edge is added if the match result is not a tie, 
with the edge  weight  
the difference between scores (absolute wins, that is, the absolute difference of the scores), 
and ties are treated the same as no match between the pair of players. 

When raw data with individual match scores on both competitors are available, we in addition construct a \emph{``finer"} version which takes ties into account, as follows: for each match, we have a reciprocal edge with potentially different weights on the two
directions. If there are multiple matches, weights are added. To distinguish zero (a tie) from
no game, for our methods, we add
a small value (here 0.1) to every existing edge. 

Real-world data sets, detailed in Appendix~\ref{appendix_sec:full_data_stats}, include 
\begin{itemize}
    \item NCAA College Basketball\footnote{\scriptsize{\url{https://www.ncaa.com/sports/basketball-men/d1}}} (used to construct the networks \textit{Basketball} and \textit{Basketball finer}) are the
outcomes of US College basketball
matches from seasons 1985-2014. 
We construct two separate digraphs for each season, 
a regular version and a ``finer" version introduced in the last paragraph. 
\item England Football Premier League\footnote{\scriptsize{\url{https://www.premierleague.com/}}} (\textit{Football} and \textit{Football finer}) are  Premier League football
match results for six seasons, from 2009 to 2014, between 20 teams. As in NCAA College Basketball, we construct two separate digraphs for each season. 

\item The 
Animal Dominance Network (\textit{Animal}) \cite{hobson2015social} describes the number of net aggressive wins of 21 captive monk parakeets. 

\item
Microsoft Halo 2 Tournament on Head-to-Head Games (\textit{HeadToHead}) collects game outcomes during the Beta testing period for the Xbox game Halo 2\footnote{Credits for using the Halo 2 Beta data set are given to Microsoft Research Ltd. and Bungie.}. 
\item
Faculty Hiring Networks (\textit{Faculty: Business}, \textit{Faculty: CS} and \textit{Faculty: History}) \cite{clauset2015systematic} contains three North American academic hiring networks tracking the flow of academics between universities.  

\item
Lead-Lag Relationships on Stocks (\textit{Finance}) \cite{bennett2021detection} contains lead-lag relationships on 1315 stocks from 2001-2019. An edge $(i,j)$ in the digraph encodes the $t$-value of the coefficient in the regression of the daily returns of stock $i$ on the  lag-1 (previous day) returns of stock $j$.
\end{itemize}

Table \ref{tab:data sets} gives the number of nodes ($n$), the number of directed edges ($|\mathcal{E}|$), the number of reciprocal edges ($|\mathcal{E}^r|$) (self-loops are counted once and for $ u \ne v$, a reciprocal edge $u\rightarrow v, v\rightarrow u$ is counted twice) and their percentage among all edges, for the real-world networks, illustrating the variability in network size and density (defined as ${|\mathcal{E}|}/[{n(n-1)}]$). 
When input features are unavailable, we stack the real and imaginary parts of the top $K$ eigenvectors of $(\mathbf{A}-\mathbf{A}^\top)\cdot i$ in line with the protocol in \cite{cucuringu2020hermitian} for GNNs. 
We report the embedding dimension $d$.

\begin{table}[tb!]
\centering
\vspace{-10pt}
\caption{Summary statistics 
for the real-world networks.
} 
\label{tab:data sets}
\resizebox{1\linewidth}{!}{
\begin{tabular}{lrrrrrrr}
\toprule
Data               & $n$ &  $|\mathcal{E}|$&density&$|\mathcal{E}^r|$&$\frac{|\mathcal{E}^r|}{|\mathcal{E}|}$(\%)&$K$&$d$  \\ \midrule
\textit{HeadToHead}&602&5010&1.38e-02&464&9.26&48&32\\
\textit{Finance}&1315&1729225&1.00e+00&1729225&100&20&64\\
\textit{Animal}&21&193&4.60e-01&64&33.16&3&8\\
\textit{Faculty:Business}&113&1787&1.41e-01&0&0.00&5&16\\
\textit{Faculty:CS}&206&1407&3.33e-02&0&0.00&9&16\\
\textit{Faculty:History}&145&1204&5.77e-02&0&0.00&12&16\\
\textit{Football(avg)}&20&201&5.29e-01&71&32.17&9&8\\
\textit{Basketball(avg)}&316&3506&3.51e-02&986&28.57&20&16\\
\textit{Football finer (avg)}&20&367&9.65e-01&367&100&9&8\\
\textit{Basketball finer (avg)}&316&6139&6.12e-02&6139&100&20&16\\
\bottomrule
\end{tabular}}
\vspace{-3mm}
\end{table} 


\textbf{Synthetic Data}
We perform experiments on graphs with $n=350$ nodes for Erd\H{o}s-R\'enyi Outlier (ERO) models 
as in \cite{d2021ranking}, with edge density $p\in\{0.05, 1\}$, noise level  $\eta\in\{0,0.1,\dots,0.8\}$ (corresponding to $\gamma$ in \cite{d2021ranking}) and style ``uniform" or ``gamma" depending on the distribution from which the ground-truth scores are generated. \footnote{This synthetic graph is
an ER graph with tunable noise level, a test case which is  informative without being  too unrealistic.} We fix $K=5, d=16$ for the ERO.

\subsection{Main Experimental Results} 
In our numerical experiments, we compare against \underline{\textbf{11 baselines}}, where results are averaged over 10 runs: 
\bb Eigenvector Centrality (Eig.Cent.) \cite{bonacich1987power},  
\bb PageRank \cite{Pageetal98},  
\bb Rank Centrality (RankCent.) \cite{negahban2017rank},  
\bb Minimum Violation Rank (MVR) \cite{MVR_1986}, 
\bb SerialRank \cite{fogel2014SerialRank}, 
\bb SyncRank \cite{cucuringu2016sync}, 
\bb SVD\_NRS and SVD\_RS by \cite{d2021ranking}, 
\bb Bradley-Terry-Luce (BTL) model \cite{firth2005bradley}, 
\bb David's Score (DavidScore) \cite{david1987ranking}, and 
\bb SpringRank \cite{de2018physical}. Models are selected based on the lowest $\mathcal{L}_\text{upset, simple}$ obtained without label supervision, where non-proximal results (for ``dist" and ``innerproduct" variants) are listed in the ``GNNRank-N" column and proximal methods (for ``proximal dist", ``proximal innerproduct" and ``proximal baseline" variants) listed with ``GNNRank-P", in all tables. We report the best-performing variant for each data set within GNNRank-N and GNNRank-P, respectively.  
Performance with respect to each of the GNNRank variants, for each individual input digraph, and on different objectives, are given in Appendix~\ref{appendix_sec:full_result_tables}.
Table \ref{tab:upset_simple_naive} compares our two groups of methods 
against baselines on 10 real-world data sets, where basketball data sets are averaged over 30 seasons, and football ones are averaged over 6 seasons.  
Our best-performing variant of the  proximal method also outperforms  its inspiration
SerialRank \cite{fogel2014SerialRank} on all real-world data sets by including more trainable parameters. Our proximal method achieves state-of-the-art performance when using a typically good baseline, such as SyncRank, as an initial guess to the score vector before applying proximal gradient steps.

Table \ref{tab:synthetic} shows
Kendall tau 
selected from the lowest 
$\mathcal{L}_\text{upset, naive}$ on synthetic models. For some dense digraphs, SerialRank (which motivated our proximal gradient steps) attains leading performance, while for some other cases it fails. GNNRank-P outperforms across all synthetic models shown here. Full results are in Appendix~\ref{appendix_subsec:synthetic_results_full}.

We conclude that both non-proximal and proximal methods can achieve leading performance on real-world data sets, while on the synthetic models listed here, the best method in GNNRank-P performs much better than the best method
in GNNRank-N. Performance results for
each of the variants are provided in Appendix~\ref{appendix_subsec:variant_wise_results_full}; 
the individual variants also attain comparable and often superior performance compared to the  baselines.

We observe across all data sets that our proximal methods:
\bb (1)
 can improve on existing baseline methods when using them as initial guesses, and never perform significantly worse than the corresponding baseline, hence they can be used to enhance existing methods;
\bb (2) do not rely on baseline methods for an initial guess  but can instead use GNNRank-N outcomes, such as ``proximal dist" and ``proximal innerproduct'';
\bb (3) can outperform SerialRank by unfolding its Fiedler vector calculations with a trainable similarity matrix and proximal gradient steps. 

\begin{table*}[!tb]
\centering
\vspace{-10pt}
\caption{Performance on $\mathcal{L}_\text{upset, simple}$ (top half) and $\mathcal{L}_\text{upset, naive}$ (bottom half), averaged over 10 runs with one standard deviation. ``avg" for time series first average over all seasons, then consider mean and standard deviation over the 10 averaged values. The best is marked in \red{bold red} while the second best is in \blue{underline blue}. When MVR does not generate results after one week, we fill in ``NAN".}
\label{tab:upset_simple_naive}
\resizebox{1\linewidth}{!}{\begin{tabular}{lrrrrrrrrrrrrrrrrrr}
\toprule
Data  & SpringRank & SyncRank & SerialRank & BTL & DavidScore & Eig.Cent. & PageRank & RankCent. & SVD\_RS & SVD\_NRS& MVR  &  GNNRank-N & GNNRank-P \\
\midrule
{\it HeadToHead} & 1.00$\pm$0.00 & 1.94$\pm$0.00 & 2.01$\pm$0.00 & 1.12$\pm$0.01 & 1.16$\pm$0.00 & 1.47$\pm$0.00 & 1.36$\pm$0.00 & 2.00$\pm$0.02 & 1.79$\pm$0.00 & 1.42$\pm$0.00 & nan$\pm$nan & \blue{0.99$\pm$0.00} & \red{0.96$\pm$0.00} \\
			{\it Finance} & 1.63$\pm$0.00 & 1.98$\pm$0.00 & 1.61$\pm$0.00 & 1.78$\pm$0.01 & 1.63$\pm$0.00 & 1.74$\pm$0.00 & 1.75$\pm$0.00 & 1.88$\pm$0.00 & 1.64$\pm$0.00 & 1.64$\pm$0.00 & nan$\pm$nan & \red{1.00$\pm$0.00} & \red{1.00$\pm$0.00} \\
			{\it Animal} & 0.50$\pm$0.00 & 1.62$\pm$0.24 & 1.98$\pm$0.48 & 0.45$\pm$0.00 & \blue{0.33$\pm$0.00} & 0.55$\pm$0.00 & 0.63$\pm$0.00 & 1.96$\pm$0.00 & 1.03$\pm$0.00 & 0.53$\pm$0.00 & 2.02$\pm$0.32 & 0.41$\pm$0.09 & \red{0.25$\pm$0.00} \\
			{\it Faculty: Business} & 0.41$\pm$0.00 & 0.83$\pm$0.00 & 1.19$\pm$0.00 & 0.41$\pm$0.01 & 0.49$\pm$0.00 & 0.49$\pm$0.00 & 0.49$\pm$0.00 & 2.01$\pm$0.03 & 0.68$\pm$0.00 & 0.46$\pm$0.00 & 0.78$\pm$0.05 & \blue{0.38$\pm$0.01} & \red{0.36$\pm$0.00} \\
			{\it Faculty: CS} & \blue{0.33$\pm$0.00} & 0.98$\pm$0.10 & 1.40$\pm$0.00 & 0.34$\pm$0.01 & 0.61$\pm$0.00 & 0.51$\pm$0.00 & 0.44$\pm$0.00 & 1.99$\pm$0.27 & 0.93$\pm$0.00 & 0.58$\pm$0.00 & 0.87$\pm$0.09 & \blue{0.33$\pm$0.03} & \red{0.32$\pm$0.00} \\
			{\it Faculty: History} & 0.32$\pm$0.00 & 0.57$\pm$0.00 & 2.16$\pm$0.80 & \blue{0.30$\pm$0.01} & 0.57$\pm$0.00 & 0.40$\pm$0.00 & 0.37$\pm$0.00 & 2.13$\pm$0.30 & 0.95$\pm$0.00 & 0.38$\pm$0.00 & 0.84$\pm$0.17 & \red{0.28$\pm$0.01} & \blue{0.30$\pm$0.01} \\
			{\it Basketball (avg)} & \blue{0.78$\pm$0.00} & 1.72$\pm$0.00 & 1.98$\pm$0.00 & 0.91$\pm$0.03 & 0.79$\pm$0.00 & 0.88$\pm$0.00 & 0.88$\pm$0.00 & 1.95$\pm$0.00 & 0.99$\pm$0.00 & 0.89$\pm$0.00 & nan$\pm$nan & 0.80$\pm$0.00 & \red{0.73$\pm$0.00} \\
			{\it Basketball finer (avg)} & \blue{0.81$\pm$0.00} & 1.73$\pm$0.01 & 1.96$\pm$0.00 & 1.39$\pm$0.02 & 0.85$\pm$0.00 & 1.19$\pm$0.00 & 1.15$\pm$0.00 & 1.97$\pm$0.00 & 1.00$\pm$0.00 & 0.90$\pm$0.00 & nan$\pm$nan & 0.84$\pm$0.00 & \red{0.74$\pm$0.00} \\
			{\it Football (avg)} & 0.91$\pm$0.00 & 1.63$\pm$0.12 & 1.20$\pm$0.00 & 0.94$\pm$0.02 & 0.94$\pm$0.00 & 1.07$\pm$0.00 & 1.08$\pm$0.00 & 1.78$\pm$0.03 & 1.00$\pm$0.00 & 0.90$\pm$0.00 & 1.72$\pm$0.09 & \blue{0.81$\pm$0.03} & \red{0.78$\pm$0.02} \\
			{\it Football finer (avg)} & 0.98$\pm$0.00 & 1.68$\pm$0.03 & 1.16$\pm$0.00 & 1.01$\pm$0.02 & 0.93$\pm$0.00 & 1.13$\pm$0.00 & 1.21$\pm$0.00 & 1.91$\pm$0.01 & 1.00$\pm$0.00 & 0.90$\pm$0.00 & 2.06$\pm$0.04 & \blue{0.89$\pm$0.06} & \red{0.82$\pm$0.01} \\
\midrule
{\it HeadToHead} & \blue{0.25$\pm$0.00} & 0.48$\pm$0.00 & 0.50$\pm$0.00 & 0.28$\pm$0.00 & 0.29$\pm$0.00 & 0.37$\pm$0.00 & 0.34$\pm$0.00 & 0.50$\pm$0.01 & 0.45$\pm$0.00 & 0.36$\pm$0.00 & nan$\pm$nan & 0.27$\pm$0.00 & \red{0.24$\pm$0.00} \\
			{\it Finance} & 0.41$\pm$0.00 & 0.50$\pm$0.00 & \red{0.40$\pm$0.00} & 0.45$\pm$0.00 & 0.41$\pm$0.00 & 0.44$\pm$0.00 & 0.44$\pm$0.00 & 0.47$\pm$0.00 & 0.41$\pm$0.00 & 0.41$\pm$0.00 & nan$\pm$nan & 0.41$\pm$0.00 & \red{0.40$\pm$0.00} \\
			{\it Animal} & 0.13$\pm$0.00 & 0.40$\pm$0.06 & 0.58$\pm$0.11 & 0.11$\pm$0.00 & \blue{0.08$\pm$0.00} & 0.14$\pm$0.00 & 0.16$\pm$0.00 & 0.49$\pm$0.00 & 0.26$\pm$0.00 & 0.13$\pm$0.00 & 0.50$\pm$0.08 & 0.10$\pm$0.02 & \red{0.06$\pm$0.00} \\
			{\it Faculty: Business} & \blue{0.10$\pm$0.00} & 0.21$\pm$0.00 & 0.30$\pm$0.00 & \blue{0.10$\pm$0.00} & 0.12$\pm$0.00 & 0.12$\pm$0.00 & 0.12$\pm$0.00 & 0.50$\pm$0.01 & 0.17$\pm$0.00 & 0.12$\pm$0.00 & 0.19$\pm$0.01 & \blue{0.10$\pm$0.00} & \red{0.09$\pm$0.00} \\
			{\it Faculty: CS} & \red{0.08$\pm$0.00} & 0.24$\pm$0.02 & 0.35$\pm$0.00 & \red{0.08$\pm$0.00} & 0.15$\pm$0.00 & 0.13$\pm$0.00 & 0.11$\pm$0.00 & 0.50$\pm$0.07 & 0.23$\pm$0.00 & 0.15$\pm$0.00 & 0.22$\pm$0.02 & \red{0.08$\pm$0.01} & \red{0.08$\pm$0.00} \\
			{\it Faculty: History} & 0.08$\pm$0.00 & 0.14$\pm$0.00 & 0.54$\pm$0.20 & 0.08$\pm$0.00 & 0.15$\pm$0.00 & 0.10$\pm$0.00 & 0.09$\pm$0.00 & 0.53$\pm$0.08 & 0.24$\pm$0.00 & 0.10$\pm$0.00 & 0.21$\pm$0.04 & \red{0.07$\pm$0.00} & \red{0.07$\pm$0.00} \\
			{\it Basketball (avg)} & \blue{0.20$\pm$0.00} & 0.43$\pm$0.00 & 0.49$\pm$0.00 & 0.23$\pm$0.01 & \blue{0.20$\pm$0.00} & 0.22$\pm$0.00 & 0.22$\pm$0.00 & 0.49$\pm$0.00 & 0.25$\pm$0.00 & 0.22$\pm$0.00 & nan$\pm$nan & \blue{0.20$\pm$0.00} & \red{0.18$\pm$0.00} \\
			{\it Basketball finer (avg)} & \blue{0.20$\pm$0.00} & 0.43$\pm$0.00 & 0.49$\pm$0.00 & 0.35$\pm$0.00 & 0.21$\pm$0.00 & 0.30$\pm$0.00 & 0.29$\pm$0.00 & 0.49$\pm$0.00 & 0.25$\pm$0.00 & 0.23$\pm$0.00 & nan$\pm$nan & 0.21$\pm$0.00 & \red{0.18$\pm$0.00} \\
			{\it Football (avg)} & 0.23$\pm$0.00 & 0.41$\pm$0.03 & 0.30$\pm$0.00 & 0.23$\pm$0.01 & 0.24$\pm$0.00 & 0.27$\pm$0.00 & 0.27$\pm$0.00 & 0.44$\pm$0.01 & 0.25$\pm$0.00 & \blue{0.22$\pm$0.00} & 0.43$\pm$0.02 & \blue{0.22$\pm$0.01} & \red{0.21$\pm$0.00} \\
			{\it Football finer (avg)} & 0.25$\pm$0.00 & 0.42$\pm$0.01 & 0.29$\pm$0.00 & 0.25$\pm$0.01 & 0.23$\pm$0.00 & 0.28$\pm$0.00 & 0.30$\pm$0.00 & 0.48$\pm$0.00 & 0.25$\pm$0.00 & \blue{0.22$\pm$0.00} & 0.51$\pm$0.01 & 0.24$\pm$0.01 & \red{0.21$\pm$0.00} \\

\bottomrule
\end{tabular}}
\end{table*}

\begin{table*}[!ht]
\centering
\vspace{-10pt}
\caption{Performance on Kendall Tau based on the lowest $\mathcal{L}_\text{upset, naive}$ on ERO models, averaged over 10 runs with one standard deviation. ``avg" for time series first average over all seasons, then consider mean and standard deviation over the 10 averaged values. The best is marked in \red{bold red} while the second best is in \blue{underline blue}. As MVR does not generate results after one week, we leave it out here.}
\label{tab:synthetic}
\resizebox{1\linewidth}{!}{\begin{tabular}{lrrrrrrrrrrrrrrrrr}
\toprule
Data  & SpringRank & SyncRank & SerialRank & BTL & DavidScore & Eig.Cent. & PageRank & RankCent. & SVD\_RS & SVD\_NRS&  GNNRank-N & GNNRank-P \\
\midrule
{\it ERO(p=0.05, style=uniform,$\eta$=0.1)} & 0.75$\pm$0.00 & 0.04$\pm$0.00 & 0.03$\pm$0.00 & 0.70$\pm$0.01 & \blue{0.77$\pm$0.00} & 0.56$\pm$0.00 & 0.58$\pm$0.00 & 0.01$\pm$0.05 & 0.74$\pm$0.00 & \blue{0.77$\pm$0.00} & 0.76$\pm$0.01 & \red{0.79$\pm$0.01} \\
			{\it ERO(p=0.05, style=gamma,$\eta$=0.2)} & 0.61$\pm$0.00 & 0.01$\pm$0.00 & -0.01$\pm$0.00 & 0.61$\pm$0.00 & \blue{0.74$\pm$0.00} & 0.52$\pm$0.00 & 0.51$\pm$0.00 & -0.01$\pm$0.01 & 0.45$\pm$0.00 & 0.64$\pm$0.00 & 0.52$\pm$0.01 & \red{0.77$\pm$0.00} \\
			{\it ERO(p=0.05, style=uniform,$\eta$=0.3)} & 0.61$\pm$0.00 & 0.05$\pm$0.00 & 0.01$\pm$0.00 & 0.59$\pm$0.01 & \blue{0.68$\pm$0.00} & 0.44$\pm$0.00 & 0.41$\pm$0.00 & 0.05$\pm$0.00 & 0.60$\pm$0.00 & 0.62$\pm$0.00 & 0.62$\pm$0.00 & \red{0.70$\pm$0.02} \\
			{\it ERO(p=0.05, style=gamma,$\eta$=0.4)} & 0.51$\pm$0.00 & 0.08$\pm$0.00 & -0.00$\pm$0.00 & 0.52$\pm$0.00 & \blue{0.65$\pm$0.00} & 0.43$\pm$0.00 & 0.43$\pm$0.00 & 0.09$\pm$0.01 & 0.23$\pm$0.00 & 0.44$\pm$0.00 & 0.38$\pm$0.08 & \red{0.66$\pm$0.01} \\
			{\it ERO(p=1, style=uniform,$\eta$=0.5)} & 0.85$\pm$0.00 & 0.07$\pm$0.00 & \red{0.92$\pm$0.00} & 0.81$\pm$0.03 & 0.91$\pm$0.00 & 0.80$\pm$0.00 & 0.73$\pm$0.00 & 0.24$\pm$0.00 & 0.89$\pm$0.00 & 0.87$\pm$0.00 & 0.90$\pm$0.01 & \red{0.92$\pm$0.00} \\
			{\it ERO(p=1, style=gamma,$\eta$=0.6)} & 0.72$\pm$0.00 & 0.09$\pm$0.00 & \red{0.89$\pm$0.00} & 0.67$\pm$0.01 & 0.88$\pm$0.00 & 0.65$\pm$0.00 & 0.64$\pm$0.00 & 0.05$\pm$0.02 & 0.74$\pm$0.00 & 0.73$\pm$0.00 & 0.77$\pm$0.00 & \red{0.89$\pm$0.00} \\
\bottomrule
\end{tabular}}
\end{table*}

\begin{table*}[tb!]
\centering
\vspace{-10pt}
	\caption{$\mathcal{L}_\text{upset, simple}$ comparison for different variants on real-world data, averaged over 10 runs, and plus/minus one standard deviation. The best for each group (GNNRank-N or GNNRank-P) is marked in \red{bold red} while the second best is in \blue{underline blue}.}
\label{tab:ablation}
\resizebox{1\linewidth}{!}{\begin{tabular}{l|rr p{5em}| rrrrrrrr}
\toprule
Methods&\multicolumn{3}{c|}{GNNRank-N}&\multicolumn{7}{c}{GNNRank-P}\\
Data/Variant &loss sum & $\mathcal{L}_\text{upset,margin} $ & $\mathcal{L}_\text{upset, ratio} $ & loss sum & $\mathcal{L}_\text{upset,margin} $  & $\mathcal{L}_\text{upset, ratio} $   & no pretrain & $\{\alpha_\gamma\}_{\gamma=1}^\Gamma$not trainable &$\Gamma=3$ & $\Gamma=7$ \\
			\midrule
{\it Animal} & \blue{0.43$\pm$0.06} & 0.59$\pm$0.08 & \red{0.41$\pm$0.09} & \red{0.25$\pm$0.00} & \red{0.25$\pm$0.00} & \red{0.25$\pm$0.01} & \red{0.25$\pm$0.00} & \red{0.25$\pm$0.00} & \red{0.25$\pm$0.00} & \red{0.25$\pm$0.00} \\
			{\it Faculty: Business} & \blue{0.40$\pm$0.02} & 0.49$\pm$0.16 & \red{0.38$\pm$0.01} & \red{0.36$\pm$0.00} & \red{0.36$\pm$0.00} & \red{0.36$\pm$0.00} & \red{0.36$\pm$0.00} & \red{0.36$\pm$0.00} & \red{0.36$\pm$0.00} & \red{0.36$\pm$0.00} \\
			{\it Faculty: CS} & \blue{0.35$\pm$0.01} & 0.36$\pm$0.01 & \red{0.33$\pm$0.03} & \red{0.32$\pm$0.00} & \red{0.32$\pm$0.00} & \red{0.32$\pm$0.00} & 0.33$\pm$0.00 & \red{0.32$\pm$0.00} & \red{0.32$\pm$0.00} & \red{0.32$\pm$0.00} \\
			{\it Faculty: History} & \red{0.28$\pm$0.01} & 0.31$\pm$0.01 & \red{0.28$\pm$0.01} & \red{0.30$\pm$0.01} & \red{0.30$\pm$0.01} & \red{0.30$\pm$0.02} & \red{0.30$\pm$0.01} & \red{0.30$\pm$0.01} & \red{0.30$\pm$0.01} & \red{0.30$\pm$0.01} \\
			{\it Football (avg)} & \red{0.82$\pm$0.01} & 0.84$\pm$0.03 & \red{0.82$\pm$0.05} & \blue{0.78$\pm$0.02} & \blue{0.78$\pm$0.01} & 0.79$\pm$0.01 & 0.79$\pm$0.02 & 0.79$\pm$0.02 & \red{0.77$\pm$0.01} & \blue{0.78$\pm$0.02} \\
			{\it Football finer (avg)} & \red{0.90$\pm$0.01} & 0.97$\pm$0.06 & \blue{0.91$\pm$0.07} & \red{0.82$\pm$0.01} & 0.84$\pm$0.01 & \red{0.82$\pm$0.01} & 0.84$\pm$0.02 & \red{0.82$\pm$0.00} & \red{0.82$\pm$0.01} & \red{0.82$\pm$0.01} \\
			\bottomrule
		\end{tabular}}
\end{table*}
\vspace{-1mm}

\subsection{Discussion}
\textbf{Ablation Study.}
Table \ref{tab:ablation} shows results on varying the current choices:
\bb 1) For GNNRank-N methods: forcing the loss to be the sum $\mathcal{L}_\text{upset, simple}+\mathcal{L}_\text{upset, margin}$, $\mathcal{L}_\text{upset, simple}$ only, or $\mathcal{L}_\text{upset, margin}$, respectively; 
\bb 2) for GNNRank-P methods: in addition to the variants for GNNRank-N methods, removing pretraining, fixing all $\alpha_\gamma$ in Algo.~\ref{algo:proximal}, and changing the number of Fiedler proximal steps from default 5 to 3 or 7. 

It is shown that for GNNRank-N methods, using $\mathcal{L}_\text{upset, margin}$  usually harms performance.  For GNNRank-P methods, using $\mathcal{L}_\text{upset, margin}$ in addition to $\mathcal{L}_\text{upset, ratio}$ usually boosts performance.
 Pretraining generally leads to better performance, so does making $\{\alpha_\gamma\}_{\gamma=1}^\Gamma$ trainable. 
The number of proximal gradient steps does not need to be large probably due to fast convergence, so we use 5 throughout. Full comparison tables are in Appendix~\ref{appendix_subsec:ablation_full} including results on $\mathcal{L}_\text{upset, naive}$, with similar conclusions.

In addition, note that essentially we could use any neural network method to obtain the node embeddings, yet digraph GNNs are natural to be employed given the input data structure. To validate the benefit of using a digraph GNN, we adopt a two-layer Multilayer perceptron to obtain node embeddings, but obtained on average 2\% worse $\mathcal{L}_{upset, simple}$ for both the non-proximal and proximal variants across all real-world data sets, respectively.

\textbf{Inductive Learning.} 
We observe that
our proximal methods, if trained only once and then applied to similar data sets, still perform comparably to multiply trained analogs. This can save training time and validates the inductive learning ability for our framework. 
To this end, Appendix~\ref{appendix_subsec:inductive} shows results on the performance of the ``IB proximal baseline" variant, trained with ``emb baseline" on the \textit{Basketball finer} data set. On average, directly applying gives $\mathcal{L}_\text{upset, simple}=0.75\pm 0.02$ and $\mathcal{L}_\text{upset, naive}=0.19\pm 0.01$ while training specifically for the 
season gives $\mathcal{L}_\text{upset, simple}=0.74\pm 0.00$ and $\mathcal{L}_\text{upset, naive}=0.19\pm 0.00$.


\textbf{Variants and Hyperparameters.} 
The results
in Tables \ref{tab:upset_simple_naive} and \ref{tab:synthetic} are selected within either non-proximal or proximal categories, depending on whether they have proximal gradient steps within the architecture. They show
the lowest reported evaluation metric (except for Kendall tau,
when we select variants based on the lowest $\mathcal{L}_\text{upset, naive}$) for all variants within the group. 
Appendix~\ref{appendix_sec:variant_hyper} gives
the variant selected based on minimizing either $\mathcal{L}_\text{upset, simple}$ or $\mathcal{L}_\text{upset, ratio}$ for non-proximal and proximal groups, respectively. We find that each variant has its scenarios where it shows competitive or even outstanding performance;  that ``dist" seems to outperform ``innerproduct" within non-proximal methods; and that ``proximal baseline" is usually the best among proximal methods, with SyncRank output as initial guess, pretrained with a SerialRank similarity matrix. This shows
that our proximal method, initialized
with a good baseline, e.g. Sync-Rank, and pretrained with information from SerialRank, can boost the corresponding baseline method
by using a learnable similarity matrix.

\textbf{Boosting Baselines.} 
Appendix~\ref{appendix_sec:improvement_over_baselines} shows improvements on $\mathcal{L}_{upset, simple}$ and $\mathcal{L}_{upset, naive}$ by ``proximal baseline" when setting a certain baseline as $\mathbf{r}'$. Across all data sets, ``proximal baseline" improves the most (by 1.02 and 0.24, respectively) with SyncRank as initial guess,
while the average improvement for SpringRank, SerialRank, BTL, Eig.Cent., PageRank and SVD\_NRS are 0.07, 0.82, 0.22, 0.19, 0.21, and 0.12, respectively, for $\mathcal{L}_{upset, simple}$, and for 0.00, 0.18, 0.04, 0.03, 0.03, and 0.01, respectively, for $\mathcal{L}_{upset, naive}$.
\section{Conclusion and Outlook}
We have proposed a general framework based on directed graph neural networks to recover global rankings from pairwise comparisons. Future directions include  learning a more powerful model to work for different input digraphs, minimizing upsets under some constraints, training with some supervision of ground-truth rankings, and exploring the interplay with low-rank matrix completion. Incorporating side information, in the form of node level covariates, and comparing to the, currently rather limited, existing literature on ranking with covariates, is another interesting direction. 

\section*{Acknowledgment} 
Yixuan He is supported by a Clarendon scholarship from University of Oxford. This work was partially done during her internship at Amazon and a visit to Shanghai Jiao Tong University.
Gesine Reinert is funded in part by EPSRC grants EP/T018445/1 and EP/R018472/1. Junchi Yan is funded in part by National Key Research and Development of China (2020AAA0107600), and Shanghai Municipal Science and Technology Major Project (2021SHZDZX0102). Gesine Reinert and Mihai Cucuringu acknowledge support from the EPSRC grant EP/N510129/1 at The Alan Turing Institute.

\bibliography{ICML_arxiv_v2}
\bibliographystyle{icml2022}

\newpage
\appendix
\section{Implementation Details}
\label{appendix_sec:implementation}
\subsection{Setup}
We use all data for training for at most 1000 epochs, and stop  early if the loss does not decrease for 200 epochs. For proximal variants, we use 50 epochs for pretraining. We use Adam \citep{kingma2014adam} and  Stochastic Gradient Descend (SGD) as the optimizers and $\ell_2$ regularization with weight decay $5\cdot 10^{-4}$ to avoid overfitting. We use as learning rate 0.01 throughout for GNNRank-N methods as well as pretraining with Adam, and 10 times that of pretraining learning rate for GNNRank-P methods with SGD. We run a grid search on hyperparameters. 

For real-world data sets, we conduct 10 repeated runs, while for synthetic data, we generate 5 synthetic networks under the same setting, each with 2 repeated runs.

Note that we do not evaluate our method by the loss which is used to devise the method, as that would not be fair; instead we employ  $\mathcal{L}_{upset, naive}$ or $\mathcal{L}_{upset, simple}$ which are never used in training. Thus, the comparison is fair. When ground truths are given, we follow \cite{d2021ranking} to use Kendall Tau values for comparison.

\subsection{Codes, Data and Hardware}
To fully  reproduce our results, codes and preprocessed data are available at \url{https://github.com/SherylHYX/GNNRank}. Experiments were conducted on a compute node with 8 Nvidia Tesla T4, 96 Intel Xeon Platinum 8259CL CPUs @ 2.50GHz and $378$GB RAM. Most experiments can be completed within a week, including all variants,  hyperparameter searches and ablation studies, except for those on MVR. 

The data sets considered here are relatively small and the same applies to GNNRank's competitive papers. However, even a network with $100$ nodes has more than $10^{157}$ possible rankings; this large scale task  needs efficient methods. Although each individual task does not require much resource (often $<$ 5min/run), 
for the paper we have 78 real-world and 18 synthetic data sets, 
each requires 10 runs for each of the 36 (proximal innerproduct, proximal dist) + 126 (proximal baseline) + 12 (non-proximal) = 174 variants, thus 167,040 runs, plus an extra ablation study.

\section{Details on Finding $\mathbf{Q}$ for Proximal Gradient Steps}
\label{appendix_sec:find_Q_details}
There are infinitely many choices of valid $\mathbf{Q}$; 
here we construct one of the special $\mathbf{Q}$'s for which
we can compute $\mathbf{Q} \mathbf{L}$ efficiently.
Since $\mathbf{Q}^\top \mathbf{e}_1 = \mathbf{1}/\sqrt{n}$, we can construct $\mathbf{R} = \mathbf{Q}^\top$ as a series of matrix multiplication of $(n-1)$ rotation matrices on two adjacent axes
$$
\mathbf{R} = \mathbf{R}_{n-1} \mathbf{R}_{n-2} \cdots \mathbf{R}_2 \mathbf{R}_1
$$
where $\mathbf{R}_k$ is defined by:
$$
\left[ \mathbf{R}_k \right]_{ij} = \begin{cases}
\sqrt{\frac{1}{n-k}} & i=k,j=k \text{ or } i=k+1, j=k+1 \\
-\sqrt{\frac{n-k-1}{n-k}} & i=k,j=k+1 \\
\sqrt{\frac{n-k-1}{n-k}} & i=k+1, j=k \\
1 & i=j \text{, and } i, j \ne k \text{ or } k + 1 \\
0 & \text{otherwise.}
\end{cases}
$$
To explain this construction, starting from $\mathbf{e}_1 = \begin{bmatrix} 1 & 0 & \cdots & 0 \end{bmatrix}^\top$, 
we carry out a rotation on the first and second axis to make the first element $\sqrt{1 / n}$.  The rotation matrix is
$$
\mathbf{R}_1 = \begin{bmatrix}
\sqrt{\frac{1}{n}} & -\sqrt{\frac{n-1}{n}} & & & \cdots & \\
\sqrt{\frac{n-1}{n}} & \sqrt{\frac{1}{n}} & & & \cdots & \\
& & 1 & & \cdots & \\
& & & 1 & \cdots & \\
\vdots & \vdots & \vdots & \vdots & \ddots & \vdots \\
& & & & & 1
\end{bmatrix} 
$$
We observe that $$\mathbf{R}_1 \mathbf{e}_1 = \begin{bmatrix} \sqrt{\frac{1}{n}} & \sqrt{\frac{n-1}{n}} & 0 & 0 & \cdots & 0 \end{bmatrix}^\top.$$

The second rotation 
occurs on the second and third axis to render $\mathbf{R}_2 \mathbf{R}_1 \mathbf{e}_1 = \begin{bmatrix} \sqrt{\frac{1}{n}} & \sqrt{\frac{1}{n}} & \sqrt{\frac{n-2}{n}} & 0 & \cdots & 0 \end{bmatrix}^\top$; this $\mathbf{R}_2$ is
$$
\mathbf{R}_2 = \begin{bmatrix}
1 & & & & \cdots & \\
& \sqrt{\frac{1}{n-1}} & -\sqrt{\frac{n-2}{n-1}} & & \cdots & \\
& \sqrt{\frac{n-2}{n-1}} & \sqrt{\frac{1}{n-1}} & & \cdots & \\
& & & 1 & \cdots & \\
\vdots & \vdots & \vdots & \vdots & \ddots & \vdots \\
& & & & & 1
\end{bmatrix}
$$
The general matrices $R_k$ are obtained by continuing this construction.

The matrix $\mathbf{R}$ is
\begin{equation*}
\begin{scriptsize}
\begin{bmatrix}
\sqrt\frac{1}{n} & -\sqrt{\frac{n-1}{n}} & & & \cdots & \\
\sqrt\frac{1}{n} & \sqrt{\frac{1}{n(n-1)}} & -\sqrt{\frac{n-2}{n-1}} & & \cdots & \\
\sqrt\frac{1}{n} & \sqrt{\frac{1}{n(n-1)}} & \sqrt{\frac{1}{(n-1)(n-2)}} & -\sqrt{\frac{n-3}{n-2}} & \cdots & \\
\sqrt{\frac{1}{n}} & \sqrt{\frac{1}{n(n-1)}} & \sqrt{\frac{1} {(n-1)(n-2)}} & \sqrt{\frac{1}{(n-2)(n-3)}} & \cdots & \\
\vdots & \vdots & \vdots & \vdots & \ddots & \vdots \\
\sqrt{\frac{1}{n}} & \sqrt{\frac{1}{n(n-1)}} & \sqrt{\frac{1}{(n-1)(n-2)}}& \sqrt{\frac{1}{(n-2)(n-3)}} & \cdots & \sqrt{\frac{1}{2}}
\end{bmatrix}
\end{scriptsize}
\end{equation*}
We then put $\mathbf{Q} = \mathbf{R}^\top$.

The resulting $\mathbf{Q}$ is 
the following upper Hessenberg matrix, which is independent of model parameters and can be efficiently precomputed
\begin{equation*}
\mathbf{Q}_{ij} = \begin{cases}
\sqrt{\frac{1}{n}} & i=1 \\
-\sqrt{\frac{n-i+1}{n-i+2}} & i\geq 2, j=i-1 \\
\sqrt{\frac{1}{(n-i+1)(n-i+2)}} & i\geq 2, j\geq i \\
0 & \text{otherwise.}
\end{cases}
\end{equation*}

The computation of $\mathbf{Q} \mathbf{L}$ takes $O(n^2)$ time since $\mathbf{Q}$ is a summation of (1) an upper-triangular matrix $\mathbf{U}$ with the same non-zero value on the same row, and (2) a shift matrix $\mathbf{V}$ (with negative subdiagonal). Then
$\mathbf{U} \mathbf{L}$ is essentially a cumulative sum from the bottom row to the top row, followed by a row-wise multiplication of the diagonal of $\mathbf{U}$, and 
$\mathbf{V} \mathbf{L}$ is essentially a combination of row-wise multiplication and indexing.

\section{Theoretical Analysis and Practical Considerations on Convergence of the Proximal Gradient Steps}
\label{appendix_sec:convergence_dis} 
First we prove a theorem useful for the main result.
\begin{theorem}
\label{thm:convergence_appendix}
Let $\{\alpha_\gamma>0\}_{\gamma=1}^\Gamma$ in Algo.~\ref{algo:proximal} be fixed and equal to $\alpha$. Let $\rho$ be the Fiedler eigenvalue of $\mathbf{S}$, i.e., the second smallest eigenvalue of $\mathbf{L}$.  Let $\mathbf{r}^*$ be a Fiedler eigenvector corresponding to $\rho,$ let  $\mathbf{y}^*=\left[\mathbf{Q}\mathbf{r}^*\right]_{2:n}$ and $\widetilde{L} = \left[\mathbf{Q L Q}^\top\right]_{2:n,2:n}.$ 
Let $\lambda_1 \geq \lambda_2 \cdots \geq \lambda_{n-1}=0$ be the eigenvalues of $ \mathbf{P}\widetilde{\mathbf{L}}\mathbf{P},$ where $\mathbf{P}=\mathbf{I}-\mathbf{y}^*\mathbf{y}^{*\top}$.
Let $\rho_\alpha=\max_{1\leq i \leq n-2}\frac{|1-2\alpha\lambda_i|}{1-2\alpha\rho}$. If $\mathbf{r}^*$ is a strict local minimum of problem~(\ref{eq:problem_original}), and if $\rho_\alpha<1$ and $\alpha(\lambda_1+\rho)<1,$ then Algo.~\ref{algo:proximal} converges locally uniformly to  $\mathbf{r}^*$.
\end{theorem} 
\begin{proof} 
Our problem~(\ref{eq:problem_reduced}) is a special case of the problem (2) in \cite{vu2019convergence} (henceforth \eqref{eq:problem_vu} further below, of minimizing a quadratic over a sphere), where our fixed point $\mathbf{y}^*$ is determined by $\mathbf{f}_\alpha(\mathbf{\overline{x}})=\mathcal{P}_{\mathcal{S}^{n-2}}(-\alpha(2\widetilde{\mathbf{L}}\mathbf{\overline{x}}))=\mathbf\mathbf{\overline{x}}$. Recall that the spherical projection operator $\mathcal{P}_{\mathcal{S}^{n-2}}(\cdot):\mathbb{R}^{n-1}\rightarrow\mathbb{R}^{n-1}$ is defined to be $\mathcal{P}_{\mathcal{S}^{n-2}}(\mathbf{x})=\frac{\mathbf{x}}{\left\lVert \mathbf{x} \right\rVert_2}$ if $\mathbf{x}\neq \mathbf{0}$, and $\mathcal{P}_{\mathcal{S}^{n-2}}(\mathbf{0})=[1,0,\dots,0]^\top,$ where $\mathbf{0}=[0,\dots,0]^\top.$  
It is shown in Lemma 1 of \cite{vu2019convergence} (henceforth Lemma \ref{lem:vuLemma1} further below) that $\mathbf{y}^*$ with $||\mathbf{y}^*||=1 $ being a solution of~(\ref{eq:problem_reduced}) is equivalent to the existence of a constant   $\rho(\mathbf{\overline{x}})$ such that
$$\widetilde{\mathbf{L}}\mathbf{y}^*
=\rho(\mathbf{y}^*) \cdot \mathbf{y}^*,$$
rendering  $\rho(\mathbf{y}^*)$ to be  an eigenvalue of  $\widetilde{\mathbf{L}}$.  Theorem 1 in \cite{vu2019convergence}  (henceforth Theorem \ref{thm:vuThm1} below)  states that with an initial step size $\alpha>0$,  such that $2\alpha(\rho(
\mathbf{y}^*)+\lambda_1)<2$, as long as $\rho_\alpha=\max_{1\leq i \leq n-2}\frac{|1-2\alpha\lambda_i|}{1-2\alpha\gamma(\mathbf{y}^*})
<1,$ the proximal gradients steps will converge to a local minimum. 

Note that problem~(\ref{eq:problem_original}) and problem~(\ref{eq:problem_reduced}) are equivalent, and  there is a one-to-one correspondence between $\mathbf{r}^*$ and $\mathbf{y}^*, $ namely,  $\mathbf{y}^*=\left[\mathbf{Q}\mathbf{r}^*\right]_{2:n}\in\mathbb{R}^{n-1},$ and $\mathbf{r}^*=\text{CONCAT}(0,\left[\mathbf{Q}^\top\mathbf{r}^*\right])\in\mathbb{R}^n.$ Hence,  
since $\mathbf{r}^*$ is a strict local minimum of problem~(\ref{eq:problem_original}), we have that $\mathbf{y}^*$ is a strict local minimum of problem~(\ref{eq:problem_reduced}).

By 
Lemma \ref{lem:vuLemma1}
from  \cite{vu2019convergence}, we have $\rho<\lambda_{n-2}$, since $\mathbf{y}^*$ is a strict local minimum. In particular,
by 
Theorem \ref{thm:vuThm1}
in \cite{vu2019convergence}, when $\mathbf{y}^*$ is a strict local minimum to problem  \eqref{eq:problem_vu}, 
if $2\alpha(\rho+\lambda_1)<2$, i.e., $\alpha(\rho+\lambda_1)<1,$ there is a constant M such that 
$$ || \mathbf{y}^{\gamma} -\mathbf{y}^* ||  \leq M  || \mathbf{y}^{0} - \mathbf{y}^*|| (\rho_\alpha + o(1))^\gamma,$$
where $\mathbf{y}^{0}$ is the initial guess at proximal gradient step 0, and $\gamma$ denotes the step number. As $\rho_\alpha <1$ is assumed, convergence follows. 
\end{proof}

For convergence,  it suffices to ensure that $\rho_\alpha<1,$ i.e. to have $2\alpha\gamma(\mathbf{\overline{x}})<1$, $|1-2\alpha\lambda_1|<1-\alpha\gamma(\mathbf{\overline{x}})$ and $|1-2\alpha\lambda_{n-2}|<1-2\alpha\gamma(\mathbf{\overline{x}})$. In particular, in our problem we take 
$\rho(\mathbf{\overline{x}})=\rho$, which is also the smallest eigenvalue of $\widetilde{\mathbf{L}},$ so $\rho< \lambda_{n-2}\leq \lambda_1.$ 
To apply Theorem \ref{thm:convergence_appendix} in order to obtain a convergence guarantee in Algo.~\ref{algo:proximal}, as
$\lambda_{n-2}>\rho,$ (which is equivalent to $\mathbf{r}^*$ being a strict local minimum) it remains to use an $\alpha$ value such that $2\alpha\lambda_{1}<1.$ 
Lemma \ref{lem:vuLemma2} 
of  \cite{vu2019convergence} also suggests that  if $2\alpha\lambda_{1}<1$,  then 
we indeed have $\mathbf{r}^*$ as a global minimizer to problem~(\ref{eq:problem_original}).

Finally we have the ingredients to prove the main result.

\begin{theorem}
\label{thm:convergence_main_result_appendix}
Let $\{\alpha_\gamma>0\}_{\gamma=1}^\Gamma$ in Algo.~\ref{algo:proximal} be fixed (equal to $\alpha$) and let $\rho$ be the Fiedler eigenvalue of $\mathbf{S}$.  Denote a Fiedler eigenvector by $\mathbf{r}^*$. Assume that $\mathbf{r}^*$ is a strict local minimizer of problem~(\ref{eq:problem_original}).
If 
$0<\alpha<\frac{1}{4(n-1)},$ 
then with our definition of the similarity matrix,  Algo.~\ref{algo:proximal} converges locally uniformly to $\mathbf{r}^*$.
\end{theorem} 
\begin{proof}
Via our similarity matrix construction process, entries in $\mathbf{S}$ are upper-bounded by 1, so the degree matrix $\mathbf{D}$ has entries bounded above by $n,$ the number of nodes, and lower bounded by 0, and that $\mathbf{L}$ is positive semi-definite. We thus  have that the eigenvalues of $\mathbf{L}$ are in $[0,(n-1)+(n-1)\times 1]$ by the Gershgorin disc theorem.  Since $\mathbf{Q}$ is orthogonal, we have that eigenvalues of $\widetilde{\mathbf{L}}$ are also within the range $[0,2(n-1)].$ To see this, suppose $\mathbf{x}$ is an eigenvector of $\mathbf{L}$ with eigenvalue $\lambda,$ i.e., $\mathbf{Lx}=\lambda\mathbf{x}$, so $\mathbf{QLQ}\top(\mathbf{Qx})=\mathbf{QL}(\mathbf{Q}\top\mathbf{Q})\mathbf{x}=\mathbf{QL}\mathbf{x}=\mathbf{Q}(\mathbf{Lx})=\mathbf{Q}\lambda\mathbf{x}=\lambda(\mathbf{Qx}).$ Therefore, the values $\{\lambda_i\}_{i=1}^{n-2}$ in Theorem~\ref{thm:convergence_appendix} are bounded above by $2(n-1).$ 
Thus, we have that $\rho+\lambda_1\leq 2(n-1)+2(n-1)=4(n-1).$ 

For convergence,  it suffices to ensure that $\rho_\alpha=\max_{1\leq i \leq n-2}\frac{|1-2\alpha\lambda_i|}{1-2\alpha\rho}<1,$ i.e. to have $2\alpha\rho<1$, $|1-2\alpha\lambda_1|<1-2\alpha\rho$,  and $|1-\alpha\lambda_{n-2}|<1-2\alpha\rho$. 
 
Thus, since we have $\lambda_{n-2}>\rho$ by 
Lemma \ref{lem:vuLemma2} 
from  \cite{vu2019convergence} and the proof of Theorem~\ref{thm:convergence_appendix}, it remains to use an $\alpha$ value such that $2\alpha\lambda_{1}<1,$ then $2\alpha\rho<2\alpha\lambda_1<1$, and $0<1-2\alpha\lambda_1=|1-2\alpha\lambda_1|\leq 1-2\alpha\lambda_{n-2}=|1-2\alpha\lambda_{n-2}|< 1-2\alpha\rho.$ 
With $\alpha<\frac{1}{4(n-1)},$ we indeed have $2\alpha\lambda_{1}<2\times\frac{1}{4(n-1)}\times 2(n-1)=1.$
\end{proof}
 
Our current setting is to initialize all $\alpha$ values to be $\frac{1}{n-1}$. 
Note that within Algo.~\ref{algo:proximal}, the $\alpha$ values are fixed. Although not guaranteed to converge with the initial $\alpha$,   they could  be adapted by our outer training loop, i.e. by optimization over the ranking loss function. 
\cite{vu2019convergence} also suggests the value of the optimal step size to be $\frac{2}{\lambda_1+\lambda_{n-2}},$ which during training our method could in principle reproduce if the optimal step size is actually important for making accurate final rankings.

We remark that real-world weighted networks have the corresponding Fiedler eigenvalue of multiplicity 1; this multiplicity is larger than 1, only for 
special classes of graphs that exhibit certain symmetries. Our pragmatic assumption that the Fiedler eigenvalue has multiplicity 1, which thus means $\mathbf{r}^*$ being a strict local minimum of problem~\ref{eq:problem_original}, is not a very restrictive one.

For completeness, the remainder of this section recalls the setup and main results from \cite{vu2019convergence}. The problem considered therein is that of minimizing a quadratic form over the sphere
\begin{align}
\label{eq:problem_vu}
\min_{\mathbf{x}} \mathbf{x}^\top  \mathbf{Bx}- \mathbf{b}^\top\mathbf{x}
\quad  \text{ s.t.} & \quad \left\lVert \mathbf{x} \right\rVert_2^2 = 1
\end{align} 
\noindent with the matrix $\mathbf{B}\in\mathbb{R}^{m\times m}$ assumed symmetric but not positive semi-definite, hence a non-convex objective function. Upon considering the corresponding Lagrangian function, with $\nu$ denoting the Lagrange multiplier, the authors prove the following results, which also constitute building blocks in our analysis. The eigenvalues of $\mathbf{A}$ are given by $ \lambda_{min}(\mathbf{B}) \leq \lambda_{m-1} \leq \ldots \leq \lambda_1 \leq \lambda_{max}(\mathbf{B})$.

\begin{lemma}[Lemma 1 from \cite{vu2019convergence}]
\label{lem:vuLemma1}
(Stationary conditions). The vector $\mathbf{x}_{*}$  is a stationary point of problem \eqref{eq:problem_vu} if and only if $\mathbf{x}_{*} \in \mathcal{S}^{m-1}=\{\mathbf{x}\in\mathbb{R}^m:\lVert \mathbf{x}\rVert=1\} $ and there exists a constant 
$ \nu(\mathbf{x}_{*}) $
such that $\mathbf{r}_{*} = A \mathbf{x}_{*} - b = \nu(\mathbf{x}_{*}) \cdot   \mathbf{x}_{*} $.
\end{lemma}

\begin{lemma}[Lemma 2 from \cite{vu2019convergence}]
\label{lem:vuLemma2}
A stationary point $\mathbf{x}_{*}$ of problem \eqref{eq:problem_vu} is a \underline{strict} local minimum if and only if 
$ \nu(\mathbf{x}_{*})  < \lambda_{m-1}(\mathbf{x}_{*})$. Furthermore, $\mathbf{x}_{*}$  is a \textbf{global minimizer} of problem \eqref{eq:problem_vu} if and only if 
$ \nu(\mathbf{x}_{*}) \leq \lambda_{min}(\mathbf{B}). $ 
\end{lemma}

The following result establishes the result that 
Projected Gradient Descent (PGD)  converges to a local minimum at an asymptotic linear rate $\rho_{\beta}$.

\begin{theorem}[Theorem 1 from \cite{vu2019convergence})]
\label{thm:vuThm1}
The vector $\mathbf{x}_{*}$ is a strict local minimum of problem \eqref{eq:problem_vu}, i.e. $ \nu < \lambda_{m-1}$, if and only if there exists $ \beta >0 $ such that Algorithm 1 [Projected Gradient Descent (PGD)] with step size $\alpha$ converges locally uniformly to $\mathbf{x}_{*}$. Furthermore, for any step size $ \beta >0 $  such that $\beta(\lambda_1 + \nu)<2$, the sequence $\{ \mathbf{x}^{(t)} \}$  satisfies
$$
\lVert \mathbf{x}^{(t)} - \mathbf{x}_{*} \rVert 
\leq M \lVert \mathbf{x}^{(0)} - \mathbf{x}_{*} \rVert  \left(  \rho_{\alpha} + o(1) \right)^t,  
$$
for some constant $M>0$ and $\rho_{\beta} = \max_{1 \leq i \leq m-1} \frac{| 1- \beta \lambda_i |}{ 1- \beta \nu }. $ 
\end{theorem}

\section{Detailed Summary Statistics}
\label{appendix_sec:full_data_stats}
Table \ref{tab:data sets full} gives the number of nodes ($n$), the number of directed edges ($|\mathcal{E}|$), the number of reciprocal edges ($|\mathcal{E}^r|$) (self-loops are counted once and for $ u \ne v$, a reciprocal edge $u\rightarrow v, v\rightarrow u$ is counted twice) as well as their percentage among all edges, for the real-world networks, illustrating the variability in network size and density (defined as ${|\mathcal{E}|}/[{n(n-1)}]$). 
As we do not have input features available, we use the eigengap of the Hermitian matrix $(\mathbf{A}-\mathbf{A}^\top)\cdot i,$ where $\mathbf{A}$ is the adjacency matrix and $i$
the imaginary unit, introduced in~\cite{cucuringu2020hermitian}, to determine the value $K,$ which is assumed to be the number of clusters if we are solving a clustering problem.
We then stack the real and imaginary parts of the top $K$ eigenvectors of $(\mathbf{A}-\mathbf{A}^\top)\cdot i$ into $2K$-dimensional input features for GNNs. We also report the embedding dimension $d$ used in the experiments, for each of the data sets.

\begin{table}[tb!]
\centering
\caption{Summary statistics 
for the real-world networks.
}
\label{tab:data sets full}
\resizebox{0.9\linewidth}{!}{
\begin{tabular}{lrrrrrrr}
\toprule
Data                & $n$ &  $|\mathcal{E}|$&density&$|\mathcal{E}^r|$&$\frac{|\mathcal{E}^r|}{|\mathcal{E}|}$(\%)&$K$&$d$  \\ \midrule
\textit{HeadToHead}&602&5010&1.38e-02&464&9.26&48&32\\
\textit{Basketball (1985)}&282&2904&3.66e-02&998&34.37&20&16\\
\textit{Basketball finer (1985)}&282&4814&6.08e-02&4814&100.00&20&16\\
\textit{Basketball (1986)}&283&2937&3.68e-02&1014&34.53&20&16\\
\textit{Basketball finer (1986)}&283&4862&6.09e-02&4862&100.00&20&16\\
\textit{Basketball (1987)}&290&3045&3.63e-02&1012&33.23&20&16\\
\textit{Basketball finer (1987)}&290&5088&6.07e-02&5088&100.00&20&16\\
\textit{Basketball (1988)}&290&3099&3.70e-02&1034&33.37&20&16\\
\textit{Basketball finer (1988)}&290&5170&6.17e-02&5170&100.00&20&16\\
\textit{Basketball (1989)}&293&3162&3.70e-02&1014&32.07&20&16\\
\textit{Basketball finer (1989)}&293&5318&6.22e-02&5318&100.00&20&16\\
\textit{Basketball (1990)}&292&3192&3.76e-02&1042&32.64&20&16\\
\textit{Basketball finer (1990)}&292&5350&6.30e-02&5350&100.00&20&16\\
\textit{Basketball (1991)}&295&3218&3.71e-02&1018&31.63&20&16\\
\textit{Basketball finer (1991)}&295&5420&6.25e-02&5420&100.00&20&16\\
\textit{Basketball (1992)}&298&3238&3.66e-02&1036&32.00&20&16\\
\textit{Basketball finer (1992)}&298&5444&6.15e-02&5444&100.00&20&16\\
\textit{Basketball (1993)}&298&3088&3.49e-02&1024&33.16&20&16\\
\textit{Basketball finer (1993)}&298&5160&5.83e-02&5160&100.00&20&16\\
\textit{Basketball (1994)}&301&3144&3.48e-02&1044&33.21&20&16\\
\textit{Basketball finer (1994)}&301&5252&5.82e-02&5252&100.00&20&16\\
\textit{Basketball (1995)}&302&3182&3.50e-02&1034&32.50&20&16\\
\textit{Basketball finer (1995)}&302&5336&5.87e-02&5336&100.00&20&16\\
\textit{Basketball (1996)}&305&3256&3.51e-02&1026&31.51&20&16\\
\textit{Basketball finer (1996)}&305&5498&5.93e-02&5498&100.00&20&16\\
\textit{Basketball (1997)}&305&3333&3.59e-02&1044&31.32&20&16\\
\textit{Basketball finer (1997)}&305&5628&6.07e-02&5628&100.00&20&16\\
\textit{Basketball (1998)}&306&3321&3.56e-02&966&29.09&20&16\\
\textit{Basketball finer (1998)}&306&5684&6.09e-02&5684&100.00&20&16\\
\textit{Basketball (1999)}&310&3385&3.53e-02&998&29.48&20&16\\
\textit{Basketball finer (1999)}&310&5788&6.04e-02&5788&100.00&20&16\\
\textit{Basketball (2000)}&318&3475&3.45e-02&852&24.52&20&16\\
\textit{Basketball finer (2000)}&318&6274&6.22e-02&6274&100.00&20&16\\
\textit{Basketball (2001)}&318&3405&3.38e-02&904&26.55&20&16\\
\textit{Basketball finer (2001)}&318&6116&6.07e-02&6116&100.00&20&16\\
\textit{Basketball (2002)}&321&3505&3.41e-02&976&27.85&20&16\\
\textit{Basketball finer (2002)}&321&6192&6.03e-02&6192&100.00&20&16\\
\textit{Basketball (2003)}&327&3560&3.34e-02&954&26.80&20&16\\
\textit{Basketball finer (2003)}&327&6356&5.96e-02&6356&100.00&20&16\\
\textit{Basketball (2004)}&326&3527&3.33e-02&952&26.99&20&16\\
\textit{Basketball finer (2004)}&326&6316&5.96e-02&6316&100.00&20&16\\
\textit{Basketball (2005)}&330&3622&3.34e-02&946&26.12&20&16\\
\textit{Basketball finer (2005)}&330&6476&5.96e-02&6476&100.00&20&16\\
\textit{Basketball (2006)}&334&3695&3.32e-02&924&25.01&20&16\\
\textit{Basketball finer (2006)}&334&6680&6.01e-02&6680&100.00&20&16\\
\textit{Basketball (2007)}&336&3974&3.53e-02&976&24.56&20&16\\
\textit{Basketball finer (2007)}&336&7186&6.38e-02&7186&100.00&20&16\\
\textit{Basketball (2008)}&342&4051&3.47e-02&972&23.99&20&16\\
\textit{Basketball finer (2008)}&342&7386&6.33e-02&7386&100.00&20&16\\
\textit{Basketball (2009)}&347&4155&3.46e-02&1046&25.17&20&16\\
\textit{Basketball finer (2009)}&347&7478&6.23e-02&7478&100.00&20&16\\
\textit{Basketball (2010)}&347&4133&3.44e-02&916&22.16&20&16\\
\textit{Basketball finer (2010)}&347&7538&6.28e-02&7538&100.00&20&16\\
\textit{Basketball (2011)}&345&4086&3.44e-02&950&23.25&20&16\\
\textit{Basketball finer (2011)}&345&7504&6.32e-02&7504&100.00&20&16\\
\textit{Basketball (2012)}&345&4126&3.48e-02&950&23.02&20&16\\
\textit{Basketball finer (2012)}&345&7580&6.39e-02&7580&100.00&20&16\\
\textit{Basketball (2013)}&347&4153&3.46e-02&960&23.12&20&16\\
\textit{Basketball finer (2013)}&347&7616&6.34e-02&7616&100.00&20&16\\
\textit{Basketball (2014)}&351&4196&3.42e-02&1008&24.02&20&16\\
\textit{Basketball finer (2014)}&351&7650&6.23e-02&7650&100.00&20&16\\
\textit{Football (2009)}&20&215&5.66e-01&78&36.28&9&8\\
\textit{Football finer (2009)}&20&380&1.00e+00&380&100.00&9&8\\
\textit{Football (2010)}&20&219&5.76e-01&86&39.27&9&8\\
\textit{Football finer (2010)}&20&380&1.00e+00&380&100.00&9&8\\
\textit{Football (2011)}&20&226&5.95e-01&92&40.71&9&8\\
\textit{Football finer (2011)}&20&380&1.00e+00&380&100.00&9&8\\
\textit{Football (2012)}&20&216&5.68e-01&86&39.81&9&8\\
\textit{Football finer (2012)}&20&380&1.00e+00&380&100.00&9&8\\
\textit{Football (2013)}&20&222&5.84e-01&82&36.94&9&8\\
\textit{Football finer (2013)}&20&380&1.00e+00&380&100.00&9&8\\
\textit{Football (2014)}&20&107&2.82e-01&0&0.00&9&8\\
\textit{Football finer (2014)}&20&300&7.89e-01&300&100.00&9&8\\
\textit{Football(avg)}&20&201&5.29e-01&71&32.17&9&8\\
\textit{Basketball(avg)}&316&3506&3.51e-02&986&28.57&20&16\\
\textit{Football finer(avg)}&20&367&9.65e-01&367&100&9&8\\
\textit{Basketball finer(avg)}&316&6139&6.12e-02&6139&100&20&16\\
\textit{Animal}&21&193&4.60e-01&64&33.16&3&8\\
\textit{Finance}&1315&1729225&1.00e+00&1729225&100&20&64\\
\textit{Faculty:Business}&113&1787&1.41e-01&0&0.00&5&16\\
\textit{Faculty:CS}&206&1407&3.33e-02&0&0.00&9&16\\
\textit{Faculty:History}&145&1204&5.77e-02&0&0.00&12&16\\
\bottomrule
\end{tabular}}
\end{table}
\section{Full Result Tables}
\label{appendix_sec:full_result_tables}

\subsection{Results on Individual Digraphs and $\mathcal{L}_\text{upset, ratio}$}
\label{appendix_subsec:individuals}
Tables \ref{tab:upset_simple_matches}, \ref{tab:upset_naive_matches} and \ref{tab:upset_ratio_matches} provide detailed comparison on both $\mathcal{L}_\text{upset, simple}$, $\mathcal{L}_\text{upset, naive}$ and $\mathcal{L}_\text{upset, ratio}$ for all time-series match data, where Table \ref{tab:upset_ratio_matches} additionally provides $\mathcal{L}_\text{upset, ratio}$ results on other real-world data sets. Again the best-performing variants of the non-proximal and of the proximal method are reported. The proximal method does not perform as well as the non-proximal method in terms of $\mathcal{L}_\text{upset, ratio}$, perhaps partially due to the fact that outputs from SerialRank \cite{fogel2014SerialRank} do not align well with skill level but rather 
relate to relative ordering. It is also natural for GNNRank-N methods that are not motivated by SerialRank to perform well on this metric as $\mathcal{L}_\text{upset, ratio}$ is in the training loss function. Therefore, here we only
present the results on $\mathcal{L}_\text{upset, ratio}$ but do not 
draw conclusions on which method outperforms others.

\begin{table*}[!ht]
\centering
\caption{Result table on $\mathcal{L}_\text{upset, simple}$ for individual directed graphs, averaged over 10 runs, and plus/minus one standard deviation. The best is marked in \red{bold red} while the second best is highlighted in \blue{underline blue}. When MVR could not generate results after a week, we omit the results and fill in ``NAN" here.}
\label{tab:upset_simple_matches}
\resizebox{1\linewidth}{!}{
\begin{tabular}{lrrrrrrrrrrrrrrrrr}
\toprule
Data  & SpringRank & SyncRank & SerialRank & BTL & DavidScore & Eig.Cent. & PageRank & RankCent. &SVD\_RS & SVD\_NRS&MVR  &  GNNRank-N & GNNRank-P \\
\midrule
{\it HeadToHead} & 1.00$\pm$0.00 & 1.94$\pm$0.00 & 2.01$\pm$0.00 & 1.12$\pm$0.01 & 1.16$\pm$0.00 & 1.47$\pm$0.00 & 1.36$\pm$0.00 & 2.00$\pm$0.02 & 1.79$\pm$0.00 & 1.42$\pm$0.00 & nan$\pm$nan & \blue{0.99$\pm$0.00} & \red{0.96$\pm$0.00} \\
			{\it Finance} & 1.63$\pm$0.00 & 1.98$\pm$0.00 & 1.61$\pm$0.00 & 1.78$\pm$0.01 & 1.63$\pm$0.00 & 1.74$\pm$0.00 & 1.75$\pm$0.00 & 1.88$\pm$0.00 & 1.64$\pm$0.00 & 1.64$\pm$0.00 & nan$\pm$nan & \red{1.00$\pm$0.00} & \red{1.00$\pm$0.00} \\
			{\it Animal} & 0.50$\pm$0.00 & 1.62$\pm$0.24 & 1.98$\pm$0.48 & 0.45$\pm$0.00 & \blue{0.33$\pm$0.00} & 0.55$\pm$0.00 & 0.63$\pm$0.00 & 1.96$\pm$0.00 & 1.03$\pm$0.00 & 0.53$\pm$0.00 & 2.02$\pm$0.32 & 0.41$\pm$0.09 & \red{0.25$\pm$0.00} \\
			{\it Faculty: Business} & 0.41$\pm$0.00 & 0.83$\pm$0.00 & 1.19$\pm$0.00 & 0.41$\pm$0.01 & 0.49$\pm$0.00 & 0.49$\pm$0.00 & 0.49$\pm$0.00 & 2.01$\pm$0.03 & 0.68$\pm$0.00 & 0.46$\pm$0.00 & 0.78$\pm$0.05 & \blue{0.38$\pm$0.01} & \red{0.36$\pm$0.00} \\
			{\it Faculty: CS} & \blue{0.33$\pm$0.00} & 0.98$\pm$0.10 & 1.40$\pm$0.00 & 0.34$\pm$0.01 & 0.61$\pm$0.00 & 0.51$\pm$0.00 & 0.44$\pm$0.00 & 1.99$\pm$0.27 & 0.93$\pm$0.00 & 0.58$\pm$0.00 & 0.87$\pm$0.09 & \blue{0.33$\pm$0.03} & \red{0.32$\pm$0.00} \\
			{\it Faculty: History} & 0.32$\pm$0.00 & 0.57$\pm$0.00 & 2.16$\pm$0.80 & \blue{0.30$\pm$0.01} & 0.57$\pm$0.00 & 0.40$\pm$0.00 & 0.37$\pm$0.00 & 2.13$\pm$0.30 & 0.95$\pm$0.00 & 0.38$\pm$0.00 & 0.84$\pm$0.17 & \red{0.28$\pm$0.01} & \blue{0.30$\pm$0.01} \\
			{\it Basketball(1985)} & 0.76$\pm$0.00 & 1.62$\pm$0.00 & 2.11$\pm$0.00 & 0.93$\pm$0.08 & \blue{0.75$\pm$0.00} & 0.89$\pm$0.00 & 0.91$\pm$0.00 & 1.92$\pm$0.00 & 1.00$\pm$0.00 & 0.86$\pm$0.00 & nan$\pm$nan & 0.77$\pm$0.01 & \red{0.71$\pm$0.00} \\
\bottomrule
\end{tabular}}
\end{table*}
\begin{table*}[!ht]
\centering
\caption{Result table on $\mathcal{L}_\text{upset, naive}$ for individual directed graphs, averaged over 10 runs, and plus/minus one standard deviation. The best is marked in \red{bold red} while the second best is highlighted in \blue{underline blue}. When MVR could not generate results after a week, we omit the results and fill in ``NAN" here.}
\label{tab:upset_naive_matches}
\resizebox{1\linewidth}{!}{\begin{tabular}{lrrrrrrrrrrrrrrrrr}
\toprule
Data  & SpringRank & SyncRank & SerialRank & BTL & DavidScore & Eig.Cent. & PageRank & RankCent. &SVD\_RS & SVD\_NRS&MVR  &  GNNRank-N & GNNRank-P \\
\midrule
{\it HeadToHead} & \blue{0.25$\pm$0.00} & 0.48$\pm$0.00 & 0.50$\pm$0.00 & 0.28$\pm$0.00 & 0.29$\pm$0.00 & 0.37$\pm$0.00 & 0.34$\pm$0.00 & 0.50$\pm$0.01 & 0.45$\pm$0.00 & 0.36$\pm$0.00 & nan$\pm$nan & 0.27$\pm$0.00 & \red{0.24$\pm$0.00} \\
			{\it Finance} & 0.41$\pm$0.00 & 0.50$\pm$0.00 & \red{0.40$\pm$0.00} & 0.45$\pm$0.00 & 0.41$\pm$0.00 & 0.44$\pm$0.00 & 0.44$\pm$0.00 & 0.47$\pm$0.00 & 0.41$\pm$0.00 & 0.41$\pm$0.00 & nan$\pm$nan & 0.41$\pm$0.00 & \red{0.40$\pm$0.00} \\
			{\it Animal} & 0.13$\pm$0.00 & 0.40$\pm$0.06 & 0.58$\pm$0.11 & 0.11$\pm$0.00 & \blue{0.08$\pm$0.00} & 0.14$\pm$0.00 & 0.16$\pm$0.00 & 0.49$\pm$0.00 & 0.26$\pm$0.00 & 0.13$\pm$0.00 & 0.50$\pm$0.08 & 0.10$\pm$0.02 & \red{0.06$\pm$0.00} \\
			{\it Faculty: Business} & \blue{0.10$\pm$0.00} & 0.21$\pm$0.00 & 0.30$\pm$0.00 & \blue{0.10$\pm$0.00} & 0.12$\pm$0.00 & 0.12$\pm$0.00 & 0.12$\pm$0.00 & 0.50$\pm$0.01 & 0.17$\pm$0.00 & 0.12$\pm$0.00 & 0.19$\pm$0.01 & \blue{0.10$\pm$0.00} & \red{0.09$\pm$0.00} \\
			{\it Faculty: CS} & \red{0.08$\pm$0.00} & 0.24$\pm$0.02 & 0.35$\pm$0.00 & \red{0.08$\pm$0.00} & 0.15$\pm$0.00 & 0.13$\pm$0.00 & 0.11$\pm$0.00 & 0.50$\pm$0.07 & 0.23$\pm$0.00 & 0.15$\pm$0.00 & 0.22$\pm$0.02 & \red{0.08$\pm$0.01} & \red{0.08$\pm$0.00} \\
			{\it Faculty: History} & 0.08$\pm$0.00 & 0.14$\pm$0.00 & 0.54$\pm$0.20 & 0.08$\pm$0.00 & 0.15$\pm$0.00 & 0.10$\pm$0.00 & 0.09$\pm$0.00 & 0.53$\pm$0.08 & 0.24$\pm$0.00 & 0.10$\pm$0.00 & 0.21$\pm$0.04 & \red{0.07$\pm$0.00} & \red{0.07$\pm$0.00} \\
			{\it Basketball(1985)} & \blue{0.19$\pm$0.00} & 0.40$\pm$0.00 & 0.53$\pm$0.00 & 0.23$\pm$0.02 & \blue{0.19$\pm$0.00} & 0.22$\pm$0.00 & 0.23$\pm$0.00 & 0.48$\pm$0.00 & 0.25$\pm$0.00 & 0.22$\pm$0.00 & nan$\pm$nan & \blue{0.19$\pm$0.00} & \red{0.18$\pm$0.00} \\
\bottomrule
\end{tabular}}
\end{table*}
\begin{table*}[!ht]
\centering
\vspace{-10pt}
\caption{Result table on $\mathcal{L}_\text{upset, ratio}$ for individual directed graphs as input, averaged over 10 runs, and plus/minus one standard deviation. The best is marked in \red{bold red} while the second best is highlighted in \blue{underline blue}. MVR could not generate scores, so we omit the results.}
\label{tab:upset_ratio_matches}
\resizebox{1\linewidth}{!}{\begin{tabular}{lrrrrrrrrrrrrrrrrr}
\toprule
Data  & SpringRank & SyncRank & SerialRank & BTL & DavidScore & Eig.Cent. & PageRank & RankCent. & SVD\_RS & SVD\_NRS& GNNRank-N & GNNRank-P \\
\midrule
{\it HeadToHead} & 0.86$\pm$0.00 & 0.97$\pm$0.00 & 0.97$\pm$0.00 & \blue{0.71$\pm$0.01} & 0.89$\pm$0.00 & 0.98$\pm$0.04 & 0.83$\pm$0.00 & 1.19$\pm$0.09 & 1.14$\pm$0.00 & 0.95$\pm$0.00 & \red{0.68$\pm$0.00} & 0.94$\pm$0.00 \\
			{\it Finance} & \red{0.26$\pm$0.00} & 0.27$\pm$0.00 & 0.27$\pm$0.00 & 0.28$\pm$0.00 & 0.64$\pm$0.00 & 0.27$\pm$0.00 & \red{0.26$\pm$0.00} & 0.59$\pm$0.00 & \red{0.26$\pm$0.00} & \red{0.26$\pm$0.00} & \red{0.26$\pm$0.00} & 0.27$\pm$0.00 \\
			{\it Animal} & 0.42$\pm$0.00 & 0.80$\pm$0.02 & 0.88$\pm$0.00 & \red{0.22$\pm$0.00} & 0.41$\pm$0.00 & 0.43$\pm$0.14 & 0.50$\pm$0.00 & 1.34$\pm$0.00 & 0.62$\pm$0.00 & 0.42$\pm$0.00 & \blue{0.24$\pm$0.03} & 0.66$\pm$0.00 \\
			{\it Faculty: Business} & 0.67$\pm$0.00 & 0.98$\pm$0.00 & 0.99$\pm$0.00 & \blue{0.34$\pm$0.00} & 0.52$\pm$0.00 & 0.74$\pm$0.25 & 0.47$\pm$0.00 & 1.40$\pm$0.13 & 0.64$\pm$0.00 & 0.62$\pm$0.00 & \red{0.31$\pm$0.00} & 0.89$\pm$0.00 \\
			{\it Faculty: CS} & 0.65$\pm$0.00 & 0.99$\pm$0.00 & 1.00$\pm$0.00 & \blue{0.28$\pm$0.00} & 0.51$\pm$0.00 & 0.63$\pm$0.24 & 0.42$\pm$0.00 & 1.33$\pm$0.16 & 0.73$\pm$0.00 & 0.72$\pm$0.00 & \red{0.26$\pm$0.02} & 0.86$\pm$0.00 \\
			{\it Faculty: History} & 0.61$\pm$0.00 & 0.99$\pm$0.00 & 1.00$\pm$0.00 & \blue{0.23$\pm$0.00} & 0.49$\pm$0.00 & 0.60$\pm$0.28 & 0.37$\pm$0.00 & 1.35$\pm$0.11 & 0.68$\pm$0.00 & 0.72$\pm$0.00 & \red{0.21$\pm$0.00} & 0.84$\pm$0.02 \\
			{\it Basketball(1985)} & 0.67$\pm$0.00 & 0.86$\pm$0.00 & 0.86$\pm$0.00 & \blue{0.54$\pm$0.04} & 0.62$\pm$0.00 & 0.67$\pm$0.17 & 0.55$\pm$0.00 & 1.32$\pm$0.00 & 0.67$\pm$0.00 & 0.63$\pm$0.00 & \red{0.46$\pm$0.01} & 0.82$\pm$0.00 \\
\bottomrule
\end{tabular}}
\end{table*}
\subsection{Full Results on Synthetic Data}
\label{appendix_subsec:synthetic_results_full}
Table~\ref{tab:synthetic_kendalltau_full} shows Kandall Tau values and corresponding $\mathcal{L}_\text{upset, naive}$ on ERO models, extending the results of Table~\ref{tab:synthetic} in the main text.
\begin{table*}[!ht]
\centering
\vspace{-10pt}
\caption{Performance on Kendall Tau (top half) based on the lowest $\mathcal{L}_\text{upset, naive}$ (bottom half for corresponding values) on ERO models, averaged over 10 runs with one standard deviation. ``avg" for time series first average over all seasons, then consider mean and standard deviation over the 10 averaged values. The best is marked in \red{bold red} while the second best is in \blue{underline blue}. As MVR does not generate results after one week, we leave it out here.}
\label{tab:synthetic_kendalltau_full}
\resizebox{1\linewidth}{!}{\begin{tabular}{lrrrrrrrrrrrrrrrrr}
\toprule
Data  & SpringRank & SyncRank & SerialRank & BTL & DavidScore & Eig.Cent. & PageRank & RankCent. & SVD\_RS & SVD\_NRS&  GNNRank-N & GNNRank-P \\
\midrule
{\it ERO(p=0.05, style=uniform,$\eta$=0)} & 0.86$\pm$0.00 & 0.06$\pm$0.00 & 0.03$\pm$0.00 & 0.79$\pm$0.05 & 0.80$\pm$0.00 & \red{0.91$\pm$0.00} & 0.87$\pm$0.00 & 0.02$\pm$0.09 & 0.78$\pm$0.00 & 0.83$\pm$0.00 & \blue{0.88$\pm$0.00} & 0.86$\pm$0.00 \\
			{\it ERO(p=0.05, style=gamma,$\eta$=0)} & 0.81$\pm$0.00 & 0.06$\pm$0.00 & 0.00$\pm$0.00 & 0.87$\pm$0.01 & 0.81$\pm$0.00 & \red{0.92$\pm$0.00} & \blue{0.90$\pm$0.00} & -0.00$\pm$0.01 & 0.58$\pm$0.00 & 0.79$\pm$0.00 & 0.87$\pm$0.00 & 0.84$\pm$0.01 \\
			{\it ERO(p=0.05, style=uniform,$\eta$=0.1)} & 0.75$\pm$0.00 & 0.04$\pm$0.00 & 0.03$\pm$0.00 & 0.70$\pm$0.01 & \blue{0.77$\pm$0.00} & 0.56$\pm$0.00 & 0.58$\pm$0.00 & 0.01$\pm$0.05 & 0.74$\pm$0.00 & \blue{0.77$\pm$0.00} & 0.76$\pm$0.01 & \red{0.79$\pm$0.01} \\
			{\it ERO(p=0.05, style=gamma,$\eta$=0.1)} & 0.69$\pm$0.00 & 0.04$\pm$0.00 & 0.01$\pm$0.00 & 0.68$\pm$0.00 & \blue{0.78$\pm$0.00} & 0.58$\pm$0.00 & 0.60$\pm$0.00 & 0.06$\pm$0.00 & 0.52$\pm$0.00 & 0.72$\pm$0.00 & 0.61$\pm$0.01 & \red{0.81$\pm$0.01} \\
			{\it ERO(p=0.05, style=uniform,$\eta$=0.2)} & 0.68$\pm$0.00 & 0.07$\pm$0.00 & 0.03$\pm$0.00 & 0.64$\pm$0.01 & \blue{0.73$\pm$0.00} & 0.50$\pm$0.00 & 0.48$\pm$0.00 & 0.02$\pm$0.04 & 0.67$\pm$0.00 & 0.70$\pm$0.00 & 0.68$\pm$0.01 & \red{0.76$\pm$0.01} \\
			{\it ERO(p=0.05, style=gamma,$\eta$=0.2)} & 0.61$\pm$0.00 & 0.01$\pm$0.00 & -0.01$\pm$0.00 & 0.61$\pm$0.00 & \blue{0.74$\pm$0.00} & 0.52$\pm$0.00 & 0.51$\pm$0.00 & -0.01$\pm$0.01 & 0.45$\pm$0.00 & 0.64$\pm$0.00 & 0.52$\pm$0.01 & \red{0.77$\pm$0.00} \\
			{\it ERO(p=0.05, style=uniform,$\eta$=0.3)} & 0.61$\pm$0.00 & 0.05$\pm$0.00 & 0.01$\pm$0.00 & 0.59$\pm$0.01 & \blue{0.68$\pm$0.00} & 0.44$\pm$0.00 & 0.41$\pm$0.00 & 0.05$\pm$0.00 & 0.60$\pm$0.00 & 0.62$\pm$0.00 & 0.62$\pm$0.00 & \red{0.70$\pm$0.02} \\
			{\it ERO(p=0.05, style=gamma,$\eta$=0.3)} & 0.56$\pm$0.00 & 0.11$\pm$0.00 & -0.00$\pm$0.00 & 0.56$\pm$0.01 & \blue{0.71$\pm$0.00} & 0.46$\pm$0.00 & 0.44$\pm$0.00 & 0.11$\pm$0.00 & 0.34$\pm$0.00 & 0.56$\pm$0.00 & 0.43$\pm$0.03 & \red{0.72$\pm$0.00} \\
			{\it ERO(p=0.05, style=uniform,$\eta$=0.4)} & 0.55$\pm$0.00 & 0.08$\pm$0.00 & 0.01$\pm$0.00 & 0.54$\pm$0.04 & \red{0.63$\pm$0.00} & 0.40$\pm$0.00 & 0.35$\pm$0.00 & 0.02$\pm$0.00 & 0.51$\pm$0.00 & 0.54$\pm$0.00 & 0.52$\pm$0.00 & \blue{0.62$\pm$0.02} \\
			{\it ERO(p=0.05, style=gamma,$\eta$=0.4)} & 0.51$\pm$0.00 & 0.08$\pm$0.00 & -0.00$\pm$0.00 & 0.52$\pm$0.00 & \blue{0.65$\pm$0.00} & 0.43$\pm$0.00 & 0.43$\pm$0.00 & 0.09$\pm$0.01 & 0.23$\pm$0.00 & 0.44$\pm$0.00 & 0.38$\pm$0.08 & \red{0.66$\pm$0.01} \\
			{\it ERO(p=0.05, style=uniform,$\eta$=0.5)} & 0.47$\pm$0.00 & 0.08$\pm$0.00 & 0.01$\pm$0.00 & 0.47$\pm$0.05 & \red{0.57$\pm$0.00} & 0.37$\pm$0.00 & 0.32$\pm$0.00 & 0.04$\pm$0.00 & 0.25$\pm$0.00 & 0.36$\pm$0.00 & 0.42$\pm$0.02 & \blue{0.56$\pm$0.00} \\
			{\it ERO(p=0.05, style=gamma,$\eta$=0.5)} & 0.44$\pm$0.00 & 0.07$\pm$0.00 & 0.01$\pm$0.00 & 0.45$\pm$0.01 & \blue{0.58$\pm$0.00} & 0.37$\pm$0.00 & 0.37$\pm$0.00 & 0.04$\pm$0.00 & 0.22$\pm$0.00 & 0.23$\pm$0.00 & 0.22$\pm$0.01 & \red{0.59$\pm$0.02} \\
			{\it ERO(p=0.05, style=uniform,$\eta$=0.6)} & 0.39$\pm$0.00 & 0.03$\pm$0.00 & -0.00$\pm$0.00 & 0.39$\pm$0.06 & \red{0.48$\pm$0.00} & 0.31$\pm$0.00 & 0.27$\pm$0.00 & -0.03$\pm$0.00 & 0.15$\pm$0.00 & 0.15$\pm$0.00 & 0.33$\pm$0.00 & \blue{0.46$\pm$0.00} \\
			{\it ERO(p=0.05, style=gamma,$\eta$=0.6)} & 0.36$\pm$0.00 & 0.03$\pm$0.00 & 0.01$\pm$0.00 & 0.37$\pm$0.01 & \blue{0.49$\pm$0.00} & 0.31$\pm$0.00 & 0.33$\pm$0.00 & 0.01$\pm$0.00 & 0.14$\pm$0.00 & 0.09$\pm$0.00 & 0.24$\pm$0.03 & \red{0.50$\pm$0.01} \\
			{\it ERO(p=0.05, style=uniform,$\eta$=0.7)} & 0.31$\pm$0.00 & 0.08$\pm$0.00 & -0.01$\pm$0.00 & 0.31$\pm$0.06 & \blue{0.38$\pm$0.00} & 0.26$\pm$0.00 & 0.23$\pm$0.00 & 0.04$\pm$0.00 & 0.07$\pm$0.00 & 0.06$\pm$0.00 & 0.23$\pm$0.01 & \red{0.39$\pm$0.03} \\
			{\it ERO(p=0.05, style=gamma,$\eta$=0.7)} & 0.29$\pm$0.00 & -0.00$\pm$0.00 & 0.02$\pm$0.00 & 0.29$\pm$0.01 & \red{0.39$\pm$0.00} & 0.25$\pm$0.00 & 0.28$\pm$0.00 & -0.00$\pm$0.00 & 0.02$\pm$0.00 & 0.03$\pm$0.00 & 0.05$\pm$0.02 & \red{0.39$\pm$0.02} \\
			{\it ERO(p=0.05, style=uniform,$\eta$=0.8)} & 0.21$\pm$0.00 & 0.02$\pm$0.00 & -0.02$\pm$0.00 & 0.22$\pm$0.07 & \blue{0.25$\pm$0.00} & 0.20$\pm$0.00 & 0.15$\pm$0.00 & -0.03$\pm$0.00 & 0.00$\pm$0.00 & 0.01$\pm$0.00 & 0.11$\pm$0.00 & \red{0.27$\pm$0.04} \\
			{\it ERO(p=0.05, style=gamma,$\eta$=0.8)} & 0.18$\pm$0.00 & 0.07$\pm$0.00 & 0.00$\pm$0.00 & 0.18$\pm$0.01 & \blue{0.24$\pm$0.00} & 0.16$\pm$0.00 & 0.19$\pm$0.00 & -0.02$\pm$0.02 & 0.03$\pm$0.00 & -0.03$\pm$0.00 & 0.05$\pm$0.02 & \red{0.27$\pm$0.03} \\
			{\it ERO(p=1, style=uniform,$\eta$=0)} & \red{1.00$\pm$0.00} & 0.08$\pm$0.00 & \red{1.00$\pm$0.00} & 0.85$\pm$0.05 & \red{1.00$\pm$0.00} & \red{1.00$\pm$0.00} & \red{1.00$\pm$0.00} & -0.00$\pm$0.06 & \red{1.00$\pm$0.00} & \red{1.00$\pm$0.00} & \red{1.00$\pm$0.00} & \red{1.00$\pm$0.00} \\
			{\it ERO(p=1, style=gamma,$\eta$=0)} & \red{1.00$\pm$0.00} & 0.08$\pm$0.00 & \red{1.00$\pm$0.00} & 0.95$\pm$0.00 & \red{1.00$\pm$0.00} & \red{1.00$\pm$0.00} & \red{1.00$\pm$0.00} & 0.01$\pm$0.09 & \red{1.00$\pm$0.00} & 0.99$\pm$0.00 & \red{1.00$\pm$0.00} & \red{1.00$\pm$0.00} \\
			{\it ERO(p=1, style=uniform,$\eta$=0.1)} & 0.94$\pm$0.00 & 0.05$\pm$0.00 & \red{0.98$\pm$0.00} & 0.85$\pm$0.04 & \red{0.98$\pm$0.00} & 0.86$\pm$0.00 & 0.63$\pm$0.00 & 0.52$\pm$0.00 & 0.97$\pm$0.00 & 0.94$\pm$0.00 & 0.93$\pm$0.01 & \red{0.98$\pm$0.00} \\
			{\it ERO(p=1, style=gamma,$\eta$=0.1)} & 0.88$\pm$0.00 & 0.10$\pm$0.00 & \red{0.98$\pm$0.00} & 0.85$\pm$0.00 & \red{0.98$\pm$0.00} & 0.72$\pm$0.00 & 0.57$\pm$0.00 & 0.38$\pm$0.00 & 0.92$\pm$0.00 & 0.90$\pm$0.00 & 0.89$\pm$0.01 & \red{0.98$\pm$0.00} \\
			{\it ERO(p=1, style=uniform,$\eta$=0.2)} & 0.91$\pm$0.00 & 0.07$\pm$0.00 & \red{0.97$\pm$0.00} & 0.84$\pm$0.03 & 0.96$\pm$0.00 & 0.85$\pm$0.00 & 0.69$\pm$0.00 & 0.33$\pm$0.09 & 0.95$\pm$0.00 & 0.93$\pm$0.00 & 0.91$\pm$0.00 & \red{0.97$\pm$0.00} \\
			{\it ERO(p=1, style=gamma,$\eta$=0.2)} & 0.85$\pm$0.00 & 0.05$\pm$0.00 & \red{0.97$\pm$0.00} & 0.83$\pm$0.00 & 0.96$\pm$0.00 & 0.71$\pm$0.00 & 0.60$\pm$0.00 & 0.28$\pm$0.00 & 0.89$\pm$0.00 & 0.87$\pm$0.00 & 0.86$\pm$0.00 & \red{0.97$\pm$0.00} \\
			{\it ERO(p=1, style=uniform,$\eta$=0.3)} & 0.89$\pm$0.00 & 0.09$\pm$0.00 & \red{0.96$\pm$0.00} & 0.84$\pm$0.03 & 0.94$\pm$0.00 & 0.83$\pm$0.00 & 0.71$\pm$0.00 & 0.29$\pm$0.00 & 0.94$\pm$0.00 & 0.91$\pm$0.00 & 0.90$\pm$0.01 & \red{0.96$\pm$0.00} \\
			{\it ERO(p=1, style=gamma,$\eta$=0.3)} & 0.82$\pm$0.00 & 0.06$\pm$0.00 & \red{0.95$\pm$0.00} & 0.80$\pm$0.00 & \red{0.95$\pm$0.00} & 0.70$\pm$0.00 & 0.64$\pm$0.00 & 0.32$\pm$0.00 & 0.85$\pm$0.00 & 0.84$\pm$0.00 & 0.85$\pm$0.00 & \red{0.95$\pm$0.00} \\
			{\it ERO(p=1, style=uniform,$\eta$=0.4)} & 0.88$\pm$0.00 & 0.06$\pm$0.00 & \red{0.94$\pm$0.00} & 0.83$\pm$0.03 & 0.93$\pm$0.00 & 0.83$\pm$0.00 & 0.73$\pm$0.00 & 0.11$\pm$0.00 & 0.92$\pm$0.00 & 0.90$\pm$0.00 & 0.91$\pm$0.01 & \red{0.94$\pm$0.00} \\
			{\it ERO(p=1, style=gamma,$\eta$=0.4)} & 0.79$\pm$0.00 & 0.04$\pm$0.00 & \red{0.94$\pm$0.00} & 0.78$\pm$0.00 & 0.93$\pm$0.00 & 0.69$\pm$0.00 & 0.65$\pm$0.00 & 0.14$\pm$0.00 & 0.83$\pm$0.00 & 0.82$\pm$0.00 & 0.83$\pm$0.01 & \red{0.94$\pm$0.00} \\
			{\it ERO(p=1, style=uniform,$\eta$=0.5)} & 0.85$\pm$0.00 & 0.07$\pm$0.00 & \red{0.92$\pm$0.00} & 0.81$\pm$0.03 & 0.91$\pm$0.00 & 0.80$\pm$0.00 & 0.73$\pm$0.00 & 0.24$\pm$0.00 & 0.89$\pm$0.00 & 0.87$\pm$0.00 & 0.90$\pm$0.01 & \red{0.92$\pm$0.00} \\
			{\it ERO(p=1, style=gamma,$\eta$=0.5)} & 0.76$\pm$0.00 & 0.03$\pm$0.00 & \red{0.92$\pm$0.00} & 0.74$\pm$0.01 & 0.90$\pm$0.00 & 0.68$\pm$0.00 & 0.66$\pm$0.00 & 0.16$\pm$0.01 & 0.79$\pm$0.00 & 0.78$\pm$0.00 & 0.81$\pm$0.01 & \red{0.92$\pm$0.00} \\
			{\it ERO(p=1, style=uniform,$\eta$=0.6)} & 0.83$\pm$0.00 & 0.08$\pm$0.00 & \red{0.89$\pm$0.00} & 0.78$\pm$0.04 & 0.88$\pm$0.00 & 0.78$\pm$0.00 & 0.73$\pm$0.00 & -0.00$\pm$0.00 & 0.86$\pm$0.00 & 0.85$\pm$0.00 & 0.87$\pm$0.01 & \red{0.89$\pm$0.00} \\
			{\it ERO(p=1, style=gamma,$\eta$=0.6)} & 0.72$\pm$0.00 & 0.09$\pm$0.00 & \red{0.89$\pm$0.00} & 0.67$\pm$0.01 & 0.88$\pm$0.00 & 0.65$\pm$0.00 & 0.64$\pm$0.00 & 0.05$\pm$0.02 & 0.74$\pm$0.00 & 0.73$\pm$0.00 & 0.77$\pm$0.00 & \red{0.89$\pm$0.00} \\
			{\it ERO(p=1, style=uniform,$\eta$=0.7)} & 0.77$\pm$0.00 & 0.08$\pm$0.00 & 0.80$\pm$0.00 & 0.72$\pm$0.05 & \red{0.83$\pm$0.00} & 0.74$\pm$0.00 & 0.70$\pm$0.00 & 0.18$\pm$0.00 & 0.80$\pm$0.00 & 0.79$\pm$0.00 & \red{0.83$\pm$0.01} & \red{0.83$\pm$0.00} \\
			{\it ERO(p=1, style=gamma,$\eta$=0.7)} & 0.67$\pm$0.00 & 0.02$\pm$0.00 & 0.75$\pm$0.00 & 0.61$\pm$0.01 & \red{0.83$\pm$0.00} & 0.62$\pm$0.00 & 0.61$\pm$0.00 & 0.09$\pm$0.00 & 0.67$\pm$0.00 & 0.67$\pm$0.00 & 0.71$\pm$0.00 & \red{0.83$\pm$0.00} \\
			{\it ERO(p=1, style=uniform,$\eta$=0.8)} & 0.68$\pm$0.00 & 0.06$\pm$0.00 & 0.11$\pm$0.00 & 0.62$\pm$0.05 & \red{0.75$\pm$0.00} & 0.65$\pm$0.00 & 0.62$\pm$0.00 & 0.15$\pm$0.00 & 0.67$\pm$0.00 & 0.66$\pm$0.00 & 0.70$\pm$0.01 & \blue{0.74$\pm$0.01} \\
			{\it ERO(p=1, style=gamma,$\eta$=0.8)} & 0.57$\pm$0.00 & 0.06$\pm$0.00 & 0.04$\pm$0.00 & 0.51$\pm$0.01 & \red{0.75$\pm$0.00} & 0.55$\pm$0.00 & 0.55$\pm$0.00 & 0.04$\pm$0.00 & 0.54$\pm$0.00 & 0.54$\pm$0.00 & 0.56$\pm$0.00 & \blue{0.74$\pm$0.00} \\
\midrule
{\it ERO(p=0.05, style=uniform,$\eta$=0)} & 0.06$\pm$0.00 & 0.47$\pm$0.00 & 0.47$\pm$0.00 & 0.09$\pm$0.03 & 0.06$\pm$0.00 & \red{0.03$\pm$0.00} & 0.05$\pm$0.00 & 0.49$\pm$0.03 & 0.10$\pm$0.00 & 0.07$\pm$0.00  & \blue{0.04$\pm$0.00} & 0.05$\pm$0.00 \\
			{\it ERO(p=0.05, style=gamma,$\eta$=0)} & 0.08$\pm$0.00 & 0.48$\pm$0.00 & 0.51$\pm$0.00 & 0.04$\pm$0.00 & 0.05$\pm$0.00 & \red{0.02$\pm$0.00} & \blue{0.03$\pm$0.00} & 0.50$\pm$0.00 & 0.20$\pm$0.00 & 0.09$\pm$0.00  & 0.04$\pm$0.00 & 0.04$\pm$0.01 \\
			{\it ERO(p=0.05, style=uniform,$\eta$=0.1)} & 0.13$\pm$0.00 & 0.46$\pm$0.00 & 0.47$\pm$0.00 & 0.15$\pm$0.01 & \blue{0.11$\pm$0.00} & 0.22$\pm$0.00 & 0.21$\pm$0.00 & 0.50$\pm$0.01 & 0.16$\pm$0.00 & 0.14$\pm$0.00  & 0.13$\pm$0.01 & \red{0.09$\pm$0.00} \\
			{\it ERO(p=0.05, style=gamma,$\eta$=0.1)} & 0.16$\pm$0.00 & 0.49$\pm$0.00 & 0.51$\pm$0.00 & 0.16$\pm$0.00 & \blue{0.11$\pm$0.00} & 0.22$\pm$0.00 & 0.21$\pm$0.00 & 0.47$\pm$0.00 & 0.24$\pm$0.00 & 0.15$\pm$0.00  & 0.19$\pm$0.00 & \red{0.09$\pm$0.00} \\
			{\it ERO(p=0.05, style=uniform,$\eta$=0.2)} & 0.18$\pm$0.00 & 0.48$\pm$0.00 & 0.50$\pm$0.00 & 0.20$\pm$0.01 & \blue{0.15$\pm$0.00} & 0.25$\pm$0.00 & 0.27$\pm$0.00 & 0.49$\pm$0.01 & 0.21$\pm$0.00 & 0.19$\pm$0.00  & 0.19$\pm$0.00 & \red{0.13$\pm$0.00} \\
			{\it ERO(p=0.05, style=gamma,$\eta$=0.2)} & 0.21$\pm$0.00 & 0.49$\pm$0.00 & 0.52$\pm$0.00 & 0.21$\pm$0.00 & \blue{0.15$\pm$0.00} & 0.28$\pm$0.00 & 0.27$\pm$0.00 & 0.51$\pm$0.01 & 0.30$\pm$0.00 & 0.21$\pm$0.00  & 0.24$\pm$0.00 & \red{0.14$\pm$0.00} \\
			{\it ERO(p=0.05, style=uniform,$\eta$=0.3)} & 0.22$\pm$0.00 & 0.46$\pm$0.00 & 0.51$\pm$0.00 & 0.23$\pm$0.01 & \blue{0.20$\pm$0.00} & 0.28$\pm$0.00 & 0.31$\pm$0.00 & 0.47$\pm$0.00 & 0.26$\pm$0.00 & 0.24$\pm$0.00  & 0.23$\pm$0.00 & \red{0.18$\pm$0.01} \\
			{\it ERO(p=0.05, style=gamma,$\eta$=0.3)} & 0.25$\pm$0.00 & 0.48$\pm$0.00 & 0.51$\pm$0.00 & 0.26$\pm$0.00 & \blue{0.20$\pm$0.00} & 0.32$\pm$0.00 & 0.33$\pm$0.00 & 0.46$\pm$0.00 & 0.37$\pm$0.00 & 0.26$\pm$0.00  & 0.31$\pm$0.01 & \red{0.19$\pm$0.00} \\
			{\it ERO(p=0.05, style=uniform,$\eta$=0.4)} & 0.26$\pm$0.00 & 0.48$\pm$0.00 & 0.51$\pm$0.00 & 0.27$\pm$0.01 & \red{0.23$\pm$0.00} & 0.31$\pm$0.00 & 0.34$\pm$0.00 & 0.50$\pm$0.00 & 0.31$\pm$0.00 & 0.28$\pm$0.00  & 0.28$\pm$0.00 & \red{0.23$\pm$0.01} \\
			{\it ERO(p=0.05, style=gamma,$\eta$=0.4)} & 0.29$\pm$0.00 & 0.47$\pm$0.00 & 0.51$\pm$0.00 & 0.28$\pm$0.00 & \blue{0.24$\pm$0.00} & 0.34$\pm$0.00 & 0.34$\pm$0.00 & 0.48$\pm$0.01 & 0.43$\pm$0.00 & 0.32$\pm$0.00  & 0.34$\pm$0.04 & \red{0.23$\pm$0.01} \\
			{\it ERO(p=0.05, style=uniform,$\eta$=0.5)} & 0.29$\pm$0.00 & 0.48$\pm$0.00 & 0.51$\pm$0.00 & 0.30$\pm$0.01 & \red{0.27$\pm$0.00} & 0.34$\pm$0.00 & 0.36$\pm$0.00 & 0.49$\pm$0.00 & 0.40$\pm$0.00 & 0.35$\pm$0.00  & 0.32$\pm$0.00 & \red{0.27$\pm$0.00} \\
			{\it ERO(p=0.05, style=gamma,$\eta$=0.5)} & 0.32$\pm$0.00 & 0.48$\pm$0.00 & 0.49$\pm$0.00 & 0.33$\pm$0.00 & \blue{0.28$\pm$0.00} & 0.37$\pm$0.00 & 0.37$\pm$0.00 & 0.48$\pm$0.00 & 0.43$\pm$0.00 & 0.43$\pm$0.00  & 0.38$\pm$0.01 & \red{0.27$\pm$0.01} \\
			{\it ERO(p=0.05, style=uniform,$\eta$=0.6)} & 0.32$\pm$0.00 & 0.47$\pm$0.00 & 0.51$\pm$0.00 & 0.33$\pm$0.01 & \blue{0.31$\pm$0.00} & 0.37$\pm$0.00 & 0.39$\pm$0.00 & 0.51$\pm$0.00 & 0.44$\pm$0.00 & 0.42$\pm$0.00  & 0.35$\pm$0.00 & \red{0.30$\pm$0.00} \\
			{\it ERO(p=0.05, style=gamma,$\eta$=0.6)} & 0.35$\pm$0.00 & 0.48$\pm$0.00 & 0.50$\pm$0.00 & 0.35$\pm$0.00 & \blue{0.32$\pm$0.00} & 0.40$\pm$0.00 & 0.40$\pm$0.00 & 0.48$\pm$0.00 & 0.47$\pm$0.00 & 0.46$\pm$0.00  & 0.40$\pm$0.01 & \red{0.30$\pm$0.01} \\
			{\it ERO(p=0.05, style=uniform,$\eta$=0.7)} & \blue{0.34$\pm$0.00} & 0.48$\pm$0.00 & 0.49$\pm$0.00 & 0.35$\pm$0.01 & \blue{0.34$\pm$0.00} & 0.39$\pm$0.00 & 0.40$\pm$0.00 & 0.49$\pm$0.00 & 0.48$\pm$0.00 & 0.47$\pm$0.00  & 0.38$\pm$0.01 & \red{0.32$\pm$0.00} \\
			{\it ERO(p=0.05, style=gamma,$\eta$=0.7)} & 0.36$\pm$0.00 & 0.49$\pm$0.00 & 0.50$\pm$0.00 & 0.37$\pm$0.01 & \blue{0.35$\pm$0.00} & 0.42$\pm$0.00 & 0.40$\pm$0.00 & 0.49$\pm$0.00 & 0.49$\pm$0.00 & 0.49$\pm$0.00  & 0.43$\pm$0.01 & \red{0.33$\pm$0.01} \\
			{\it ERO(p=0.05, style=uniform,$\eta$=0.8)} & \blue{0.36$\pm$0.00} & 0.48$\pm$0.00 & 0.49$\pm$0.00 & 0.37$\pm$0.00 & \blue{0.36$\pm$0.00} & 0.40$\pm$0.00 & 0.41$\pm$0.00 & 0.51$\pm$0.00 & 0.50$\pm$0.00 & 0.49$\pm$0.00  & 0.42$\pm$0.00 & \red{0.34$\pm$0.00} \\
			{\it ERO(p=0.05, style=gamma,$\eta$=0.8)} & 0.38$\pm$0.00 & 0.48$\pm$0.00 & 0.50$\pm$0.00 & 0.39$\pm$0.01 & \blue{0.36$\pm$0.00} & 0.42$\pm$0.00 & 0.41$\pm$0.00 & 0.50$\pm$0.01 & 0.49$\pm$0.00 & 0.50$\pm$0.00  & 0.43$\pm$0.02 & \red{0.34$\pm$0.01} \\
			{\it ERO(p=1, style=uniform,$\eta$=0)} & \red{0.00$\pm$0.00} & 0.46$\pm$0.00 & \red{0.00$\pm$0.00} & 0.07$\pm$0.02 & \red{0.00$\pm$0.00} & \red{0.00$\pm$0.00} & \red{0.00$\pm$0.00} & 0.50$\pm$0.03 & \red{0.00$\pm$0.00} & \red{0.00$\pm$0.00}  & \red{0.00$\pm$0.00} & \red{0.00$\pm$0.00} \\
			{\it ERO(p=1, style=gamma,$\eta$=0)} & \red{0.00$\pm$0.00} & 0.46$\pm$0.00 & \red{0.00$\pm$0.00} & 0.02$\pm$0.00 & \red{0.00$\pm$0.00} & \red{0.00$\pm$0.00} & \red{0.00$\pm$0.00} & 0.50$\pm$0.05 & \red{0.00$\pm$0.00} & \red{0.00$\pm$0.00}  & \red{0.00$\pm$0.00} & \red{0.00$\pm$0.00} \\
			{\it ERO(p=1, style=uniform,$\eta$=0.1)} & 0.08$\pm$0.00 & 0.48$\pm$0.00 & \red{0.06$\pm$0.00} & 0.12$\pm$0.02 & \red{0.06$\pm$0.00} & 0.11$\pm$0.00 & 0.21$\pm$0.00 & 0.27$\pm$0.00 & \red{0.06$\pm$0.00} & 0.08$\pm$0.00  & 0.08$\pm$0.00 & \red{0.06$\pm$0.00} \\
			{\it ERO(p=1, style=gamma,$\eta$=0.1)} & 0.10$\pm$0.00 & 0.46$\pm$0.00 & \red{0.06$\pm$0.00} & 0.11$\pm$0.00 & \red{0.06$\pm$0.00} & 0.18$\pm$0.00 & 0.24$\pm$0.00 & 0.33$\pm$0.00 & 0.09$\pm$0.00 & 0.09$\pm$0.00  & 0.10$\pm$0.00 & \red{0.06$\pm$0.00} \\
			{\it ERO(p=1, style=uniform,$\eta$=0.2)} & 0.13$\pm$0.00 & 0.47$\pm$0.00 & \red{0.11$\pm$0.00} & 0.16$\pm$0.01 & 0.12$\pm$0.00 & 0.16$\pm$0.00 & 0.22$\pm$0.00 & 0.37$\pm$0.04 & 0.12$\pm$0.00 & 0.13$\pm$0.00  & 0.14$\pm$0.00 & \red{0.11$\pm$0.00} \\
			{\it ERO(p=1, style=gamma,$\eta$=0.2)} & 0.16$\pm$0.00 & 0.48$\pm$0.00 & \red{0.11$\pm$0.00} & 0.17$\pm$0.00 & 0.12$\pm$0.00 & 0.21$\pm$0.00 & 0.26$\pm$0.00 & 0.39$\pm$0.00 & 0.14$\pm$0.00 & 0.15$\pm$0.00  & 0.15$\pm$0.00 & \red{0.11$\pm$0.00} \\
			{\it ERO(p=1, style=uniform,$\eta$=0.3)} & 0.19$\pm$0.00 & 0.47$\pm$0.00 & \red{0.16$\pm$0.00} & 0.21$\pm$0.01 & 0.17$\pm$0.00 & 0.21$\pm$0.00 & 0.25$\pm$0.00 & 0.40$\pm$0.00 & 0.17$\pm$0.00 & 0.18$\pm$0.00  & 0.18$\pm$0.00 & \red{0.16$\pm$0.00} \\
			{\it ERO(p=1, style=gamma,$\eta$=0.3)} & 0.21$\pm$0.00 & 0.48$\pm$0.00 & \red{0.16$\pm$0.00} & 0.21$\pm$0.00 & 0.17$\pm$0.00 & 0.25$\pm$0.00 & 0.27$\pm$0.00 & 0.39$\pm$0.00 & 0.20$\pm$0.00 & 0.20$\pm$0.00  & 0.20$\pm$0.00 & \red{0.16$\pm$0.00} \\
			{\it ERO(p=1, style=uniform,$\eta$=0.4)} & 0.23$\pm$0.00 & 0.48$\pm$0.00 & \red{0.21$\pm$0.00} & 0.25$\pm$0.01 & 0.22$\pm$0.00 & 0.25$\pm$0.00 & 0.28$\pm$0.00 & 0.47$\pm$0.00 & 0.22$\pm$0.00 & 0.23$\pm$0.00  & 0.23$\pm$0.00 & \red{0.21$\pm$0.00} \\
			{\it ERO(p=1, style=gamma,$\eta$=0.4)} & 0.26$\pm$0.00 & 0.49$\pm$0.00 & \red{0.22$\pm$0.00} & 0.26$\pm$0.00 & \red{0.22$\pm$0.00} & 0.29$\pm$0.00 & 0.30$\pm$0.00 & 0.46$\pm$0.00 & 0.25$\pm$0.00 & 0.25$\pm$0.00  & 0.25$\pm$0.00 & \red{0.22$\pm$0.00} \\
			{\it ERO(p=1, style=uniform,$\eta$=0.5)} & 0.28$\pm$0.00 & 0.49$\pm$0.00 & \red{0.27$\pm$0.00} & 0.30$\pm$0.01 & \red{0.27$\pm$0.00} & 0.30$\pm$0.00 & 0.31$\pm$0.00 & 0.44$\pm$0.00 & \red{0.27$\pm$0.00} & 0.28$\pm$0.00  & \red{0.27$\pm$0.00} & \red{0.27$\pm$0.00} \\
			{\it ERO(p=1, style=gamma,$\eta$=0.5)} & 0.30$\pm$0.00 & 0.49$\pm$0.00 & \red{0.27$\pm$0.00} & 0.31$\pm$0.00 & \red{0.27$\pm$0.00} & 0.32$\pm$0.00 & 0.33$\pm$0.00 & 0.46$\pm$0.00 & 0.30$\pm$0.00 & 0.30$\pm$0.00  & 0.29$\pm$0.00 & \red{0.27$\pm$0.00} \\
			{\it ERO(p=1, style=uniform,$\eta$=0.6)} & 0.33$\pm$0.00 & 0.48$\pm$0.00 & \red{0.32$\pm$0.00} & 0.34$\pm$0.01 & \red{0.32$\pm$0.00} & 0.34$\pm$0.00 & 0.35$\pm$0.00 & 0.50$\pm$0.00 & \red{0.32$\pm$0.00} & 0.33$\pm$0.00  & \red{0.32$\pm$0.00} & \red{0.32$\pm$0.00} \\
			{\it ERO(p=1, style=gamma,$\eta$=0.6)} & 0.35$\pm$0.00 & 0.48$\pm$0.00 & \red{0.32$\pm$0.00} & 0.36$\pm$0.00 & \red{0.32$\pm$0.00} & 0.36$\pm$0.00 & 0.36$\pm$0.00 & 0.49$\pm$0.00 & 0.34$\pm$0.00 & 0.34$\pm$0.00  & 0.34$\pm$0.00 & \red{0.32$\pm$0.00} \\
			{\it ERO(p=1, style=uniform,$\eta$=0.7)} & 0.38$\pm$0.00 & 0.49$\pm$0.00 & \red{0.37$\pm$0.00} & 0.39$\pm$0.01 & \red{0.37$\pm$0.00} & 0.38$\pm$0.00 & 0.39$\pm$0.00 & 0.47$\pm$0.00 & \red{0.37$\pm$0.00} & 0.38$\pm$0.00  & \red{0.37$\pm$0.00} & \red{0.37$\pm$0.00} \\
			{\it ERO(p=1, style=gamma,$\eta$=0.7)} & 0.39$\pm$0.00 & 0.50$\pm$0.00 & 0.38$\pm$0.00 & 0.40$\pm$0.00 & \red{0.37$\pm$0.00} & 0.40$\pm$0.00 & 0.40$\pm$0.00 & 0.49$\pm$0.00 & 0.39$\pm$0.00 & 0.39$\pm$0.00  & 0.38$\pm$0.00 & \red{0.37$\pm$0.00} \\
			{\it ERO(p=1, style=uniform,$\eta$=0.8)} & \blue{0.42$\pm$0.00} & 0.49$\pm$0.00 & 0.49$\pm$0.00 & 0.43$\pm$0.00 & \blue{0.42$\pm$0.00} & 0.43$\pm$0.00 & 0.43$\pm$0.00 & 0.48$\pm$0.00 & \blue{0.42$\pm$0.00} & \blue{0.42$\pm$0.00}  & \blue{0.42$\pm$0.00} & \red{0.41$\pm$0.00} \\
			{\it ERO(p=1, style=gamma,$\eta$=0.8)} & 0.43$\pm$0.00 & 0.49$\pm$0.00 & 0.49$\pm$0.00 & 0.44$\pm$0.00 & \blue{0.42$\pm$0.00} & 0.43$\pm$0.00 & 0.43$\pm$0.00 & 0.50$\pm$0.00 & 0.43$\pm$0.00 & 0.43$\pm$0.00  & 0.43$\pm$0.00 & \red{0.41$\pm$0.00} \\
\bottomrule
\end{tabular}}
\end{table*}
\subsection{Results on Different Variants}
\label{appendix_subsec:variant_wise_results_full}
Tables \ref{tab:upset_simple_variant_wise_full}, \ref{tab:upset_naive_variant_wise_full} and \ref{tab:upset_ratio_variant_wise_full} compare different variants in the proposed GNNRank framework with themselves, together with the best and the worst baseline method, across all real-world data sets. Again, $\mathcal{L}_\text{upset, ratio}$ results are only  shown to illustrate the performance
in terms of minimizing the loss function $\mathcal{L}_\text{upset, ratio}$, but are not used to compare
methods. We conclude that ``proximal baseline" with IB as aggregation GNN usually performs the best among the variants. Compared with baselines, our different variants can attain comparable and often superior performance, and are never strongly outperformed. 
\begin{table*}[!ht]
\centering
\vspace{-10pt}
\caption{Performance on $\mathcal{L}_\text{upset, simple}$ for each variant in the proposed GNNRank framework, compared with the worst and the best baseline method, averaged over 10 runs, and plus/minus one standard deviation. ``avg" for time series data sets first average over all seasons, then consider mean and standard deviation over the 10 averaged values. The best is marked in \red{bold red} while the second best is highlighted in \blue{underline blue}.}
\label{tab:upset_simple_variant_wise_full}
\resizebox{1\linewidth}{!}{
\begin{tabular}{lrr|rrrrr|rrrrr}
\toprule
Type&\multicolumn{2}{c}{Baselines}&\multicolumn{5}{|c|}{DIMPA }&\multicolumn{5}{c}{IB}\\
Data&worst baseline&best baseline  & dist&innerproduct&proximal dist&proximal innerproduct&proximal baseline& dist&innerproduct&proximal dist&proximal innerproduct&proximal baseline \\
\midrule
{\it HeadToHead} & 2.01$\pm$0.00 & 1.00$\pm$0.00 & 1.07$\pm$0.01 & 1.07$\pm$0.01 & 1.12$\pm$0.01 & 1.09$\pm$0.01 & \blue{0.97$\pm$0.00} & 0.99$\pm$0.00 & 0.99$\pm$0.00 & 0.99$\pm$0.00 & 0.99$\pm$0.00 & \red{0.96$\pm$0.00} \\
			{\it Finance} & 1.98$\pm$0.00 & 1.61$\pm$0.00 & 1.63$\pm$0.00 & 1.63$\pm$0.00 & 1.63$\pm$0.00 & 1.63$\pm$0.00 & 1.61$\pm$0.00 & \red{1.00$\pm$0.00} & \red{1.00$\pm$0.00} & \red{1.00$\pm$0.00} & \red{1.00$\pm$0.00} & 1.61$\pm$0.00 \\
			{\it Animal} & 2.02$\pm$0.32 & \blue{0.33$\pm$0.00} & 0.41$\pm$0.09 & 0.48$\pm$0.19 & 0.41$\pm$0.05 & 0.42$\pm$0.05 & 0.42$\pm$0.05 & 0.61$\pm$0.07 & 0.57$\pm$0.08 & 0.90$\pm$0.33 & 0.83$\pm$0.13 & \red{0.25$\pm$0.00} \\
			{\it Faculty: Business} & 2.01$\pm$0.03 & 0.41$\pm$0.00 & 0.41$\pm$0.02 & \blue{0.38$\pm$0.01} & 0.42$\pm$0.02 & 0.47$\pm$0.03 & 0.39$\pm$0.01 & 0.92$\pm$0.15 & 0.94$\pm$0.07 & 0.89$\pm$0.02 & 0.90$\pm$0.04 & \red{0.36$\pm$0.00} \\
			{\it Faculty: CS} & 1.99$\pm$0.27 & \blue{0.33$\pm$0.00} & \blue{0.33$\pm$0.03} & 0.34$\pm$0.01 & 0.38$\pm$0.01 & 0.36$\pm$0.01 & \blue{0.33$\pm$0.01} & 0.92$\pm$0.01 & 0.95$\pm$0.06 & 0.92$\pm$0.03 & 0.92$\pm$0.03 & \red{0.32$\pm$0.00} \\
			{\it Faculty: History} & 2.16$\pm$0.80 & \blue{0.30$\pm$0.01} & \blue{0.30$\pm$0.01} & \red{0.28$\pm$0.01} & \blue{0.30$\pm$0.01} & 0.31$\pm$0.01 & \blue{0.30$\pm$0.01} & 0.92$\pm$0.01 & 0.94$\pm$0.08 & 0.93$\pm$0.06 & 0.90$\pm$0.06 & \blue{0.30$\pm$0.01} \\
\bottomrule
\end{tabular}}
\end{table*}
\begin{table*}[!ht]
\centering
\vspace{-10pt}
\caption{Performance on $\mathcal{L}_\text{upset, naive}$ for each variant in the proposed GNNRank framework, compared with the worst and the best baseline method, averaged over 10 runs, and plus/minus one standard deviation. The best is marked in \red{bold red} while the second best is highlighted in \blue{underline blue}.}
\label{tab:upset_naive_variant_wise_full}
\resizebox{0.9\linewidth}{!}{\begin{tabular}{lrr|rrrrr|rrrrr}
\toprule
Type&\multicolumn{2}{c}{Baselines}&\multicolumn{5}{|c|}{DIMPA }&\multicolumn{5}{c}{IB}\\
Data&worst baseline&best baseline  & dist&innerproduct&proximal dist&proximal innerproduct&proximal baseline& dist&innerproduct&proximal dist&proximal innerproduct&proximal baseline \\
\midrule
{\it HeadToHead} & 0.50$\pm$0.00 & 0.25$\pm$0.00 & 0.27$\pm$0.00 & 0.27$\pm$0.00 & 0.28$\pm$0.00 & 0.27$\pm$0.00 & \red{0.24$\pm$0.00} & 0.55$\pm$0.07 & 0.49$\pm$0.05 & 0.64$\pm$0.08 & 0.51$\pm$0.03 & \red{0.24$\pm$0.00} \\
			{\it Finance} & 0.50$\pm$0.00 & \red{0.40$\pm$0.00} & 0.41$\pm$0.00 & 0.41$\pm$0.00 & 0.41$\pm$0.00 & 0.41$\pm$0.00 & \red{0.40$\pm$0.00} & 0.44$\pm$0.00 & 0.46$\pm$0.02 & 0.45$\pm$0.01 & 0.45$\pm$0.01 & \red{0.40$\pm$0.00} \\
			{\it Animal} & 0.58$\pm$0.11 & \blue{0.08$\pm$0.00} & 0.10$\pm$0.02 & 0.12$\pm$0.05 & 0.10$\pm$0.01 & 0.10$\pm$0.01 & 0.11$\pm$0.01 & 0.52$\pm$0.12 & 0.36$\pm$0.09 & 0.48$\pm$0.09 & 0.39$\pm$0.10 & \red{0.06$\pm$0.00} \\
			{\it Faculty: Business} & 0.50$\pm$0.01 & \blue{0.10$\pm$0.00} & \blue{0.10$\pm$0.00} & \blue{0.10$\pm$0.00} & 0.11$\pm$0.00 & 0.12$\pm$0.01 & \blue{0.10$\pm$0.00} & 0.24$\pm$0.07 & 0.24$\pm$0.02 & 0.25$\pm$0.07 & 0.23$\pm$0.01 & \red{0.09$\pm$0.00} \\
			{\it Faculty: CS} & 0.50$\pm$0.07 & \red{0.08$\pm$0.00} & \red{0.08$\pm$0.01} & 0.09$\pm$0.00 & 0.09$\pm$0.00 & 0.09$\pm$0.00 & \red{0.08$\pm$0.00} & 0.48$\pm$0.02 & 0.44$\pm$0.07 & 0.44$\pm$0.02 & 0.46$\pm$0.03 & \red{0.08$\pm$0.00} \\
			{\it Faculty: History} & 0.54$\pm$0.20 & 0.08$\pm$0.00 & \red{0.07$\pm$0.00} & \red{0.07$\pm$0.00} & \red{0.07$\pm$0.00} & 0.08$\pm$0.00 & 0.08$\pm$0.00 & 0.50$\pm$0.08 & 0.43$\pm$0.16 & 0.51$\pm$0.24 & 0.46$\pm$0.21 & 0.08$\pm$0.00 \\
\bottomrule
\end{tabular}}
\end{table*}
\begin{table*}[!ht]
\centering
\vspace{-10pt}
\caption{Performance on $\mathcal{L}_\text{upset, ratio}$ for each variant in the proposed GNNRank framework, compared with the worst and the best baseline method, averaged over 10 runs, and plus/minus one standard deviation. The best is marked in \red{bold red} while the second best is highlighted in \blue{underline blue}.}
\label{tab:upset_ratio_variant_wise_full}
\resizebox{\linewidth}{!}{\begin{tabular}{lrr|rrrrr|rrrrr}
\toprule
Type&\multicolumn{2}{c}{Baselines}&\multicolumn{5}{|c|}{DIMPA }&\multicolumn{5}{c}{IB}\\
Data&worst baseline&best baseline  & dist&innerproduct&proximal dist&proximal innerproduct&proximal baseline& dist&innerproduct&proximal dist&proximal innerproduct&proximal baseline \\
\midrule
{\it HeadToHead} & 1.19$\pm$0.09 & 0.71$\pm$0.01 & \red{0.68$\pm$0.00} & \blue{0.69$\pm$0.01} & 0.95$\pm$0.00 & 0.95$\pm$0.00 & 0.94$\pm$0.00 & 0.96$\pm$0.00 & 0.96$\pm$0.00 & 0.96$\pm$0.00 & 0.96$\pm$0.00 & 0.94$\pm$0.00 \\
			{\it Finance} & 0.64$\pm$0.00 & \red{0.26$\pm$0.00} & \red{0.26$\pm$0.00} & \red{0.26$\pm$0.00} & 0.27$\pm$0.00 & 0.27$\pm$0.00 & 0.27$\pm$0.00 & 0.27$\pm$0.00 & 0.27$\pm$0.00 & 0.27$\pm$0.00 & 0.27$\pm$0.00 & 0.27$\pm$0.00 \\
			{\it Animal} & 1.34$\pm$0.00 & \red{0.22$\pm$0.00} & \blue{0.24$\pm$0.03} & 0.35$\pm$0.24 & 0.68$\pm$0.02 & 0.67$\pm$0.01 & 0.66$\pm$0.00 & 0.60$\pm$0.14 & 0.45$\pm$0.02 & 0.79$\pm$0.01 & 0.72$\pm$0.02 & 0.66$\pm$0.00 \\
			{\it Faculty: Business} & 1.40$\pm$0.13 & 0.34$\pm$0.00 & \blue{0.33$\pm$0.01} & \red{0.31$\pm$0.00} & 0.89$\pm$0.00 & 0.90$\pm$0.00 & 0.89$\pm$0.01 & 0.80$\pm$0.20 & 0.80$\pm$0.07 & 0.93$\pm$0.02 & 0.93$\pm$0.01 & 0.89$\pm$0.00 \\
			{\it Faculty: CS} & 1.33$\pm$0.16 & 0.28$\pm$0.00 & \red{0.26$\pm$0.02} & \red{0.26$\pm$0.01} & 0.92$\pm$0.01 & 0.91$\pm$0.00 & 0.86$\pm$0.00 & 0.90$\pm$0.00 & 0.90$\pm$0.06 & 0.95$\pm$0.00 & 0.95$\pm$0.00 & 0.90$\pm$0.00 \\
			{\it Faculty: History} & 1.35$\pm$0.11 & \blue{0.23$\pm$0.00} & \blue{0.23$\pm$0.01} & \red{0.21$\pm$0.00} & 0.85$\pm$0.00 & 0.88$\pm$0.00 & 0.84$\pm$0.02 & 0.89$\pm$0.04 & 0.83$\pm$0.11 & 0.96$\pm$0.00 & 0.93$\pm$0.01 & 0.87$\pm$0.00 \\
\bottomrule
\end{tabular}}
\end{table*}
\subsection{Ablation Study Full Tables}
\label{appendix_subsec:ablation_full}
Tables \ref{tab:ablation_simple_full}, \ref{tab:ablation_naive_full} and \ref{tab:ablation_ratio_full} extend results of the ablation study to all seasons, and additionally report $\mathcal{L}_\text{upset, ratio}$ and $\mathcal{L}_\text{upset, naive}$ results. Again, $\mathcal{L}_\text{upset, ratio}$ results are only  shown to illustrate the performance
in terms of minimizing the loss function $\mathcal{L}_\text{upset, ratio}$, but are not used to compare methods.

\begin{table*}[tb!]
\centering
\vspace{-10pt}
	\caption{$\mathcal{L}_\text{upset, simple}$ comparison for different variants on selected real-world data, averaged over 10 runs, and plus/minus one standard deviation. ``avg" for time series data first average over all seasons, then consider mean and standard deviation over the 10 averaged values. The best for each group of variants (GNNRank-N or GNNRank-P) is marked in \red{bold red} while the second best is highlighted in \blue{underline blue}.}
\label{tab:ablation_simple_full}
\resizebox{\linewidth}{!}{\begin{tabular}{l|rr p{5em}| rrrrrrrr}
\toprule
Methods&\multicolumn{3}{c|}{GNNRank-N}&\multicolumn{7}{c}{GNNRank-P}\\
Data/Variant &loss sum & $\mathcal{L}_\text{upset,margin} $ & $\mathcal{L}_\text{upset, ratio} $ & loss sum & $\mathcal{L}_\text{upset,margin} $  & $\mathcal{L}_\text{upset, ratio} $   & no pretrain & $\{\alpha_\gamma\}_{\gamma=1}^\Gamma$not trainable &$\Gamma=3$ & $\Gamma=7$ \\
			\midrule
{\it Animal} & \blue{0.43$\pm$0.06} & 0.59$\pm$0.08 & \red{0.41$\pm$0.09} & \red{0.25$\pm$0.00} & \red{0.25$\pm$0.00} & \red{0.25$\pm$0.01} & \red{0.25$\pm$0.00} & \red{0.25$\pm$0.00} & \red{0.25$\pm$0.00} & \red{0.25$\pm$0.00} \\
			{\it Faculty: Business} & \blue{0.40$\pm$0.02} & 0.49$\pm$0.16 & \red{0.38$\pm$0.01} & \red{0.36$\pm$0.00} & \red{0.36$\pm$0.00} & \red{0.36$\pm$0.00} & \red{0.36$\pm$0.00} & \red{0.36$\pm$0.00} & \red{0.36$\pm$0.00} & \red{0.36$\pm$0.00} \\
			{\it Faculty: CS} & \blue{0.35$\pm$0.01} & 0.36$\pm$0.01 & \red{0.33$\pm$0.03} & \red{0.32$\pm$0.00} & \red{0.32$\pm$0.00} & \red{0.32$\pm$0.00} & 0.33$\pm$0.00 & \red{0.32$\pm$0.00} & \red{0.32$\pm$0.00} & \red{0.32$\pm$0.00} \\
			{\it Faculty: History} & \red{0.28$\pm$0.01} & 0.31$\pm$0.01 & \red{0.28$\pm$0.01} & \red{0.30$\pm$0.01} & \red{0.30$\pm$0.01} & \red{0.30$\pm$0.02} & \red{0.30$\pm$0.01} & \red{0.30$\pm$0.01} & \red{0.30$\pm$0.01} & \red{0.30$\pm$0.01} \\
			{\it Football(2009)} & \blue{0.76$\pm$0.03} & 0.79$\pm$0.04 & \red{0.75$\pm$0.06} & \red{0.61$\pm$0.00} & \red{0.61$\pm$0.00} & \red{0.61$\pm$0.00} & \red{0.61$\pm$0.00} & \red{0.61$\pm$0.00} & \red{0.61$\pm$0.00} & \red{0.61$\pm$0.00} \\
			{\it Football(2010)} & \red{0.90$\pm$0.02} & 0.93$\pm$0.05 & \blue{0.92$\pm$0.03} & \red{0.95$\pm$0.05} & 0.96$\pm$0.06 & 0.96$\pm$0.05 & \red{0.95$\pm$0.06} & 0.96$\pm$0.06 & \red{0.95$\pm$0.05} & \red{0.95$\pm$0.04} \\
			{\it Football(2011)} & 0.86$\pm$0.10 & \red{0.83$\pm$0.03} & \blue{0.84$\pm$0.08} & \red{0.80$\pm$0.00} & \red{0.80$\pm$0.00} & \red{0.80$\pm$0.01} & \red{0.80$\pm$0.00} & \red{0.80$\pm$0.01} & \red{0.80$\pm$0.00} & \red{0.80$\pm$0.00} \\
			{\it Football(2012)} & 0.81$\pm$0.09 & \red{0.75$\pm$0.03} & \blue{0.78$\pm$0.03} & \blue{0.80$\pm$0.08} & \blue{0.80$\pm$0.08} & 0.84$\pm$0.00 & 0.84$\pm$0.00 & \blue{0.80$\pm$0.08} & \red{0.78$\pm$0.03} & \blue{0.80$\pm$0.06} \\
			{\it Football(2013)} & \blue{0.65$\pm$0.03} & 0.75$\pm$0.05 & \red{0.64$\pm$0.04} & \red{0.56$\pm$0.00} & \red{0.56$\pm$0.00} & \red{0.56$\pm$0.00} & \red{0.56$\pm$0.00} & \red{0.56$\pm$0.00} & \red{0.56$\pm$0.00} & \red{0.56$\pm$0.00} \\
			{\it Football(2014)} & \red{0.96$\pm$0.12} & 1.00$\pm$0.11 & \blue{0.98$\pm$0.17} & 0.98$\pm$0.06 & \blue{0.96$\pm$0.07} & 0.98$\pm$0.09 & 0.99$\pm$0.10 & 0.98$\pm$0.06 & \red{0.95$\pm$0.05} & 0.97$\pm$0.08 \\
			{\it Football finer(2009)} & \red{0.76$\pm$0.03} & 0.93$\pm$0.38 & \red{0.76$\pm$0.18} & 0.65$\pm$0.02 & 0.65$\pm$0.02 & 0.65$\pm$0.03 & 0.66$\pm$0.00 & 0.65$\pm$0.03 & \red{0.63$\pm$0.03} & \blue{0.64$\pm$0.02} \\
			{\it Football finer(2010)} & \red{1.00$\pm$0.01} & \red{1.00$\pm$0.01} & \red{1.00$\pm$0.01} & \red{0.99$\pm$0.02} & 1.00$\pm$0.01 & 1.00$\pm$0.00 & \red{0.99$\pm$0.02} & 1.00$\pm$0.01 & 1.00$\pm$0.00 & 1.00$\pm$0.00 \\
			{\it Football finer(2011)} & \blue{0.99$\pm$0.03} & \red{0.92$\pm$0.16} & \blue{0.99$\pm$0.04} & 0.85$\pm$0.01 & \blue{0.84$\pm$0.02} & 0.85$\pm$0.00 & 0.85$\pm$0.00 & 0.85$\pm$0.01 & \red{0.83$\pm$0.03} & \blue{0.84$\pm$0.03} \\
			{\it Football finer(2012)} & \red{0.93$\pm$0.07} & 1.00$\pm$0.02 & \blue{0.96$\pm$0.03} & \red{0.86$\pm$0.02} & \red{0.86$\pm$0.03} & \red{0.86$\pm$0.02} & \red{0.86$\pm$0.00} & \red{0.86$\pm$0.03} & \red{0.86$\pm$0.01} & \red{0.86$\pm$0.00} \\
			{\it Football finer(2013)} & \blue{0.74$\pm$0.04} & 0.98$\pm$0.04 & \red{0.73$\pm$0.19} & 0.59$\pm$0.02 & 0.69$\pm$0.03 & \red{0.57$\pm$0.01} & 0.58$\pm$0.00 & 0.58$\pm$0.01 & \red{0.57$\pm$0.02} & \red{0.57$\pm$0.02} \\
			{\it Football finer(2014)} & \red{1.00$\pm$0.00} & 1.01$\pm$0.04 & \red{1.00$\pm$0.00} & \red{1.00$\pm$0.00} & 1.01$\pm$0.03 & \red{1.00$\pm$0.03} & 1.08$\pm$0.12 & \red{1.00$\pm$0.03} & \red{1.00$\pm$0.00} & \red{1.00$\pm$0.00} \\
			\bottomrule
		\end{tabular}}
\end{table*}
\begin{table*}[tb!]
\centering
\vspace{-10pt}
	\caption{$\mathcal{L}_\text{upset, naive}$ comparison for different variants on selected real-world data, averaged over 10 runs, and plus/minus one standard deviation. ``avg" for time series data first average over all seasons, then consider mean and standard deviation over the 10 averaged values. The best is marked in \red{bold red} while the second best is highlighted in \blue{underline blue}.}
\label{tab:ablation_naive_full}
\resizebox{1\linewidth}{!}{\begin{tabular}{l|rr p{5em}| rrrrrrrr}
\toprule
Methods&\multicolumn{3}{c|}{GNNRank-N}&\multicolumn{7}{c}{GNNRank-P}\\
Data/Variant &loss sum & $\mathcal{L}_\text{upset,margin} $ & $\mathcal{L}_\text{upset, ratio} $ & loss sum & $\mathcal{L}_\text{upset,margin} $  & $\mathcal{L}_\text{upset, ratio} $   & no pretrain & $\{\alpha_\gamma\}_{\gamma=1}^\Gamma$not trainable &$\Gamma=3$ & $\Gamma=7$ \\
			\midrule
{\it Animal} & \blue{0.11$\pm$0.02} & 0.15$\pm$0.02 & \red{0.10$\pm$0.02} & \red{0.06$\pm$0.00} & \red{0.06$\pm$0.00} & \red{0.06$\pm$0.00} & \red{0.06$\pm$0.00} & \red{0.06$\pm$0.00} & \red{0.06$\pm$0.00} & \red{0.06$\pm$0.00} \\
			{\it Faculty: Business} & \red{0.10$\pm$0.00} & 0.12$\pm$0.04 & \red{0.10$\pm$0.00} & \red{0.09$\pm$0.00} & \red{0.09$\pm$0.00} & \red{0.09$\pm$0.00} & \red{0.09$\pm$0.00} & \red{0.09$\pm$0.00} & \red{0.09$\pm$0.00} & \red{0.09$\pm$0.00} \\
			{\it Faculty: CS} & \blue{0.09$\pm$0.00} & \blue{0.09$\pm$0.00} & \red{0.08$\pm$0.01} & \red{0.08$\pm$0.00} & \red{0.08$\pm$0.00} & \red{0.08$\pm$0.00} & \red{0.08$\pm$0.00} & \red{0.08$\pm$0.00} & \red{0.08$\pm$0.00} & \red{0.08$\pm$0.00} \\
			{\it Faculty: History} & \red{0.07$\pm$0.00} & 0.08$\pm$0.00 & \red{0.07$\pm$0.00} & \red{0.07$\pm$0.00} & 0.08$\pm$0.00 & 0.08$\pm$0.00 & \red{0.07$\pm$0.00} & \red{0.07$\pm$0.00} & 0.08$\pm$0.00 & 0.08$\pm$0.00 \\
			{\it Football(2009)} & \red{0.19$\pm$0.01} & 0.26$\pm$0.06 & \red{0.19$\pm$0.01} & \red{0.15$\pm$0.00} & \red{0.15$\pm$0.00} & \red{0.15$\pm$0.00} & \red{0.15$\pm$0.00} & \red{0.15$\pm$0.00} & \red{0.15$\pm$0.00} & \red{0.15$\pm$0.00} \\
			{\it Football(2010)} & \red{0.30$\pm$0.01} & 0.36$\pm$0.06 & \red{0.30$\pm$0.01} & \red{0.29$\pm$0.00} & \red{0.29$\pm$0.00} & \red{0.29$\pm$0.00} & \red{0.29$\pm$0.00} & \red{0.29$\pm$0.00} & \red{0.29$\pm$0.01} & \red{0.29$\pm$0.00} \\
			{\it Football(2011)} & \red{0.22$\pm$0.01} & 0.26$\pm$0.07 & \red{0.22$\pm$0.01} & \red{0.20$\pm$0.00} & \red{0.20$\pm$0.00} & \red{0.20$\pm$0.00} & \red{0.20$\pm$0.00} & \red{0.20$\pm$0.00} & \red{0.20$\pm$0.00} & \red{0.20$\pm$0.00} \\
			{\it Football(2012)} & \red{0.20$\pm$0.02} & 0.29$\pm$0.06 & \blue{0.21$\pm$0.04} & \blue{0.21$\pm$0.00} & \blue{0.21$\pm$0.00} & \blue{0.21$\pm$0.00} & \blue{0.21$\pm$0.00} & \blue{0.21$\pm$0.00} & \red{0.20$\pm$0.00} & \blue{0.21$\pm$0.00} \\
			{\it Football(2013)} & \red{0.16$\pm$0.01} & 0.20$\pm$0.07 & \red{0.16$\pm$0.01} & \red{0.14$\pm$0.00} & \red{0.14$\pm$0.00} & \red{0.14$\pm$0.00} & \red{0.14$\pm$0.00} & \red{0.14$\pm$0.00} & \red{0.14$\pm$0.00} & \red{0.14$\pm$0.00} \\
			{\it Football(2014)} & \red{0.24$\pm$0.01} & 0.30$\pm$0.02 & \blue{0.25$\pm$0.04} & \red{0.26$\pm$0.01} & 0.27$\pm$0.00 & 0.27$\pm$0.02 & 0.27$\pm$0.00 & \red{0.26$\pm$0.01} & 0.27$\pm$0.02 & 0.27$\pm$0.00 \\
			{\it Football finer(2009)} & \red{0.19$\pm$0.01} & 0.23$\pm$0.09 & \red{0.19$\pm$0.05} & \red{0.16$\pm$0.01} & \red{0.16$\pm$0.01} & \red{0.16$\pm$0.01} & \red{0.16$\pm$0.00} & \red{0.16$\pm$0.01} & \red{0.16$\pm$0.01} & \red{0.16$\pm$0.00} \\
			{\it Football finer(2010)} & \red{0.31$\pm$0.01} & 0.35$\pm$0.05 & \red{0.31$\pm$0.02} & \red{0.29$\pm$0.00} & \red{0.29$\pm$0.00} & \red{0.29$\pm$0.00} & \red{0.29$\pm$0.00} & \red{0.29$\pm$0.00} & \red{0.29$\pm$0.00} & \red{0.29$\pm$0.00} \\
			{\it Football finer(2011)} & \blue{0.25$\pm$0.01} & \red{0.23$\pm$0.04} & \blue{0.25$\pm$0.01} & \red{0.21$\pm$0.00} & \red{0.21$\pm$0.01} & \red{0.21$\pm$0.00} & \red{0.21$\pm$0.00} & \red{0.21$\pm$0.01} & \red{0.21$\pm$0.01} & \red{0.21$\pm$0.01} \\
			{\it Football finer(2012)} & \red{0.24$\pm$0.02} & 0.25$\pm$0.04 & \red{0.24$\pm$0.01} & \red{0.21$\pm$0.00} & \red{0.21$\pm$0.01} & 0.22$\pm$0.01 & 0.22$\pm$0.00 & \red{0.21$\pm$0.01} & 0.22$\pm$0.00 & 0.22$\pm$0.00 \\
			{\it Football finer(2013)} & \blue{0.19$\pm$0.01} & 0.26$\pm$0.12 & \red{0.18$\pm$0.05} & 0.15$\pm$0.00 & 0.17$\pm$0.01 & \red{0.14$\pm$0.00} & 0.15$\pm$0.00 & \red{0.14$\pm$0.00} & \red{0.14$\pm$0.00} & \red{0.14$\pm$0.00} \\
			{\it Football finer(2014)} & \red{0.27$\pm$0.01} & 0.35$\pm$0.09 & \blue{0.28$\pm$0.04} & \red{0.27$\pm$0.02} & \red{0.27$\pm$0.00} & \red{0.27$\pm$0.02} & \red{0.27$\pm$0.00} & \red{0.27$\pm$0.02} & \red{0.27$\pm$0.00} & \red{0.27$\pm$0.01} \\
			\bottomrule
		\end{tabular}}
\end{table*}
\begin{table*}[tb!]
\centering
\vspace{-10pt}
	\caption{$\mathcal{L}_\text{upset, ratio}$ comparison for different variants on selected real-world data, averaged over 10 runs, and plus/minus one standard deviation. ``avg" for time series data first average over all seasons, then consider mean and standard deviation over the 10 averaged values. The best for each group of variants (GNNRank-N or GNNRank-P) is marked in \red{bold red} while the second best is highlighted in \blue{underline blue}.}
\label{tab:ablation_ratio_full}
\resizebox{1\linewidth}{!}{\begin{tabular}{l|rr p{5em}| rrrrrrrr}
\toprule
Methods&\multicolumn{3}{c|}{GNNRank-N}&\multicolumn{7}{c}{GNNRank-P}\\
Data/Variant &loss sum & $\mathcal{L}_\text{upset,margin} $ & $\mathcal{L}_\text{upset, ratio} $ & loss sum & $\mathcal{L}_\text{upset,margin} $  & $\mathcal{L}_\text{upset, ratio} $   & no pretrain & $\{\alpha_\gamma\}_{\gamma=1}^\Gamma$not trainable &$\Gamma=3$ & $\Gamma=7$ \\
			\midrule
{\it Animal} & \red{0.24$\pm$0.01} & 0.41$\pm$0.09 & \red{0.24$\pm$0.03} & \red{0.66$\pm$0.00} & \red{0.66$\pm$0.00} & \red{0.66$\pm$0.00} & \red{0.66$\pm$0.00} & \red{0.66$\pm$0.00} & \red{0.66$\pm$0.00} & \red{0.66$\pm$0.00} \\
			{\it Faculty: Business} & \blue{0.32$\pm$0.02} & 0.71$\pm$0.14 & \red{0.31$\pm$0.00} & \red{0.89$\pm$0.00} & \red{0.89$\pm$0.01} & \red{0.89$\pm$0.00} & \red{0.89$\pm$0.01} & \red{0.89$\pm$0.00} & \red{0.89$\pm$0.01} & \red{0.89$\pm$0.00} \\
			{\it Faculty: CS} & \blue{0.27$\pm$0.01} & 0.69$\pm$0.05 & \red{0.26$\pm$0.02} & 0.90$\pm$0.00 & \red{0.86$\pm$0.00} & 0.90$\pm$0.00 & 0.90$\pm$0.00 & \red{0.86$\pm$0.00} & \red{0.86$\pm$0.00} & \red{0.86$\pm$0.00} \\
			{\it Faculty: History} & \red{0.21$\pm$0.00} & 0.60$\pm$0.10 & \red{0.21$\pm$0.00} & 0.86$\pm$0.00 & \red{0.84$\pm$0.02} & 0.85$\pm$0.00 & 0.87$\pm$0.00 & \red{0.84$\pm$0.02} & \red{0.84$\pm$0.02} & \red{0.84$\pm$0.02} \\
			{\it Football(2009)} & \red{0.46$\pm$0.01} & 0.72$\pm$0.14 & \blue{0.48$\pm$0.01} & \red{0.69$\pm$0.00} & \red{0.69$\pm$0.00} & \red{0.69$\pm$0.00} & \red{0.69$\pm$0.00} & \red{0.69$\pm$0.00} & \red{0.69$\pm$0.00} & \red{0.69$\pm$0.00} \\
			{\it Football(2010)} & \red{0.65$\pm$0.00} & 0.75$\pm$0.02 & \blue{0.68$\pm$0.07} & \red{0.73$\pm$0.01} & 0.74$\pm$0.00 & 0.74$\pm$0.00 & \red{0.73$\pm$0.00} & \red{0.73$\pm$0.01} & \red{0.73$\pm$0.00} & \red{0.73$\pm$0.01} \\
			{\it Football(2011)} & \red{0.53$\pm$0.01} & 0.70$\pm$0.04 & \blue{0.54$\pm$0.02} & \red{0.69$\pm$0.00} & \red{0.69$\pm$0.00} & \red{0.69$\pm$0.00} & \red{0.69$\pm$0.00} & \red{0.69$\pm$0.00} & \red{0.69$\pm$0.00} & \red{0.69$\pm$0.00} \\
			{\it Football(2012)} & \red{0.51$\pm$0.01} & 0.65$\pm$0.04 & \blue{0.53$\pm$0.08} & \red{0.71$\pm$0.00} & \red{0.71$\pm$0.00} & \red{0.71$\pm$0.00} & \red{0.71$\pm$0.00} & \red{0.71$\pm$0.00} & \red{0.71$\pm$0.00} & \red{0.71$\pm$0.00} \\
			{\it Football(2013)} & \red{0.46$\pm$0.01} & 0.56$\pm$0.16 & \red{0.46$\pm$0.01} & \red{0.71$\pm$0.00} & \red{0.71$\pm$0.00} & \red{0.71$\pm$0.00} & \red{0.71$\pm$0.00} & \red{0.71$\pm$0.00} & \red{0.71$\pm$0.00} & \red{0.71$\pm$0.00} \\
			{\it Football(2014)} & \red{0.69$\pm$0.01} & 0.92$\pm$0.07 & \red{0.69$\pm$0.07} & \red{0.85$\pm$0.00} & \red{0.85$\pm$0.00} & \red{0.85$\pm$0.00} & \red{0.85$\pm$0.00} & \red{0.85$\pm$0.00} & \red{0.85$\pm$0.00} & \red{0.85$\pm$0.00} \\
			{\it Football finer(2009)} & \red{0.17$\pm$0.00} & 0.28$\pm$0.07 & \red{0.17$\pm$0.03} & \blue{0.21$\pm$0.00} & \blue{0.21$\pm$0.00} & \blue{0.21$\pm$0.01} & \blue{0.21$\pm$0.00} & \blue{0.21$\pm$0.00} & \red{0.20$\pm$0.00} & \blue{0.21$\pm$0.01} \\
			{\it Football finer(2010)} & \blue{0.19$\pm$0.03} & 0.25$\pm$0.00 & \red{0.17$\pm$0.00} & \red{0.18$\pm$0.00} & 0.19$\pm$0.00 & \red{0.18$\pm$0.00} & \red{0.18$\pm$0.00} & \red{0.18$\pm$0.00} & \red{0.18$\pm$0.00} & \red{0.18$\pm$0.00} \\
			{\it Football finer(2011)} & \red{0.17$\pm$0.00} & 0.22$\pm$0.09 & \blue{0.18$\pm$0.01} & \red{0.19$\pm$0.00} & 0.20$\pm$0.02 & \red{0.19$\pm$0.00} & \red{0.19$\pm$0.00} & \red{0.19$\pm$0.00} & \red{0.19$\pm$0.00} & \red{0.19$\pm$0.00} \\
			{\it Football finer(2012)} & \red{0.15$\pm$0.00} & 0.23$\pm$0.02 & \blue{0.16$\pm$0.03} & \red{0.17$\pm$0.00} & \red{0.17$\pm$0.00} & \red{0.17$\pm$0.00} & \red{0.17$\pm$0.00} & \red{0.17$\pm$0.00} & \red{0.17$\pm$0.00} & \red{0.17$\pm$0.00} \\
			{\it Football finer(2013)} & \blue{0.21$\pm$0.01} & 0.29$\pm$0.03 & \red{0.19$\pm$0.04} & 0.25$\pm$0.01 & 0.25$\pm$0.00 & \red{0.24$\pm$0.01} & 0.25$\pm$0.00 & \red{0.24$\pm$0.01} & \red{0.24$\pm$0.00} & 0.25$\pm$0.00 \\
			{\it Football finer(2014)} & \red{0.34$\pm$0.00} & 0.46$\pm$0.04 & \blue{0.35$\pm$0.05} & \red{0.38$\pm$0.00} & \red{0.38$\pm$0.00} & \red{0.38$\pm$0.00} & \red{0.38$\pm$0.00} & \red{0.38$\pm$0.00} & \red{0.38$\pm$0.00} & \red{0.38$\pm$0.00} \\
			\bottomrule
		\end{tabular}}
\end{table*}
\subsection{Inductive Learning Full Tables}
\label{appendix_subsec:inductive}
Tables \ref{tab:upset_simple_matches_inductive} and \ref{tab:upset_ratio_matches_inductive} contain results on the performance of the ``IB proximal baseline" variant, trained with the ``emb baseline" on the \textit{Basketball finer} data set. The second to last column contains results of directly applying the model trained for season 1985 without further training, while the last column is the result for the model specifically trained for that season. For each year we have 10 runs, and we average over the total of 30 years for each run. On average, 
directly applying the original trained model gives $\mathcal{L}_\text{upset, simple}=0.75\pm 0.02$, $\mathcal{L}_\text{upset, naive}=0.19\pm 0.01$ and $\mathcal{L}_\text{upset, ratio}=0.01\pm 0.00$, while training specifically for the that season gives $\mathcal{L}_\text{upset, simple}=0.74\pm 0.00$, $\mathcal{L}_\text{upset, naive}=0.19\pm 0.00$ and $\mathcal{L}_\text{upset, ratio}=0.01\pm 0.00$.
Thus, in this example, applying the general model produces almost the same superior performance.
\begin{table*}[!ht]
\centering
\caption{Result table on $\mathcal{L}_\text{upset, simple}$ for each year in the time series matches, applying the trained model for 1985 on all seasons without further training, averaged over 10 runs, and plus/minus one standard deviation. The best is marked in \red{bold red} while the second best is highlighted in \blue{underline blue}. As MVR could not generate results after a week, we omit the results here.}
\label{tab:upset_simple_matches_inductive}
\resizebox{1\linewidth}{!}{\begin{tabular}{lrrrrrrrrrrrrrrrrr}
\toprule
Data  & SpringRank & SyncRank & SerialRank & BTL & DavidScore & Eig.Cent. & PageRank & RankCent. & SVD\_RS & SVD\_NRS  & Directly Apply&Train Specifically \\
\midrule
Basketball finer(1985) & 0.76$\pm$0.00 & 1.63$\pm$0.00 & 1.96$\pm$0.10 & 1.46$\pm$0.05 & 0.83$\pm$0.00 & 1.18$\pm$0.00 & 1.16$\pm$0.00 & 1.97$\pm$0.00 & 1.00$\pm$0.00 & 0.87$\pm$0.00 & \red{0.71$\pm$0.00} & \red{0.71$\pm$0.00} \\
			Basketball finer(1986) & 0.77$\pm$0.00 & 1.81$\pm$0.00 & 1.99$\pm$0.00 & 1.42$\pm$0.06 & 0.84$\pm$0.00 & 1.16$\pm$0.00 & 1.15$\pm$0.00 & 1.99$\pm$0.00 & 1.09$\pm$0.00 & 0.86$\pm$0.00 & \red{0.69$\pm$0.00} & \red{0.69$\pm$0.00} \\
			Basketball finer(1987) & 0.82$\pm$0.00 & 1.79$\pm$0.00 & 1.87$\pm$0.00 & 1.41$\pm$0.06 & 0.89$\pm$0.00 & 1.17$\pm$0.00 & 1.21$\pm$0.00 & 1.95$\pm$0.00 & 0.99$\pm$0.00 & 0.91$\pm$0.00 & \red{0.77$\pm$0.00} & \red{0.77$\pm$0.00} \\
			Basketball finer(1988) & 0.78$\pm$0.00 & 1.79$\pm$0.00 & 1.90$\pm$0.00 & 1.43$\pm$0.10 & 0.84$\pm$0.00 & 1.23$\pm$0.00 & 1.19$\pm$0.00 & 1.97$\pm$0.00 & 0.97$\pm$0.00 & 0.83$\pm$0.00 & \red{0.70$\pm$0.00} & \red{0.70$\pm$0.00} \\
			Basketball finer(1989) & \blue{0.77$\pm$0.00} & 1.67$\pm$0.00 & 1.86$\pm$0.00 & 1.43$\pm$0.05 & 0.83$\pm$0.00 & 1.13$\pm$0.00 & 1.14$\pm$0.00 & 1.94$\pm$0.00 & 0.99$\pm$0.00 & 0.90$\pm$0.00 & 0.88$\pm$0.52 & \red{0.70$\pm$0.00} \\
			Basketball finer(1990) & 0.79$\pm$0.00 & 1.67$\pm$0.00 & 1.93$\pm$0.00 & 1.45$\pm$0.05 & 0.82$\pm$0.00 & 1.28$\pm$0.00 & 1.17$\pm$0.00 & 1.98$\pm$0.00 & 0.91$\pm$0.00 & 0.84$\pm$0.00 & \red{0.71$\pm$0.00} & \red{0.71$\pm$0.00} \\
			Basketball finer(1991) & 0.81$\pm$0.00 & 1.83$\pm$0.00 & 2.03$\pm$0.00 & 1.36$\pm$0.06 & 0.83$\pm$0.00 & 1.38$\pm$0.00 & 1.31$\pm$0.00 & 1.97$\pm$0.00 & 0.99$\pm$0.00 & 0.89$\pm$0.00 & \red{0.71$\pm$0.00} & \red{0.71$\pm$0.00} \\
			Basketball finer(1992) & 0.73$\pm$0.00 & 1.72$\pm$0.00 & 1.88$\pm$0.00 & 1.33$\pm$0.06 & 0.77$\pm$0.00 & 1.26$\pm$0.00 & 1.21$\pm$0.00 & 1.87$\pm$0.00 & 0.95$\pm$0.00 & 0.84$\pm$0.00 & \red{0.67$\pm$0.00} & \red{0.67$\pm$0.00} \\
			Basketball finer(1993) & 0.75$\pm$0.00 & 1.66$\pm$0.00 & 2.03$\pm$0.00 & 1.35$\pm$0.05 & 0.78$\pm$0.00 & 1.18$\pm$0.00 & 1.10$\pm$0.00 & 1.97$\pm$0.00 & 0.98$\pm$0.00 & 0.86$\pm$0.00 & \red{0.69$\pm$0.00} & \red{0.69$\pm$0.00} \\
			Basketball finer(1994) & 0.74$\pm$0.00 & 1.69$\pm$0.00 & 2.01$\pm$0.00 & 1.35$\pm$0.08 & 0.78$\pm$0.00 & 1.23$\pm$0.00 & 1.10$\pm$0.00 & 1.94$\pm$0.00 & 0.90$\pm$0.00 & 0.83$\pm$0.00 & \red{0.67$\pm$0.00} & \red{0.67$\pm$0.00} \\
			Basketball finer(1995) & 0.79$\pm$0.00 & 1.78$\pm$0.00 & 1.89$\pm$0.00 & 1.35$\pm$0.06 & 0.83$\pm$0.00 & 1.19$\pm$0.00 & 1.13$\pm$0.00 & 1.92$\pm$0.01 & 0.95$\pm$0.00 & 0.87$\pm$0.00 & \red{0.73$\pm$0.00} & \red{0.73$\pm$0.00} \\
			Basketball finer(1996) & 0.81$\pm$0.00 & 1.67$\pm$0.00 & 1.95$\pm$0.00 & 1.44$\pm$0.06 & 0.88$\pm$0.00 & 1.22$\pm$0.00 & 1.20$\pm$0.00 & 1.94$\pm$0.00 & 1.08$\pm$0.00 & 0.95$\pm$0.00 & \red{0.77$\pm$0.00} & \red{0.77$\pm$0.00} \\
			Basketball finer(1997) & 0.83$\pm$0.00 & 1.77$\pm$0.00 & 1.94$\pm$0.00 & 1.40$\pm$0.04 & 0.86$\pm$0.00 & 1.19$\pm$0.00 & 1.16$\pm$0.00 & 2.05$\pm$0.00 & 0.96$\pm$0.00 & 0.92$\pm$0.00 & \red{0.77$\pm$0.00} & \red{0.77$\pm$0.00} \\
			Basketball finer(1998) & 0.78$\pm$0.00 & 1.70$\pm$0.00 & 1.92$\pm$0.00 & 1.36$\pm$0.07 & 0.83$\pm$0.00 & 1.14$\pm$0.00 & 1.13$\pm$0.00 & 1.91$\pm$0.00 & 0.97$\pm$0.00 & 0.90$\pm$0.00 & \red{0.74$\pm$0.00} & \red{0.74$\pm$0.00} \\
			Basketball finer(1999) & 0.81$\pm$0.00 & 1.64$\pm$0.00 & 2.02$\pm$0.00 & 1.38$\pm$0.07 & 0.86$\pm$0.00 & 1.17$\pm$0.00 & 1.11$\pm$0.00 & 1.99$\pm$0.00 & 1.17$\pm$0.00 & 0.94$\pm$0.00 & \red{0.73$\pm$0.00} & \red{0.73$\pm$0.00} \\
			Basketball finer(2000) & 0.84$\pm$0.00 & 1.75$\pm$0.00 & 1.97$\pm$0.00 & 1.39$\pm$0.05 & 0.90$\pm$0.00 & 1.26$\pm$0.00 & 1.18$\pm$0.00 & 1.92$\pm$0.00 & 1.12$\pm$0.00 & 0.95$\pm$0.00 & \red{0.78$\pm$0.00} & \red{0.78$\pm$0.00} \\
			Basketball finer(2001) & 0.81$\pm$0.00 & 1.69$\pm$0.00 & 2.06$\pm$0.00 & 1.41$\pm$0.06 & 0.86$\pm$0.00 & 1.25$\pm$0.00 & 1.18$\pm$0.00 & 2.03$\pm$0.00 & 1.08$\pm$0.00 & 0.97$\pm$0.00 & \red{0.73$\pm$0.00} & \red{0.73$\pm$0.00} \\
			Basketball finer(2002) & 0.87$\pm$0.00 & 1.75$\pm$0.00 & 1.86$\pm$0.00 & 1.43$\pm$0.08 & 0.89$\pm$0.00 & 1.20$\pm$0.00 & 1.13$\pm$0.00 & 2.03$\pm$0.00 & 1.07$\pm$0.00 & 0.92$\pm$0.00 & \red{0.78$\pm$0.00} & \red{0.78$\pm$0.00} \\
			Basketball finer(2003) & 0.87$\pm$0.00 & 1.78$\pm$0.00 & 1.98$\pm$0.07 & 1.46$\pm$0.09 & 0.91$\pm$0.00 & 1.18$\pm$0.00 & 1.14$\pm$0.00 & 2.00$\pm$0.00 & 1.02$\pm$0.00 & 0.95$\pm$0.00 & \red{0.78$\pm$0.00} & \red{0.78$\pm$0.00} \\
			Basketball finer(2004) & 0.77$\pm$0.00 & 1.71$\pm$0.06 & 1.87$\pm$0.00 & 1.41$\pm$0.06 & 0.80$\pm$0.00 & 1.17$\pm$0.00 & 1.13$\pm$0.00 & 1.98$\pm$0.02 & 0.95$\pm$0.00 & 0.88$\pm$0.00 & \red{0.72$\pm$0.00} & \red{0.72$\pm$0.00} \\
			Basketball finer(2005) & \blue{0.84$\pm$0.00} & 1.82$\pm$0.00 & 1.93$\pm$0.00 & 1.38$\pm$0.06 & 0.88$\pm$0.00 & 1.14$\pm$0.00 & 1.09$\pm$0.00 & 2.00$\pm$0.00 & 1.08$\pm$0.00 & 0.95$\pm$0.00 & 0.93$\pm$0.52 & \red{0.75$\pm$0.00} \\
			Basketball finer(2006) & 0.86$\pm$0.00 & 1.76$\pm$0.00 & 1.97$\pm$0.00 & 1.40$\pm$0.07 & 0.85$\pm$0.00 & 1.21$\pm$0.00 & 1.11$\pm$0.00 & 1.96$\pm$0.00 & 1.06$\pm$0.00 & 0.94$\pm$0.00 & \red{0.76$\pm$0.00} & \red{0.76$\pm$0.00} \\
			Basketball finer(2007) & 0.86$\pm$0.00 & 1.85$\pm$0.00 & 1.97$\pm$0.14 & 1.39$\pm$0.05 & 0.93$\pm$0.00 & 1.15$\pm$0.00 & 1.09$\pm$0.00 & 1.92$\pm$0.00 & 0.95$\pm$0.00 & 0.93$\pm$0.00 & \red{0.80$\pm$0.00} & \red{0.80$\pm$0.00} \\
			Basketball finer(2008) & 0.85$\pm$0.00 & 1.72$\pm$0.00 & 1.98$\pm$0.00 & 1.36$\pm$0.08 & 0.88$\pm$0.00 & 1.20$\pm$0.00 & 1.13$\pm$0.00 & 1.96$\pm$0.00 & 0.99$\pm$0.00 & 0.91$\pm$0.00 & \red{0.78$\pm$0.00} & \red{0.78$\pm$0.00} \\
			Basketball finer(2009) & 0.84$\pm$0.00 & 1.71$\pm$0.00 & 2.00$\pm$0.04 & 1.37$\pm$0.04 & 0.90$\pm$0.00 & 1.17$\pm$0.00 & 1.12$\pm$0.00 & 2.06$\pm$0.00 & 0.99$\pm$0.00 & 0.91$\pm$0.00 & \red{0.75$\pm$0.00} & \red{0.75$\pm$0.00} \\
			Basketball finer(2010) & 0.82$\pm$0.00 & 1.68$\pm$0.00 & 1.98$\pm$0.00 & 1.34$\pm$0.06 & 0.84$\pm$0.00 & 1.10$\pm$0.00 & 1.13$\pm$0.00 & 1.97$\pm$0.00 & 0.94$\pm$0.00 & 0.92$\pm$0.00 & \red{0.75$\pm$0.00} & \red{0.75$\pm$0.00} \\
			Basketball finer(2011) & 0.85$\pm$0.00 & 1.65$\pm$0.00 & 1.96$\pm$0.00 & 1.36$\pm$0.06 & 0.87$\pm$0.00 & 1.12$\pm$0.00 & 1.14$\pm$0.00 & 2.01$\pm$0.00 & 0.96$\pm$0.00 & 0.89$\pm$0.00 & \red{0.77$\pm$0.00} & \red{0.77$\pm$0.00} \\
			Basketball finer(2012) & 0.80$\pm$0.00 & 1.68$\pm$0.00 & 1.97$\pm$0.00 & 1.37$\pm$0.07 & 0.83$\pm$0.00 & 1.17$\pm$0.00 & 1.10$\pm$0.00 & 1.91$\pm$0.00 & 0.93$\pm$0.00 & 0.89$\pm$0.00 & \red{0.75$\pm$0.00} & \red{0.75$\pm$0.00} \\
			Basketball finer(2013) & 0.83$\pm$0.00 & 1.73$\pm$0.00 & 2.06$\pm$0.00 & 1.37$\pm$0.07 & 0.88$\pm$0.00 & 1.15$\pm$0.00 & 1.13$\pm$0.00 & 1.97$\pm$0.00 & 1.01$\pm$0.00 & 0.92$\pm$0.00 & \red{0.79$\pm$0.00} & \red{0.79$\pm$0.00} \\
			Basketball finer(2014) & 0.84$\pm$0.00 & 1.77$\pm$0.19 & 2.07$\pm$0.00 & 1.42$\pm$0.09 & 0.86$\pm$0.00 & 1.18$\pm$0.00 & 1.13$\pm$0.00 & 1.97$\pm$0.00 & 1.02$\pm$0.00 & 0.88$\pm$0.00 & \red{0.79$\pm$0.00} & \red{0.79$\pm$0.00} \\
\bottomrule
\end{tabular}}
\end{table*}
\begin{table*}[!ht]
\centering
\caption{Result table on $\mathcal{L}_\text{upset, naive}$ for each year in the time series matches, applying the trained model for 1985 directly without further training, averaged over 10 runs, and plus/minus one standard deviation. The best is marked in \red{bold red} while the second best is highlighted in \blue{underline blue}. As MVR could not generate scores, we omit the results here.}
\label{tab:upset_naive_matches_inductive}
\resizebox{1\linewidth}{!}{\begin{tabular}{lrrrrrrrrrrrrrrrrr}
\toprule
Data  & SpringRank & SyncRank & SerialRank & BTL & DavidScore & Eig.Cent. & PageRank & RankCent. & SVD\_RS & SVD\_NRS& Directly Apply&Train Specifically \\
\midrule
Basketball finer(1985) & 0.19$\pm$0.00 & 0.41$\pm$0.00 & 0.49$\pm$0.02 & 0.36$\pm$0.01 & 0.21$\pm$0.00 & 0.29$\pm$0.00 & 0.29$\pm$0.00 & 0.49$\pm$0.00 & 0.25$\pm$0.00 & 0.22$\pm$0.00 & \red{0.18$\pm$0.00} & \red{0.18$\pm$0.00} \\
			Basketball finer(1986) & 0.19$\pm$0.00 & 0.45$\pm$0.00 & 0.50$\pm$0.00 & 0.36$\pm$0.02 & 0.21$\pm$0.00 & 0.29$\pm$0.00 & 0.29$\pm$0.00 & 0.50$\pm$0.00 & 0.27$\pm$0.00 & 0.21$\pm$0.00 & \red{0.17$\pm$0.00} & \red{0.17$\pm$0.00} \\
			Basketball finer(1987) & 0.20$\pm$0.00 & 0.45$\pm$0.00 & 0.47$\pm$0.00 & 0.35$\pm$0.01 & 0.22$\pm$0.00 & 0.29$\pm$0.00 & 0.30$\pm$0.00 & 0.49$\pm$0.00 & 0.25$\pm$0.00 & 0.23$\pm$0.00 & \red{0.19$\pm$0.00} & \red{0.19$\pm$0.00} \\
			Basketball finer(1988) & 0.19$\pm$0.00 & 0.45$\pm$0.00 & 0.48$\pm$0.00 & 0.36$\pm$0.03 & 0.21$\pm$0.00 & 0.31$\pm$0.00 & 0.30$\pm$0.00 & 0.49$\pm$0.00 & 0.24$\pm$0.00 & 0.21$\pm$0.00 & \red{0.18$\pm$0.00} & \red{0.18$\pm$0.00} \\
			Basketball finer(1989) & \blue{0.19$\pm$0.00} & 0.42$\pm$0.00 & 0.46$\pm$0.00 & 0.36$\pm$0.01 & 0.21$\pm$0.00 & 0.28$\pm$0.00 & 0.29$\pm$0.00 & 0.49$\pm$0.00 & 0.25$\pm$0.00 & 0.23$\pm$0.00 & 0.22$\pm$0.13 & \red{0.18$\pm$0.00} \\
			Basketball finer(1990) & 0.20$\pm$0.00 & 0.42$\pm$0.00 & 0.48$\pm$0.00 & 0.36$\pm$0.01 & 0.21$\pm$0.00 & 0.32$\pm$0.00 & 0.29$\pm$0.00 & 0.50$\pm$0.00 & 0.23$\pm$0.00 & 0.21$\pm$0.00 & \red{0.18$\pm$0.00} & \red{0.18$\pm$0.00} \\
			Basketball finer(1991) & 0.20$\pm$0.00 & 0.46$\pm$0.00 & 0.51$\pm$0.00 & 0.34$\pm$0.02 & 0.21$\pm$0.00 & 0.35$\pm$0.00 & 0.33$\pm$0.00 & 0.49$\pm$0.00 & 0.25$\pm$0.00 & 0.22$\pm$0.00 & \red{0.18$\pm$0.00} & \red{0.18$\pm$0.00} \\
			Basketball finer(1992) & 0.18$\pm$0.00 & 0.43$\pm$0.00 & 0.47$\pm$0.00 & 0.33$\pm$0.01 & 0.19$\pm$0.00 & 0.31$\pm$0.00 & 0.30$\pm$0.00 & 0.47$\pm$0.00 & 0.24$\pm$0.00 & 0.21$\pm$0.00 & \red{0.17$\pm$0.00} & \red{0.17$\pm$0.00} \\
			Basketball finer(1993) & 0.19$\pm$0.00 & 0.42$\pm$0.00 & 0.51$\pm$0.00 & 0.34$\pm$0.01 & 0.20$\pm$0.00 & 0.29$\pm$0.00 & 0.27$\pm$0.00 & 0.49$\pm$0.00 & 0.25$\pm$0.00 & 0.21$\pm$0.00 & \red{0.17$\pm$0.00} & \red{0.17$\pm$0.00} \\
			Basketball finer(1994) & 0.18$\pm$0.00 & 0.42$\pm$0.00 & 0.50$\pm$0.00 & 0.34$\pm$0.02 & 0.19$\pm$0.00 & 0.31$\pm$0.00 & 0.27$\pm$0.00 & 0.49$\pm$0.00 & 0.22$\pm$0.00 & 0.21$\pm$0.00 & \red{0.17$\pm$0.00} & \red{0.17$\pm$0.00} \\
			Basketball finer(1995) & 0.20$\pm$0.00 & 0.44$\pm$0.00 & 0.47$\pm$0.00 & 0.34$\pm$0.02 & 0.21$\pm$0.00 & 0.30$\pm$0.00 & 0.28$\pm$0.00 & 0.48$\pm$0.00 & 0.24$\pm$0.00 & 0.22$\pm$0.00 & \red{0.18$\pm$0.00} & \red{0.18$\pm$0.00} \\
			Basketball finer(1996) & 0.20$\pm$0.00 & 0.42$\pm$0.00 & 0.49$\pm$0.00 & 0.36$\pm$0.02 & 0.22$\pm$0.00 & 0.30$\pm$0.00 & 0.30$\pm$0.00 & 0.49$\pm$0.00 & 0.27$\pm$0.00 & 0.24$\pm$0.00 & \red{0.19$\pm$0.00} & \red{0.19$\pm$0.00} \\
			Basketball finer(1997) & 0.21$\pm$0.00 & 0.44$\pm$0.00 & 0.49$\pm$0.00 & 0.35$\pm$0.01 & 0.21$\pm$0.00 & 0.30$\pm$0.00 & 0.29$\pm$0.00 & 0.51$\pm$0.00 & 0.24$\pm$0.00 & 0.23$\pm$0.00 & \red{0.19$\pm$0.00} & \red{0.19$\pm$0.00} \\
			Basketball finer(1998) & 0.20$\pm$0.00 & 0.42$\pm$0.00 & 0.48$\pm$0.00 & 0.34$\pm$0.02 & 0.21$\pm$0.00 & 0.29$\pm$0.00 & 0.28$\pm$0.00 & 0.48$\pm$0.00 & 0.24$\pm$0.00 & 0.22$\pm$0.00 & \red{0.18$\pm$0.00} & \red{0.18$\pm$0.00} \\
			Basketball finer(1999) & 0.20$\pm$0.00 & 0.41$\pm$0.00 & 0.50$\pm$0.00 & 0.34$\pm$0.02 & 0.22$\pm$0.00 & 0.29$\pm$0.00 & 0.28$\pm$0.00 & 0.50$\pm$0.00 & 0.29$\pm$0.00 & 0.24$\pm$0.00 & \red{0.18$\pm$0.00} & \red{0.18$\pm$0.00} \\
			Basketball finer(2000) & 0.21$\pm$0.00 & 0.44$\pm$0.00 & 0.49$\pm$0.00 & 0.35$\pm$0.01 & 0.23$\pm$0.00 & 0.32$\pm$0.00 & 0.30$\pm$0.00 & 0.48$\pm$0.00 & 0.28$\pm$0.00 & 0.24$\pm$0.00 & \red{0.19$\pm$0.00} & \red{0.19$\pm$0.00} \\
			Basketball finer(2001) & 0.20$\pm$0.00 & 0.42$\pm$0.00 & 0.51$\pm$0.00 & 0.35$\pm$0.01 & 0.21$\pm$0.00 & 0.31$\pm$0.00 & 0.30$\pm$0.00 & 0.51$\pm$0.00 & 0.27$\pm$0.00 & 0.24$\pm$0.00 & \red{0.18$\pm$0.00} & \red{0.18$\pm$0.00} \\
			Basketball finer(2002) & 0.22$\pm$0.00 & 0.44$\pm$0.00 & 0.47$\pm$0.00 & 0.36$\pm$0.02 & 0.22$\pm$0.00 & 0.30$\pm$0.00 & 0.28$\pm$0.00 & 0.51$\pm$0.00 & 0.27$\pm$0.00 & 0.23$\pm$0.00 & \red{0.19$\pm$0.00} & \red{0.19$\pm$0.00} \\
			Basketball finer(2003) & 0.22$\pm$0.00 & 0.45$\pm$0.00 & 0.50$\pm$0.02 & 0.36$\pm$0.02 & 0.23$\pm$0.00 & 0.29$\pm$0.00 & 0.29$\pm$0.00 & 0.50$\pm$0.00 & 0.26$\pm$0.00 & 0.24$\pm$0.00 & \red{0.19$\pm$0.00} & \red{0.19$\pm$0.00} \\
			Basketball finer(2004) & 0.19$\pm$0.00 & 0.43$\pm$0.01 & 0.47$\pm$0.00 & 0.35$\pm$0.01 & 0.20$\pm$0.00 & 0.29$\pm$0.00 & 0.28$\pm$0.00 & 0.49$\pm$0.00 & 0.24$\pm$0.00 & 0.22$\pm$0.00 & \red{0.18$\pm$0.00} & \red{0.18$\pm$0.00} \\
			Basketball finer(2005) & \blue{0.21$\pm$0.00} & 0.46$\pm$0.00 & 0.48$\pm$0.00 & 0.34$\pm$0.01 & 0.22$\pm$0.00 & 0.29$\pm$0.00 & 0.27$\pm$0.00 & 0.50$\pm$0.00 & 0.27$\pm$0.00 & 0.24$\pm$0.00 & 0.23$\pm$0.13 & \red{0.19$\pm$0.00} \\
			Basketball finer(2006) & 0.21$\pm$0.00 & 0.44$\pm$0.00 & 0.49$\pm$0.00 & 0.35$\pm$0.02 & 0.21$\pm$0.00 & 0.30$\pm$0.00 & 0.28$\pm$0.00 & 0.49$\pm$0.00 & 0.27$\pm$0.00 & 0.23$\pm$0.00 & \red{0.19$\pm$0.00} & \red{0.19$\pm$0.00} \\
			Basketball finer(2007) & 0.21$\pm$0.00 & 0.46$\pm$0.00 & 0.49$\pm$0.03 & 0.35$\pm$0.01 & 0.23$\pm$0.00 & 0.29$\pm$0.00 & 0.27$\pm$0.00 & 0.48$\pm$0.00 & 0.24$\pm$0.00 & 0.23$\pm$0.00 & \red{0.20$\pm$0.00} & \red{0.20$\pm$0.00} \\
			Basketball finer(2008) & 0.21$\pm$0.00 & 0.43$\pm$0.00 & 0.49$\pm$0.00 & 0.34$\pm$0.02 & 0.22$\pm$0.00 & 0.30$\pm$0.00 & 0.28$\pm$0.00 & 0.49$\pm$0.00 & 0.25$\pm$0.00 & 0.23$\pm$0.00 & \red{0.20$\pm$0.00} & \red{0.20$\pm$0.00} \\
			Basketball finer(2009) & 0.21$\pm$0.00 & 0.43$\pm$0.00 & 0.50$\pm$0.01 & 0.34$\pm$0.01 & 0.22$\pm$0.00 & 0.29$\pm$0.00 & 0.28$\pm$0.00 & 0.52$\pm$0.00 & 0.25$\pm$0.00 & 0.23$\pm$0.00 & \red{0.19$\pm$0.00} & \red{0.19$\pm$0.00} \\
			Basketball finer(2010) & 0.20$\pm$0.00 & 0.42$\pm$0.00 & 0.49$\pm$0.00 & 0.33$\pm$0.02 & 0.21$\pm$0.00 & 0.28$\pm$0.00 & 0.28$\pm$0.00 & 0.49$\pm$0.00 & 0.24$\pm$0.00 & 0.23$\pm$0.00 & \red{0.19$\pm$0.00} & \red{0.19$\pm$0.00} \\
			Basketball finer(2011) & 0.21$\pm$0.00 & 0.41$\pm$0.00 & 0.49$\pm$0.00 & 0.34$\pm$0.01 & 0.22$\pm$0.00 & 0.28$\pm$0.00 & 0.28$\pm$0.00 & 0.50$\pm$0.00 & 0.24$\pm$0.00 & 0.22$\pm$0.00 & \red{0.19$\pm$0.00} & \red{0.19$\pm$0.00} \\
			Basketball finer(2012) & 0.20$\pm$0.00 & 0.42$\pm$0.00 & 0.49$\pm$0.00 & 0.34$\pm$0.02 & 0.21$\pm$0.00 & 0.29$\pm$0.00 & 0.28$\pm$0.00 & 0.48$\pm$0.00 & 0.23$\pm$0.00 & 0.22$\pm$0.00 & \red{0.19$\pm$0.00} & \red{0.19$\pm$0.00} \\
			Basketball finer(2013) & 0.21$\pm$0.00 & 0.43$\pm$0.00 & 0.52$\pm$0.00 & 0.34$\pm$0.02 & 0.22$\pm$0.00 & 0.29$\pm$0.00 & 0.28$\pm$0.00 & 0.49$\pm$0.00 & 0.25$\pm$0.00 & 0.23$\pm$0.00 & \red{0.20$\pm$0.00} & \red{0.20$\pm$0.00} \\
			Basketball finer(2014) & 0.21$\pm$0.00 & 0.44$\pm$0.05 & 0.52$\pm$0.00 & 0.36$\pm$0.02 & 0.22$\pm$0.00 & 0.30$\pm$0.00 & 0.28$\pm$0.00 & 0.49$\pm$0.00 & 0.26$\pm$0.00 & 0.22$\pm$0.00 & \red{0.20$\pm$0.00} & \red{0.20$\pm$0.00} \\
\bottomrule
\end{tabular}}
\end{table*}
\begin{table*}[!ht]
\centering
\caption{Result table on $\mathcal{L}_\text{upset, ratio}$ for each year in the time series matches, applying the trained model for 1985 directly without further training, averaged over 10 runs, and plus/minus one standard deviation. The best is marked in \red{bold red} while the second best is highlighted in \blue{underline blue}. As MVR could not generate scores, we omit the results here.}
\label{tab:upset_ratio_matches_inductive}
\resizebox{1\linewidth}{!}{\begin{tabular}{lrrrrrrrrrrrrrrrrr}
\toprule
Data  & SpringRank & SyncRank & SerialRank & BTL & DavidScore & Eig.Cent. & PageRank & RankCent. & SVD\_RS & SVD\_NRS& Directly Apply&Train Specifically \\
\midrule
Basketball finer(1985) & \red{0.01$\pm$0.00} & \red{0.01$\pm$0.00} & \red{0.01$\pm$0.00} & \red{0.01$\pm$0.00} & 0.47$\pm$0.00 & \red{0.01$\pm$0.00} & \red{0.01$\pm$0.00} & 0.57$\pm$0.00 & 0.46$\pm$0.00 & 0.46$\pm$0.00 & \red{0.01$\pm$0.00} & \red{0.01$\pm$0.00} \\
			Basketball finer(1986) & \red{0.01$\pm$0.00} & \red{0.01$\pm$0.00} & \red{0.01$\pm$0.00} & \red{0.01$\pm$0.00} & 0.49$\pm$0.00 & \red{0.01$\pm$0.00} & \red{0.01$\pm$0.00} & 0.60$\pm$0.00 & 0.48$\pm$0.00 & 0.49$\pm$0.00 & \red{0.01$\pm$0.00} & \red{0.01$\pm$0.00} \\
			Basketball finer(1987) & \red{0.01$\pm$0.00} & \red{0.01$\pm$0.00} & \red{0.01$\pm$0.00} & \red{0.01$\pm$0.00} & 0.51$\pm$0.00 & \red{0.01$\pm$0.00} & \red{0.01$\pm$0.00} & 0.58$\pm$0.00 & 0.46$\pm$0.00 & 0.48$\pm$0.00 & \red{0.01$\pm$0.00} & \red{0.01$\pm$0.00} \\
			Basketball finer(1988) & \red{0.01$\pm$0.00} & \red{0.01$\pm$0.00} & \red{0.01$\pm$0.00} & \red{0.01$\pm$0.00} & 0.52$\pm$0.00 & \red{0.01$\pm$0.00} & \red{0.01$\pm$0.00} & 0.57$\pm$0.00 & 0.56$\pm$0.00 & 0.52$\pm$0.00 & \red{0.01$\pm$0.00} & \red{0.01$\pm$0.00} \\
			Basketball finer(1989) & \red{0.01$\pm$0.00} & \red{0.01$\pm$0.00} & \red{0.01$\pm$0.00} & \red{0.01$\pm$0.00} & 0.49$\pm$0.00 & \red{0.01$\pm$0.00} & \red{0.01$\pm$0.00} & 0.56$\pm$0.00 & 0.49$\pm$0.00 & 0.49$\pm$0.00 & \red{0.01$\pm$0.00} & \red{0.01$\pm$0.00} \\
			Basketball finer(1990) & \red{0.01$\pm$0.00} & \red{0.01$\pm$0.00} & \red{0.01$\pm$0.00} & \red{0.01$\pm$0.00} & 0.49$\pm$0.00 & \red{0.01$\pm$0.00} & \red{0.01$\pm$0.00} & 0.55$\pm$0.00 & 0.48$\pm$0.00 & 0.49$\pm$0.00 & \red{0.01$\pm$0.00} & \red{0.01$\pm$0.00} \\
			Basketball finer(1991) & \red{0.01$\pm$0.00} & \red{0.01$\pm$0.00} & \red{0.01$\pm$0.00} & \red{0.01$\pm$0.00} & 0.49$\pm$0.00 & \red{0.01$\pm$0.00} & \red{0.01$\pm$0.00} & 0.54$\pm$0.00 & 0.49$\pm$0.00 & 0.47$\pm$0.00 & \red{0.01$\pm$0.00} & \red{0.01$\pm$0.00} \\
			Basketball finer(1992) & \red{0.01$\pm$0.00} & \red{0.01$\pm$0.00} & \red{0.01$\pm$0.00} & \red{0.01$\pm$0.00} & 0.49$\pm$0.00 & \red{0.01$\pm$0.00} & \red{0.01$\pm$0.00} & 0.55$\pm$0.00 & 0.44$\pm$0.00 & 0.49$\pm$0.00 & \red{0.01$\pm$0.00} & \red{0.01$\pm$0.00} \\
			Basketball finer(1993) & \red{0.01$\pm$0.00} & \red{0.01$\pm$0.00} & \red{0.01$\pm$0.00} & \red{0.01$\pm$0.00} & 0.51$\pm$0.00 & \red{0.01$\pm$0.00} & \red{0.01$\pm$0.00} & 0.58$\pm$0.00 & 0.47$\pm$0.00 & 0.49$\pm$0.00 & \red{0.01$\pm$0.00} & \red{0.01$\pm$0.00} \\
			Basketball finer(1994) & \red{0.01$\pm$0.00} & \red{0.01$\pm$0.00} & \red{0.01$\pm$0.00} & \red{0.01$\pm$0.00} & 0.49$\pm$0.00 & \red{0.01$\pm$0.00} & \red{0.01$\pm$0.00} & 0.58$\pm$0.00 & 0.46$\pm$0.00 & 0.49$\pm$0.00 & \red{0.01$\pm$0.00} & \red{0.01$\pm$0.00} \\
			Basketball finer(1995) & \red{0.01$\pm$0.00} & \red{0.01$\pm$0.00} & \red{0.01$\pm$0.00} & \red{0.01$\pm$0.00} & 0.50$\pm$0.00 & \red{0.01$\pm$0.00} & \red{0.01$\pm$0.00} & 0.55$\pm$0.00 & 0.47$\pm$0.00 & 0.49$\pm$0.00 & \red{0.01$\pm$0.00} & \red{0.01$\pm$0.00} \\
			Basketball finer(1996) & \red{0.01$\pm$0.00} & \red{0.01$\pm$0.00} & \red{0.01$\pm$0.00} & \red{0.01$\pm$0.00} & 0.48$\pm$0.00 & \red{0.01$\pm$0.00} & \red{0.01$\pm$0.00} & 0.56$\pm$0.00 & 0.45$\pm$0.00 & 0.47$\pm$0.00 & \red{0.01$\pm$0.00} & \red{0.01$\pm$0.00} \\
			Basketball finer(1997) & 0.04$\pm$0.00 & \red{0.01$\pm$0.00} & \red{0.01$\pm$0.00} & \red{0.01$\pm$0.00} & 0.50$\pm$0.00 & \red{0.01$\pm$0.00} & \red{0.01$\pm$0.00} & 0.56$\pm$0.00 & 0.48$\pm$0.00 & 0.47$\pm$0.00 & \red{0.01$\pm$0.00} & \red{0.01$\pm$0.00} \\
			Basketball finer(1998) & \red{0.01$\pm$0.00} & \red{0.01$\pm$0.00} & \red{0.01$\pm$0.00} & \red{0.01$\pm$0.00} & 0.50$\pm$0.00 & \red{0.01$\pm$0.00} & \red{0.01$\pm$0.00} & 0.56$\pm$0.00 & 0.47$\pm$0.00 & 0.45$\pm$0.00 & \red{0.01$\pm$0.00} & \red{0.01$\pm$0.00} \\
			Basketball finer(1999) & \red{0.01$\pm$0.00} & \red{0.01$\pm$0.00} & \red{0.01$\pm$0.00} & \red{0.01$\pm$0.00} & 0.50$\pm$0.00 & \red{0.01$\pm$0.00} & \red{0.01$\pm$0.00} & 0.57$\pm$0.00 & 0.48$\pm$0.00 & 0.46$\pm$0.00 & \red{0.01$\pm$0.00} & \red{0.01$\pm$0.00} \\
			Basketball finer(2000) & \red{0.01$\pm$0.00} & \red{0.01$\pm$0.00} & \red{0.01$\pm$0.00} & \red{0.01$\pm$0.00} & 0.51$\pm$0.00 & \red{0.01$\pm$0.00} & \red{0.01$\pm$0.00} & 0.55$\pm$0.00 & 0.50$\pm$0.00 & 0.47$\pm$0.00 & \red{0.01$\pm$0.00} & \red{0.01$\pm$0.00} \\
			Basketball finer(2001) & \red{0.01$\pm$0.00} & \red{0.01$\pm$0.00} & \red{0.01$\pm$0.00} & \red{0.01$\pm$0.00} & 0.51$\pm$0.00 & \red{0.01$\pm$0.00} & \red{0.01$\pm$0.00} & 0.59$\pm$0.00 & 0.51$\pm$0.00 & 0.47$\pm$0.00 & \red{0.01$\pm$0.00} & \red{0.01$\pm$0.00} \\
			Basketball finer(2002) & \red{0.01$\pm$0.00} & \red{0.01$\pm$0.00} & \red{0.01$\pm$0.00} & \red{0.01$\pm$0.00} & 0.51$\pm$0.00 & \red{0.01$\pm$0.00} & \red{0.01$\pm$0.00} & 0.57$\pm$0.00 & 0.51$\pm$0.00 & 0.50$\pm$0.00 & \red{0.01$\pm$0.00} & \red{0.01$\pm$0.00} \\
			Basketball finer(2003) & \red{0.01$\pm$0.00} & \red{0.01$\pm$0.00} & \red{0.01$\pm$0.00} & \red{0.01$\pm$0.00} & 0.50$\pm$0.00 & \red{0.01$\pm$0.00} & \red{0.01$\pm$0.00} & 0.58$\pm$0.00 & 0.47$\pm$0.00 & 0.48$\pm$0.00 & \red{0.01$\pm$0.00} & \red{0.01$\pm$0.00} \\
			Basketball finer(2004) & \red{0.01$\pm$0.00} & \red{0.01$\pm$0.00} & \red{0.01$\pm$0.00} & \red{0.01$\pm$0.00} & 0.52$\pm$0.00 & \red{0.01$\pm$0.00} & \red{0.01$\pm$0.00} & 0.57$\pm$0.00 & 0.45$\pm$0.00 & 0.47$\pm$0.00 & \red{0.01$\pm$0.00} & \red{0.01$\pm$0.00} \\
			Basketball finer(2005) & \red{0.01$\pm$0.00} & \red{0.01$\pm$0.00} & \red{0.01$\pm$0.00} & \red{0.01$\pm$0.00} & 0.51$\pm$0.00 & \red{0.01$\pm$0.00} & \red{0.01$\pm$0.00} & 0.60$\pm$0.00 & 0.43$\pm$0.00 & 0.47$\pm$0.00 & \red{0.01$\pm$0.00} & \red{0.01$\pm$0.00} \\
			Basketball finer(2006) & \red{0.01$\pm$0.00} & \red{0.01$\pm$0.00} & \red{0.01$\pm$0.00} & \red{0.01$\pm$0.00} & 0.51$\pm$0.00 & \red{0.01$\pm$0.00} & \red{0.01$\pm$0.00} & 0.57$\pm$0.00 & 0.46$\pm$0.00 & 0.48$\pm$0.00 & \red{0.01$\pm$0.00} & \red{0.01$\pm$0.00} \\
			Basketball finer(2007) & \red{0.01$\pm$0.00} & \red{0.01$\pm$0.00} & \red{0.01$\pm$0.00} & \red{0.01$\pm$0.00} & 0.53$\pm$0.00 & \red{0.01$\pm$0.00} & \red{0.01$\pm$0.00} & 0.53$\pm$0.00 & 0.49$\pm$0.00 & 0.48$\pm$0.00 & \red{0.01$\pm$0.00} & \red{0.01$\pm$0.00} \\
			Basketball finer(2008) & \red{0.01$\pm$0.00} & \red{0.01$\pm$0.00} & \red{0.01$\pm$0.00} & \red{0.01$\pm$0.00} & 0.51$\pm$0.00 & \red{0.01$\pm$0.00} & \red{0.01$\pm$0.00} & 0.59$\pm$0.00 & 0.47$\pm$0.00 & 0.47$\pm$0.00 & \red{0.01$\pm$0.00} & \red{0.01$\pm$0.00} \\
			Basketball finer(2009) & \red{0.01$\pm$0.00} & \red{0.01$\pm$0.00} & \red{0.01$\pm$0.00} & \red{0.01$\pm$0.00} & 0.52$\pm$0.00 & \red{0.01$\pm$0.00} & \red{0.01$\pm$0.00} & 0.57$\pm$0.00 & 0.50$\pm$0.00 & 0.49$\pm$0.00 & \red{0.01$\pm$0.00} & \red{0.01$\pm$0.00} \\
			Basketball finer(2010) & \red{0.01$\pm$0.00} & \red{0.01$\pm$0.00} & \red{0.01$\pm$0.00} & \red{0.01$\pm$0.00} & 0.52$\pm$0.00 & \red{0.01$\pm$0.00} & \red{0.01$\pm$0.00} & 0.58$\pm$0.00 & 0.48$\pm$0.00 & 0.47$\pm$0.00 & \red{0.01$\pm$0.00} & \red{0.01$\pm$0.00} \\
			Basketball finer(2011) & \red{0.01$\pm$0.00} & \red{0.01$\pm$0.00} & \red{0.01$\pm$0.00} & \red{0.01$\pm$0.00} & 0.52$\pm$0.00 & \red{0.01$\pm$0.00} & \red{0.01$\pm$0.00} & 0.59$\pm$0.00 & 0.46$\pm$0.00 & 0.47$\pm$0.00 & \red{0.01$\pm$0.00} & \red{0.01$\pm$0.00} \\
			Basketball finer(2012) & \red{0.01$\pm$0.00} & \red{0.01$\pm$0.00} & \red{0.01$\pm$0.00} & \red{0.01$\pm$0.00} & 0.52$\pm$0.00 & \red{0.01$\pm$0.00} & \red{0.01$\pm$0.00} & 0.57$\pm$0.00 & 0.48$\pm$0.00 & 0.48$\pm$0.00 & \red{0.01$\pm$0.00} & \red{0.01$\pm$0.00} \\
			Basketball finer(2013) & \red{0.01$\pm$0.00} & \red{0.01$\pm$0.00} & \red{0.01$\pm$0.00} & \red{0.01$\pm$0.00} & 0.52$\pm$0.00 & \red{0.01$\pm$0.00} & \red{0.01$\pm$0.00} & 0.56$\pm$0.00 & 0.46$\pm$0.00 & 0.48$\pm$0.00 & \red{0.01$\pm$0.00} & \red{0.01$\pm$0.00} \\
			Basketball finer(2014) & \red{0.01$\pm$0.00} & \red{0.01$\pm$0.00} & \red{0.01$\pm$0.00} & \red{0.01$\pm$0.00} & 0.53$\pm$0.00 & \red{0.01$\pm$0.00} & \red{0.01$\pm$0.00} & 0.56$\pm$0.00 & 0.48$\pm$0.00 & 0.47$\pm$0.00 & \red{0.01$\pm$0.00} & \red{0.01$\pm$0.00} \\
\bottomrule
\end{tabular}}
\end{table*}

\section{Improvement on Baselines when Employed as Initial Guess for ``Proximal Baseline" Variant}
\label{appendix_sec:improvement_over_baselines}
Table \ref{tab:upset_simple_improve_over_baseline} shows improvements on $\mathcal{L}_{upset, simple}$ by ``proximal baseline" when setting a certain baseline as $\mathbf{r}'$. Across all data sets, ``proximal baseline" improves the most, by 1.02, with SyncRank as initial guess,
while the average improvement for SpringRank, SerialRank, BTL, Eig.Cent., PageRank and SVD\_NRS are 0.07, 0.82, 0.22, 0.19, 0.21, and 0.12, respectively.

Table \ref{tab:upset_naive_improve_over_baseline} shows improvements on $\mathcal{L}_{upset, naive}$ by ``proximal baseline" when setting a certain baseline as $\mathbf{r}'$. Across all real data sets, ``proximal baseline" improves the most, by 0.24, again with SyncRank as initial guess,
while the average improvement for SpringRank, SerialRank, BTL, Eig.Cent., PageRank and SVD\_NRS are 0.00, 0.18, 0.04, 0.03, 0.03, and 0.01, respectively.

\begin{table*}[!ht]
\centering
\caption{Result table on $\mathcal{L}_\text{upset, simple}$ improvement with ``proximal baseline" training starting from a baseline as initial guess, for individual directed graphs, averaged over 10 runs, and plus/minus one standard deviation. The best is marked in \red{bold red} while the second best is highlighted in \blue{underline blue}.}
\label{tab:upset_simple_improve_over_baseline}
\resizebox{0.5\linewidth}{!}{\begin{tabular}{lrrrrrrrrrrrrrrrrrr}
\toprule
Data  & SpringRank & SyncRank & SerialRank & BTL  & Eig.Cent. & PageRank  & SVD\_NRS  \\
\midrule
{\it HeadToHead} & -0.01$\pm$0.00 & \blue{-0.98$\pm$0.00} & \red{-1.02$\pm$0.00} & -0.12$\pm$0.03 & -0.48$\pm$0.00 & -0.37$\pm$0.00 & -0.43$\pm$0.00 \\
			{\it Finance} & -0.63$\pm$0.00 & \red{-0.98$\pm$0.00} & -0.61$\pm$0.00 & \blue{-0.78$\pm$0.01} & -0.74$\pm$0.00 & -0.75$\pm$0.00 & -0.64$\pm$0.00 \\
			{\it Animal} & -0.09$\pm$0.05 & \blue{-1.36$\pm$0.24} & \red{-1.42$\pm$0.50} & -0.01$\pm$0.01 & -0.02$\pm$0.02 & -0.08$\pm$0.07 & -0.02$\pm$0.07 \\
			{\it Faculty: Business} & -0.03$\pm$0.00 & \blue{-0.47$\pm$0.00} & \red{-0.70$\pm$0.04} & -0.00$\pm$0.02 & -0.02$\pm$0.03 & -0.04$\pm$0.03 & -0.03$\pm$0.02 \\
			{\it Faculty: CS} & -0.00$\pm$0.00 & \blue{-0.66$\pm$0.10} & \red{-0.72$\pm$0.07} & 0.00$\pm$0.01 & 0.01$\pm$0.03 & -0.00$\pm$0.00 & -0.10$\pm$0.04 \\
			{\it Faculty: History} & -0.02$\pm$0.01 & \blue{-0.27$\pm$0.00} & \red{-1.77$\pm$0.82} & 0.00$\pm$0.01 & -0.01$\pm$0.03 & -0.00$\pm$0.00 & -0.00$\pm$0.00 \\
			{\it Football(2009)} & -0.00$\pm$0.00 & \red{-0.70$\pm$0.46} & -0.16$\pm$0.06 & -0.03$\pm$0.08 & -0.12$\pm$0.05 & \blue{-0.24$\pm$0.06} & -0.02$\pm$0.01 \\
			{\it Football(2010)} & -0.24$\pm$0.06 & \red{-0.85$\pm$0.17} & \blue{-0.58$\pm$0.06} & -0.22$\pm$0.09 & -0.40$\pm$0.04 & -0.41$\pm$0.04 & -0.21$\pm$0.05 \\
			{\it Football(2011)} & -0.00$\pm$0.01 & \red{-0.74$\pm$0.37} & -0.10$\pm$0.06 & -0.03$\pm$0.04 & -0.13$\pm$0.09 & \blue{-0.19$\pm$0.09} & -0.02$\pm$0.09 \\
			{\it Football(2012)} & -0.17$\pm$0.06 & \red{-0.72$\pm$0.04} & \blue{-0.57$\pm$0.07} & -0.14$\pm$0.10 & -0.10$\pm$0.09 & -0.22$\pm$0.03 & -0.05$\pm$0.08 \\
			{\it Football(2013)} & -0.03$\pm$0.03 & \red{-1.08$\pm$0.10} & -0.08$\pm$0.06 & -0.06$\pm$0.05 & \blue{-0.16$\pm$0.06} & -0.15$\pm$0.03 & -0.04$\pm$0.03 \\
			{\it Football(2014)} & -0.17$\pm$0.05 & \red{-0.83$\pm$0.20} & \blue{-0.59$\pm$0.08} & -0.17$\pm$0.06 & -0.27$\pm$0.09 & -0.24$\pm$0.07 & -0.07$\pm$0.09 \\
			{\it Football finer(2009)} & -0.03$\pm$0.03 & \red{-1.04$\pm$0.05} & -0.18$\pm$0.06 & -0.06$\pm$0.14 & -0.17$\pm$0.11 & \blue{-0.24$\pm$0.12} & -0.02$\pm$0.03 \\
			{\it Football finer(2010)} & -0.30$\pm$0.02 & \red{-0.63$\pm$0.02} & -0.45$\pm$0.02 & -0.31$\pm$0.09 & -0.37$\pm$0.00 & \blue{-0.49$\pm$0.00} & -0.17$\pm$0.00 \\
			{\it Football finer(2011)} & -0.00$\pm$0.02 & \red{-0.75$\pm$0.04} & -0.14$\pm$0.06 & -0.03$\pm$0.06 & -0.08$\pm$0.05 & \blue{-0.19$\pm$0.04} & -0.01$\pm$0.02 \\
			{\it Football finer(2012)} & -0.03$\pm$0.03 & \red{-0.90$\pm$0.18} & -0.12$\pm$0.08 & -0.03$\pm$0.06 & -0.14$\pm$0.06 & \blue{-0.29$\pm$0.05} & -0.01$\pm$0.03 \\
			{\it Football finer(2013)} & -0.13$\pm$0.07 & \red{-1.01$\pm$0.02} & -0.12$\pm$0.07 & -0.08$\pm$0.05 & -0.19$\pm$0.11 & \blue{-0.26$\pm$0.09} & -0.07$\pm$0.03 \\
			{\it Football finer(2014)} & -0.16$\pm$0.00 & \red{-0.74$\pm$0.04} & \blue{-0.57$\pm$0.00} & -0.19$\pm$0.04 & -0.23$\pm$0.00 & -0.23$\pm$0.00 & -0.09$\pm$0.03 \\
			{\it Basketball(1985)} & -0.01$\pm$0.01 & \blue{-0.91$\pm$0.00} & \red{-1.26$\pm$0.04} & -0.10$\pm$0.07 & -0.05$\pm$0.01 & -0.07$\pm$0.02 & -0.07$\pm$0.01 \\
			{\it Basketball(1986)} & -0.01$\pm$0.01 & \blue{-1.10$\pm$0.00} & \red{-1.11$\pm$0.03} & -0.14$\pm$0.02 & -0.07$\pm$0.01 & -0.06$\pm$0.02 & -0.03$\pm$0.02 \\
			{\it Basketball(1987)} & -0.00$\pm$0.01 & \red{-1.02$\pm$0.01} & \blue{-0.94$\pm$0.03} & -0.11$\pm$0.02 & -0.11$\pm$0.02 & -0.01$\pm$0.01 & -0.05$\pm$0.02 \\
			{\it Basketball(1988)} & -0.00$\pm$0.00 & \blue{-0.96$\pm$0.00} & \red{-1.07$\pm$0.04} & -0.11$\pm$0.04 & -0.06$\pm$0.02 & -0.04$\pm$0.02 & -0.03$\pm$0.01 \\
			{\it Basketball(1989)} & 0.00$\pm$0.00 & \red{-1.04$\pm$0.00} & \blue{-0.99$\pm$0.02} & -0.12$\pm$0.02 & -0.00$\pm$0.00 & -0.02$\pm$0.01 & -0.07$\pm$0.01 \\
			{\it Basketball(1990)} & -0.00$\pm$0.01 & \blue{-0.98$\pm$0.00} & \red{-1.09$\pm$0.03} & -0.10$\pm$0.11 & -0.09$\pm$0.01 & -0.07$\pm$0.01 & -0.05$\pm$0.01 \\
			{\it Basketball(1991)} & -0.02$\pm$0.01 & \blue{-1.10$\pm$0.00} & \red{-1.15$\pm$0.04} & -0.09$\pm$0.06 & -0.03$\pm$0.01 & -0.01$\pm$0.02 & -0.08$\pm$0.01 \\
			{\it Basketball(1992)} & -0.01$\pm$0.01 & \red{-1.07$\pm$0.00} & \blue{-1.05$\pm$0.02} & -0.06$\pm$0.10 & -0.04$\pm$0.01 & 0.00$\pm$0.00 & -0.07$\pm$0.01 \\
			{\it Basketball(1993)} & -0.01$\pm$0.01 & \blue{-1.00$\pm$0.01} & \red{-1.18$\pm$0.06} & -0.09$\pm$0.07 & -0.05$\pm$0.01 & -0.06$\pm$0.01 & -0.07$\pm$0.01 \\
			{\it Basketball(1994)} & 0.00$\pm$0.00 & \blue{-0.98$\pm$0.00} & \red{-1.15$\pm$0.02} & -0.02$\pm$0.09 & -0.02$\pm$0.02 & -0.03$\pm$0.01 & -0.05$\pm$0.01 \\
			{\it Basketball(1995)} & 0.00$\pm$0.00 & \red{-1.06$\pm$0.01} & \blue{-1.01$\pm$0.03} & -0.03$\pm$0.10 & -0.02$\pm$0.01 & -0.06$\pm$0.01 & -0.04$\pm$0.01 \\
			{\it Basketball(1996)} & -0.00$\pm$0.01 & \blue{-0.90$\pm$0.01} & \red{-1.11$\pm$0.02} & -0.08$\pm$0.08 & -0.02$\pm$0.01 & -0.02$\pm$0.01 & -0.07$\pm$0.01 \\
			{\it Basketball(1997)} & -0.01$\pm$0.01 & \blue{-0.96$\pm$0.01} & \red{-1.02$\pm$0.04} & -0.12$\pm$0.04 & -0.05$\pm$0.01 & -0.03$\pm$0.01 & -0.07$\pm$0.01 \\
			{\it Basketball(1998)} & -0.00$\pm$0.01 & \blue{-0.95$\pm$0.01} & \red{-1.00$\pm$0.06} & -0.14$\pm$0.02 & -0.06$\pm$0.01 & -0.02$\pm$0.01 & -0.07$\pm$0.01 \\
			{\it Basketball(1999)} & -0.00$\pm$0.01 & \blue{-0.83$\pm$0.00} & \red{-1.10$\pm$0.02} & -0.09$\pm$0.07 & -0.07$\pm$0.01 & -0.07$\pm$0.01 & -0.07$\pm$0.01 \\
			{\it Basketball(2000)} & -0.01$\pm$0.01 & \blue{-1.00$\pm$0.00} & \red{-1.03$\pm$0.03} & -0.09$\pm$0.02 & -0.02$\pm$0.01 & -0.03$\pm$0.01 & -0.05$\pm$0.01 \\
			{\it Basketball(2001)} & -0.00$\pm$0.00 & \blue{-1.02$\pm$0.00} & \red{-1.10$\pm$0.11} & -0.09$\pm$0.08 & -0.05$\pm$0.02 & -0.02$\pm$0.02 & -0.11$\pm$0.02 \\
			{\it Basketball(2002)} & -0.02$\pm$0.00 & \blue{-0.95$\pm$0.00} & \red{-1.05$\pm$0.02} & -0.09$\pm$0.06 & -0.01$\pm$0.01 & -0.06$\pm$0.01 & -0.05$\pm$0.01 \\
			{\it Basketball(2003)} & 0.00$\pm$0.00 & \blue{-1.07$\pm$0.00} & \red{-1.13$\pm$0.01} & -0.08$\pm$0.07 & -0.08$\pm$0.01 & -0.07$\pm$0.01 & -0.07$\pm$0.01 \\
			{\it Basketball(2004)} & -0.00$\pm$0.00 & \blue{-0.98$\pm$0.00} & \red{-1.14$\pm$0.11} & -0.09$\pm$0.06 & -0.05$\pm$0.01 & -0.06$\pm$0.01 & -0.07$\pm$0.01 \\
			{\it Basketball(2005)} & 0.00$\pm$0.00 & \blue{-1.05$\pm$0.01} & \red{-1.06$\pm$0.06} & -0.14$\pm$0.03 & -0.00$\pm$0.01 & -0.06$\pm$0.01 & -0.07$\pm$0.01 \\
			{\it Basketball(2006)} & 0.00$\pm$0.00 & \blue{-1.04$\pm$0.00} & \red{-1.09$\pm$0.03} & 0.02$\pm$0.09 & 0.00$\pm$0.00 & -0.03$\pm$0.01 & -0.04$\pm$0.01 \\
			{\it Basketball(2007)} & -0.00$\pm$0.00 & \red{-1.10$\pm$0.00} & \blue{-1.03$\pm$0.08} & -0.09$\pm$0.02 & -0.03$\pm$0.01 & -0.02$\pm$0.01 & -0.04$\pm$0.01 \\
			{\it Basketball(2008)} & 0.00$\pm$0.00 & \blue{-0.97$\pm$0.00} & \red{-1.10$\pm$0.03} & -0.03$\pm$0.07 & -0.04$\pm$0.01 & -0.05$\pm$0.01 & -0.06$\pm$0.01 \\
			{\it Basketball(2009)} & -0.00$\pm$0.00 & \blue{-1.02$\pm$0.01} & \red{-1.17$\pm$0.02} & -0.12$\pm$0.02 & -0.03$\pm$0.01 & -0.03$\pm$0.01 & -0.06$\pm$0.01 \\
			{\it Basketball(2010)} & -0.00$\pm$0.00 & \blue{-0.83$\pm$0.00} & \red{-1.06$\pm$0.01} & -0.03$\pm$0.13 & -0.04$\pm$0.01 & -0.04$\pm$0.01 & -0.05$\pm$0.01 \\
			{\it Basketball(2011)} & -0.00$\pm$0.01 & \blue{-0.90$\pm$0.00} & \red{-1.09$\pm$0.03} & 0.00$\pm$0.08 & -0.04$\pm$0.01 & -0.05$\pm$0.01 & -0.02$\pm$0.01 \\
			{\it Basketball(2012)} & -0.00$\pm$0.00 & \red{-1.11$\pm$0.00} & \red{-1.11$\pm$0.05} & -0.07$\pm$0.08 & -0.04$\pm$0.01 & -0.03$\pm$0.02 & -0.04$\pm$0.01 \\
			{\it Basketball(2013)} & -0.00$\pm$0.00 & \blue{-0.95$\pm$0.00} & \red{-1.09$\pm$0.06} & -0.11$\pm$0.06 & -0.06$\pm$0.01 & -0.03$\pm$0.01 & -0.07$\pm$0.01 \\
			{\it Basketball(2014)} & -0.00$\pm$0.01 & \blue{-0.96$\pm$0.00} & \red{-1.14$\pm$0.05} & -0.00$\pm$0.10 & -0.01$\pm$0.01 & -0.04$\pm$0.01 & -0.02$\pm$0.01 \\
			{\it Basketball finer(1985)} & -0.01$\pm$0.01 & \blue{-0.92$\pm$0.00} & \red{-0.96$\pm$0.10} & -0.56$\pm$0.07 & -0.35$\pm$0.01 & -0.32$\pm$0.02 & -0.07$\pm$0.01 \\
			{\it Basketball finer(1986)} & -0.00$\pm$0.00 & \red{-1.12$\pm$0.00} & \blue{-1.01$\pm$0.03} & -0.52$\pm$0.08 & -0.30$\pm$0.02 & -0.27$\pm$0.01 & -0.05$\pm$0.01 \\
			{\it Basketball finer(1987)} & -0.01$\pm$0.00 & \red{-1.02$\pm$0.01} & \blue{-0.90$\pm$0.02} & -0.45$\pm$0.07 & -0.25$\pm$0.01 & -0.28$\pm$0.01 & -0.05$\pm$0.01 \\
			{\it Basketball finer(1988)} & -0.01$\pm$0.00 & \red{-1.09$\pm$0.01} & \blue{-1.02$\pm$0.02} & -0.55$\pm$0.11 & -0.38$\pm$0.02 & -0.34$\pm$0.01 & -0.03$\pm$0.01 \\
			{\it Basketball finer(1989)} & -0.00$\pm$0.00 & \red{-0.96$\pm$0.00} & \blue{-0.87$\pm$0.03} & -0.51$\pm$0.06 & -0.27$\pm$0.01 & -0.29$\pm$0.01 & -0.07$\pm$0.01 \\
			{\it Basketball finer(1990)} & -0.01$\pm$0.01 & \blue{-0.96$\pm$0.00} & \red{-1.02$\pm$0.01} & -0.55$\pm$0.05 & -0.42$\pm$0.01 & -0.31$\pm$0.01 & -0.03$\pm$0.01 \\
			{\it Basketball finer(1991)} & -0.01$\pm$0.00 & \red{-1.13$\pm$0.00} & \blue{-1.03$\pm$0.00} & -0.45$\pm$0.06 & -0.46$\pm$0.02 & -0.41$\pm$0.01 & -0.05$\pm$0.01 \\
			{\it Basketball finer(1992)} & -0.00$\pm$0.00 & \red{-1.04$\pm$0.00} & \blue{-0.99$\pm$0.02} & -0.46$\pm$0.06 & -0.41$\pm$0.02 & -0.37$\pm$0.02 & -0.07$\pm$0.01 \\
			{\it Basketball finer(1993)} & -0.00$\pm$0.01 & \blue{-0.98$\pm$0.01} & \red{-1.08$\pm$0.10} & -0.49$\pm$0.06 & -0.35$\pm$0.01 & -0.27$\pm$0.01 & -0.05$\pm$0.01 \\
			{\it Basketball finer(1994)} & 0.00$\pm$0.00 & \blue{-1.02$\pm$0.00} & \red{-1.11$\pm$0.04} & -0.50$\pm$0.07 & -0.39$\pm$0.02 & -0.28$\pm$0.01 & -0.03$\pm$0.01 \\
			{\it Basketball finer(1995)} & -0.01$\pm$0.01 & \red{-1.06$\pm$0.01} & \blue{-0.98$\pm$0.01} & -0.48$\pm$0.08 & -0.35$\pm$0.01 & -0.30$\pm$0.01 & -0.06$\pm$0.01 \\
			{\it Basketball finer(1996)} & -0.00$\pm$0.01 & \blue{-0.90$\pm$0.00} & \red{-0.95$\pm$0.00} & -0.44$\pm$0.06 & -0.27$\pm$0.02 & -0.25$\pm$0.02 & -0.06$\pm$0.01 \\
			{\it Basketball finer(1997)} & -0.01$\pm$0.01 & \red{-1.02$\pm$0.01} & \blue{-0.98$\pm$0.01} & -0.47$\pm$0.04 & -0.28$\pm$0.01 & -0.24$\pm$0.01 & -0.08$\pm$0.01 \\
			{\it Basketball finer(1998)} & -0.01$\pm$0.01 & \red{-0.96$\pm$0.01} & \blue{-0.95$\pm$0.01} & -0.44$\pm$0.06 & -0.25$\pm$0.01 & -0.24$\pm$0.01 & -0.08$\pm$0.01 \\
			{\it Basketball finer(1999)} & -0.00$\pm$0.01 & \blue{-0.90$\pm$0.00} & \red{-1.05$\pm$0.02} & -0.44$\pm$0.08 & -0.27$\pm$0.01 & -0.19$\pm$0.01 & -0.08$\pm$0.01 \\
			{\it Basketball finer(2000)} & 0.00$\pm$0.00 & \blue{-0.97$\pm$0.00} & \red{-0.99$\pm$0.02} & -0.42$\pm$0.05 & -0.31$\pm$0.02 & -0.25$\pm$0.01 & -0.06$\pm$0.01 \\
			{\it Basketball finer(2001)} & -0.00$\pm$0.00 & \blue{-0.96$\pm$0.00} & \red{-1.06$\pm$0.00} & -0.46$\pm$0.06 & -0.31$\pm$0.01 & -0.25$\pm$0.02 & -0.10$\pm$0.01 \\
			{\it Basketball finer(2002)} & -0.04$\pm$0.02 & \red{-0.99$\pm$0.01} & \blue{-0.92$\pm$0.02} & -0.52$\pm$0.09 & -0.30$\pm$0.02 & -0.25$\pm$0.01 & -0.05$\pm$0.01 \\
			{\it Basketball finer(2003)} & 0.00$\pm$0.00 & \red{-1.01$\pm$0.00} & \blue{-0.98$\pm$0.07} & -0.49$\pm$0.10 & -0.25$\pm$0.01 & -0.21$\pm$0.01 & -0.05$\pm$0.01 \\
			{\it Basketball finer(2004)} & -0.02$\pm$0.01 & \red{-1.00$\pm$0.06} & \blue{-0.93$\pm$0.01} & -0.48$\pm$0.07 & -0.29$\pm$0.01 & -0.24$\pm$0.01 & -0.07$\pm$0.01 \\
			{\it Basketball finer(2005)} & -0.01$\pm$0.01 & \red{-1.08$\pm$0.01} & \blue{-0.93$\pm$0.00} & -0.42$\pm$0.05 & -0.22$\pm$0.01 & -0.18$\pm$0.01 & -0.06$\pm$0.01 \\
			{\it Basketball finer(2006)} & -0.00$\pm$0.01 & \red{-1.01$\pm$0.00} & \blue{-0.97$\pm$0.00} & -0.42$\pm$0.07 & -0.24$\pm$0.02 & -0.17$\pm$0.01 & -0.04$\pm$0.01 \\
			{\it Basketball finer(2007)} & -0.00$\pm$0.01 & \red{-1.05$\pm$0.01} & \blue{-0.97$\pm$0.14} & -0.46$\pm$0.06 & -0.25$\pm$0.02 & -0.18$\pm$0.02 & -0.06$\pm$0.01 \\
			{\it Basketball finer(2008)} & -0.00$\pm$0.00 & \blue{-0.94$\pm$0.00} & \red{-0.98$\pm$0.00} & -0.42$\pm$0.09 & -0.29$\pm$0.01 & -0.21$\pm$0.01 & -0.04$\pm$0.01 \\
			{\it Basketball finer(2009)} & -0.00$\pm$0.00 & \blue{-0.96$\pm$0.00} & \red{-1.00$\pm$0.04} & -0.45$\pm$0.04 & -0.26$\pm$0.01 & -0.21$\pm$0.02 & -0.06$\pm$0.01 \\
			{\it Basketball finer(2010)} & -0.00$\pm$0.00 & \blue{-0.93$\pm$0.00} & \red{-0.98$\pm$0.02} & -0.41$\pm$0.07 & -0.18$\pm$0.01 & -0.21$\pm$0.01 & -0.05$\pm$0.01 \\
			{\it Basketball finer(2011)} & -0.01$\pm$0.00 & \blue{-0.88$\pm$0.00} & \red{-0.99$\pm$0.01} & -0.44$\pm$0.06 & -0.21$\pm$0.01 & -0.22$\pm$0.01 & -0.03$\pm$0.01 \\
			{\it Basketball finer(2012)} & 0.00$\pm$0.00 & \blue{-0.94$\pm$0.01} & \red{-0.98$\pm$0.03} & -0.45$\pm$0.07 & -0.26$\pm$0.01 & -0.21$\pm$0.01 & -0.03$\pm$0.00 \\
			{\it Basketball finer(2013)} & -0.00$\pm$0.01 & \blue{-0.95$\pm$0.01} & \red{-1.06$\pm$0.00} & -0.45$\pm$0.08 & -0.27$\pm$0.01 & -0.25$\pm$0.01 & -0.07$\pm$0.01 \\
			{\it Basketball finer(2014)} & -0.01$\pm$0.00 & \blue{-0.99$\pm$0.19} & \red{-1.07$\pm$0.00} & -0.47$\pm$0.09 & -0.24$\pm$0.02 & -0.20$\pm$0.01 & -0.02$\pm$0.01 \\
			{\it ERO(p=0.05, style=uniform,$\eta$=0)} & -0.05$\pm$0.02 & \red{-1.71$\pm$0.01} & \blue{-1.64$\pm$0.01} & -0.16$\pm$0.11 & 0.08$\pm$0.01 & -0.00$\pm$0.01 & -0.06$\pm$0.02 \\
			{\it ERO(p=0.05, style=gamma,$\eta$=0)} & -0.13$\pm$0.01 & \red{-1.76$\pm$0.03} & \blue{-1.61$\pm$0.40} & 0.22$\pm$0.03 & 0.31$\pm$0.03 & 0.22$\pm$0.02 & 0.07$\pm$0.01 \\
			{\it ERO(p=0.05, style=uniform,$\eta$=0.1)} & -0.00$\pm$0.00 & \red{-1.49$\pm$0.02} & \blue{-1.30$\pm$0.01} & -0.05$\pm$0.04 & -0.27$\pm$0.01 & -0.24$\pm$0.01 & -0.00$\pm$0.01 \\
			{\it ERO(p=0.05, style=gamma,$\eta$=0.1)} & -0.00$\pm$0.00 & \red{-1.60$\pm$0.01} & \blue{-1.19$\pm$0.04} & 0.05$\pm$0.01 & -0.05$\pm$0.03 & 0.02$\pm$0.03 & 0.02$\pm$0.01 \\
			{\it ERO(p=0.05, style=uniform,$\eta$=0.2)} & -0.01$\pm$0.02 & \red{-1.37$\pm$0.01} & \blue{-1.21$\pm$0.04} & -0.01$\pm$0.05 & -0.22$\pm$0.04 & -0.31$\pm$0.04 & -0.02$\pm$0.03 \\
			{\it ERO(p=0.05, style=gamma,$\eta$=0.2)} & -0.00$\pm$0.01 & \red{-1.43$\pm$0.02} & \blue{-1.14$\pm$0.02} & 0.00$\pm$0.01 & -0.17$\pm$0.03 & -0.14$\pm$0.02 & -0.01$\pm$0.01 \\
			{\it ERO(p=0.05, style=uniform,$\eta$=0.3)} & -0.01$\pm$0.01 & \red{-1.10$\pm$0.02} & \blue{-0.98$\pm$0.11} & 0.05$\pm$0.05 & -0.14$\pm$0.05 & -0.21$\pm$0.06 & -0.01$\pm$0.03 \\
			{\it ERO(p=0.05, style=gamma,$\eta$=0.3)} & -0.03$\pm$0.01 & \red{-1.16$\pm$0.02} & \blue{-1.05$\pm$0.02} & -0.04$\pm$0.02 & -0.29$\pm$0.01 & -0.30$\pm$0.02 & -0.05$\pm$0.01 \\
			{\it ERO(p=0.05, style=uniform,$\eta$=0.4)} & -0.00$\pm$0.00 & \red{-1.00$\pm$0.05} & \blue{-0.75$\pm$0.03} & 0.04$\pm$0.05 & -0.01$\pm$0.02 & -0.09$\pm$0.04 & 0.02$\pm$0.01 \\
			{\it ERO(p=0.05, style=gamma,$\eta$=0.4)} & -0.11$\pm$0.03 & \red{-0.98$\pm$0.05} & \red{-0.98$\pm$0.09} & -0.10$\pm$0.03 & -0.31$\pm$0.06 & -0.34$\pm$0.02 & -0.17$\pm$0.15 \\
			{\it ERO(p=0.05, style=uniform,$\eta$=0.5)} & -0.00$\pm$0.00 & \red{-0.84$\pm$0.01} & \blue{-0.64$\pm$0.02} & 0.02$\pm$0.15 & -0.01$\pm$0.03 & -0.04$\pm$0.05 & -0.02$\pm$0.04 \\
			{\it ERO(p=0.05, style=gamma,$\eta$=0.5)} & -0.27$\pm$0.01 & \blue{-0.84$\pm$0.04} & \red{-0.87$\pm$0.13} & -0.25$\pm$0.04 & -0.38$\pm$0.06 & -0.44$\pm$0.02 & -0.66$\pm$0.02 \\
			{\it ERO(p=0.05, style=uniform,$\eta$=0.6)} & -0.00$\pm$0.00 & \red{-0.66$\pm$0.01} & \blue{-0.57$\pm$0.01} & 0.00$\pm$0.03 & -0.00$\pm$0.02 & -0.05$\pm$0.05 & -0.18$\pm$0.05 \\
			{\it ERO(p=0.05, style=gamma,$\eta$=0.6)} & -0.39$\pm$0.02 & \blue{-0.83$\pm$0.09} & \red{-0.92$\pm$0.02} & -0.32$\pm$0.09 & -0.53$\pm$0.01 & -0.52$\pm$0.03 & -0.76$\pm$0.09 \\
			{\it ERO(p=0.05, style=uniform,$\eta$=0.7)} & -0.01$\pm$0.03 & \red{-0.64$\pm$0.02} & \blue{-0.56$\pm$0.26} & -0.01$\pm$0.04 & -0.14$\pm$0.27 & -0.23$\pm$0.30 & -0.48$\pm$0.27 \\
			{\it ERO(p=0.05, style=gamma,$\eta$=0.7)} & -0.42$\pm$0.00 & -0.86$\pm$0.05 & \red{-0.89$\pm$0.05} & -0.38$\pm$0.06 & -0.59$\pm$0.02 & -0.54$\pm$0.02 & \blue{-0.87$\pm$0.04} \\
			{\it ERO(p=0.05, style=uniform,$\eta$=0.8)} & -0.24$\pm$0.20 & -0.61$\pm$0.23 & \blue{-0.69$\pm$0.23} & -0.19$\pm$0.21 & -0.26$\pm$0.28 & -0.26$\pm$0.24 & \red{-0.72$\pm$0.24} \\
			{\it ERO(p=0.05, style=gamma,$\eta$=0.8)} & -0.51$\pm$0.00 & -0.84$\pm$0.03 & \red{-0.90$\pm$0.03} & -0.44$\pm$0.05 & -0.61$\pm$0.03 & -0.55$\pm$0.05 & \blue{-0.88$\pm$0.09} \\
			{\it ERO(p=1, style=uniform,$\eta$=0)} & 0.00$\pm$0.00 & \red{-1.85$\pm$0.00} & 0.00$\pm$0.00 & \blue{-0.23$\pm$0.11} & 0.00$\pm$0.00 & 0.00$\pm$0.00 & 0.00$\pm$0.00 \\
			{\it ERO(p=1, style=gamma,$\eta$=0)} & 0.00$\pm$0.00 & \red{-1.83$\pm$0.00} & 0.00$\pm$0.00 & \blue{-0.09$\pm$0.01} & 0.00$\pm$0.00 & 0.00$\pm$0.00 & -0.01$\pm$0.00 \\
			{\it ERO(p=1, style=uniform,$\eta$=0.1)} & -0.02$\pm$0.01 & \red{-1.67$\pm$0.00} & 0.00$\pm$0.00 & -0.12$\pm$0.07 & -0.09$\pm$0.02 & \blue{-0.49$\pm$0.03} & -0.00$\pm$0.00 \\
			{\it ERO(p=1, style=gamma,$\eta$=0.1)} & -0.01$\pm$0.01 & \red{-1.58$\pm$0.00} & 0.00$\pm$0.00 & -0.03$\pm$0.01 & -0.25$\pm$0.02 & \blue{-0.51$\pm$0.02} & -0.00$\pm$0.00 \\
			{\it ERO(p=1, style=uniform,$\eta$=0.2)} & -0.01$\pm$0.01 & \red{-1.41$\pm$0.00} & -0.00$\pm$0.00 & -0.11$\pm$0.05 & -0.10$\pm$0.01 & \blue{-0.34$\pm$0.02} & -0.00$\pm$0.00 \\
			{\it ERO(p=1, style=gamma,$\eta$=0.2)} & -0.01$\pm$0.01 & \red{-1.45$\pm$0.00} & -0.00$\pm$0.00 & -0.01$\pm$0.02 & -0.15$\pm$0.02 & \blue{-0.32$\pm$0.03} & -0.00$\pm$0.00 \\
			{\it ERO(p=1, style=uniform,$\eta$=0.3)} & -0.00$\pm$0.00 & \red{-1.21$\pm$0.00} & 0.00$\pm$0.00 & -0.06$\pm$0.04 & -0.07$\pm$0.01 & \blue{-0.23$\pm$0.02} & -0.00$\pm$0.00 \\
			{\it ERO(p=1, style=gamma,$\eta$=0.3)} & -0.00$\pm$0.00 & \red{-1.25$\pm$0.00} & -0.00$\pm$0.00 & -0.00$\pm$0.01 & -0.12$\pm$0.03 & \blue{-0.21$\pm$0.03} & 0.00$\pm$0.00 \\
			{\it ERO(p=1, style=uniform,$\eta$=0.4)} & -0.00$\pm$0.00 & \red{-1.05$\pm$0.00} & -0.00$\pm$0.00 & -0.05$\pm$0.04 & -0.04$\pm$0.01 & \blue{-0.14$\pm$0.02} & -0.00$\pm$0.01 \\
			{\it ERO(p=1, style=gamma,$\eta$=0.4)} & -0.02$\pm$0.00 & \red{-1.07$\pm$0.00} & 0.00$\pm$0.00 & -0.04$\pm$0.00 & -0.15$\pm$0.00 & \blue{-0.21$\pm$0.00} & -0.00$\pm$0.00 \\
			{\it ERO(p=1, style=uniform,$\eta$=0.5)} & -0.13$\pm$0.00 & \red{-0.94$\pm$0.00} & -0.06$\pm$0.00 & -0.18$\pm$0.03 & -0.19$\pm$0.00 & \blue{-0.26$\pm$0.00} & -0.11$\pm$0.00 \\
			{\it ERO(p=1, style=gamma,$\eta$=0.5)} & -0.21$\pm$0.00 & \red{-0.98$\pm$0.00} & -0.07$\pm$0.00 & -0.23$\pm$0.01 & -0.30$\pm$0.00 & \blue{-0.33$\pm$0.00} & -0.19$\pm$0.00 \\
			{\it ERO(p=1, style=uniform,$\eta$=0.6)} & -0.32$\pm$0.00 & \red{-0.94$\pm$0.00} & -0.27$\pm$0.00 & -0.36$\pm$0.03 & -0.37$\pm$0.00 & \blue{-0.40$\pm$0.00} & -0.31$\pm$0.00 \\
			{\it ERO(p=1, style=gamma,$\eta$=0.6)} & -0.38$\pm$0.00 & \red{-0.93$\pm$0.00} & -0.27$\pm$0.00 & -0.42$\pm$0.01 & -0.44$\pm$0.00 & \blue{-0.46$\pm$0.00} & -0.37$\pm$0.00 \\
			{\it ERO(p=1, style=uniform,$\eta$=0.7)} & -0.51$\pm$0.00 & \red{-0.95$\pm$0.00} & -0.49$\pm$0.00 & -0.54$\pm$0.03 & -0.53$\pm$0.00 & \blue{-0.55$\pm$0.00} & -0.51$\pm$0.00 \\
			{\it ERO(p=1, style=gamma,$\eta$=0.7)} & -0.55$\pm$0.00 & \red{-0.98$\pm$0.00} & -0.52$\pm$0.00 & -0.59$\pm$0.01 & -0.59$\pm$0.00 & \blue{-0.60$\pm$0.00} & -0.56$\pm$0.00 \\
			{\it ERO(p=1, style=uniform,$\eta$=0.8)} & -0.68$\pm$0.00 & \red{-0.97$\pm$0.00} & \blue{-0.95$\pm$0.00} & -0.71$\pm$0.02 & -0.70$\pm$0.00 & -0.71$\pm$0.00 & -0.70$\pm$0.00 \\
			{\it ERO(p=1, style=gamma,$\eta$=0.8)} & -0.71$\pm$0.00 & \blue{-0.96$\pm$0.00} & \red{-0.98$\pm$0.00} & -0.75$\pm$0.01 & -0.73$\pm$0.00 & -0.73$\pm$0.00 & -0.73$\pm$0.00 \\
\bottomrule
\end{tabular}}
\end{table*}

\begin{table*}[!ht]
\centering
\caption{Result table on $\mathcal{L}_\text{upset, naive}$ improvement with ``proximal baseline" training starting from a baseline as initial guess, for individual directed graphs, averaged over 10 runs, and plus/minus one standard deviation. The best is marked in \red{bold red} while the second best is highlighted in \blue{underline blue}.}
\label{tab:upset_naive_improve_over_baseline}
\resizebox{0.5\linewidth}{!}{\begin{tabular}{lrrrrrrrrrrrrrrrrrr}
\toprule
Data  & SpringRank & SyncRank & SerialRank & BTL  & Eig.Cent. & PageRank  & SVD\_NRS  \\
\midrule
{\it HeadToHead} & -0.00$\pm$0.00 & \red{-0.24$\pm$0.00} & \blue{-0.18$\pm$0.00} & 0.00$\pm$0.00 & -0.05$\pm$0.00 & -0.02$\pm$0.00 & -0.04$\pm$0.00 \\
			{\it Finance} & -0.00$\pm$0.00 & \red{-0.09$\pm$0.00} & 0.00$\pm$0.00 & \blue{-0.03$\pm$0.00} & -0.02$\pm$0.00 & -0.02$\pm$0.00 & -0.00$\pm$0.00 \\
			{\it Animal} & -0.02$\pm$0.01 & \blue{-0.34$\pm$0.06} & \red{-0.44$\pm$0.12} & 0.00$\pm$0.00 & -0.01$\pm$0.00 & -0.02$\pm$0.02 & -0.00$\pm$0.02 \\
			{\it Faculty: Business} & -0.01$\pm$0.00 & \blue{-0.12$\pm$0.00} & \red{-0.17$\pm$0.01} & -0.00$\pm$0.00 & -0.01$\pm$0.01 & -0.01$\pm$0.01 & -0.01$\pm$0.00 \\
			{\it Faculty: CS} & -0.00$\pm$0.00 & \blue{-0.17$\pm$0.02} & \red{-0.18$\pm$0.02} & 0.00$\pm$0.00 & 0.00$\pm$0.01 & -0.00$\pm$0.00 & -0.02$\pm$0.01 \\
			{\it Faculty: History} & -0.00$\pm$0.00 & \blue{-0.07$\pm$0.00} & \red{-0.44$\pm$0.20} & 0.00$\pm$0.00 & -0.00$\pm$0.01 & -0.00$\pm$0.00 & -0.00$\pm$0.00 \\
			{\it Football(2009)} & 0.00$\pm$0.00 & \red{-0.18$\pm$0.11} & \blue{-0.04$\pm$0.02} & -0.01$\pm$0.02 & -0.03$\pm$0.01 & \blue{-0.04$\pm$0.02} & -0.00$\pm$0.00 \\
			{\it Football(2010)} & -0.01$\pm$0.01 & \red{-0.15$\pm$0.04} & \blue{-0.08$\pm$0.01} & -0.00$\pm$0.01 & -0.03$\pm$0.02 & -0.03$\pm$0.02 & -0.00$\pm$0.00 \\
			{\it Football(2011)} & -0.00$\pm$0.00 & \red{-0.19$\pm$0.09} & -0.01$\pm$0.01 & -0.01$\pm$0.01 & -0.01$\pm$0.01 & \blue{-0.03$\pm$0.02} & -0.00$\pm$0.01 \\
			{\it Football(2012)} & -0.01$\pm$0.01 & \red{-0.17$\pm$0.01} & \blue{-0.12$\pm$0.03} & -0.01$\pm$0.02 & -0.00$\pm$0.01 & -0.01$\pm$0.00 & 0.00$\pm$0.00 \\
			{\it Football(2013)} & -0.01$\pm$0.01 & \red{-0.27$\pm$0.03} & -0.02$\pm$0.01 & -0.02$\pm$0.01 & \blue{-0.04$\pm$0.01} & \blue{-0.04$\pm$0.01} & -0.01$\pm$0.01 \\
			{\it Football(2014)} & -0.00$\pm$0.01 & \red{-0.18$\pm$0.05} & \blue{-0.12$\pm$0.03} & -0.01$\pm$0.01 & -0.02$\pm$0.03 & -0.00$\pm$0.00 & -0.01$\pm$0.01 \\
			{\it Football finer(2009)} & -0.01$\pm$0.01 & \red{-0.26$\pm$0.01} & -0.04$\pm$0.01 & -0.01$\pm$0.03 & -0.04$\pm$0.03 & \blue{-0.06$\pm$0.03} & -0.01$\pm$0.01 \\
			{\it Football finer(2010)} & -0.02$\pm$0.01 & \red{-0.11$\pm$0.01} & -0.05$\pm$0.02 & -0.03$\pm$0.03 & -0.03$\pm$0.02 & \blue{-0.06$\pm$0.01} & -0.00$\pm$0.00 \\
			{\it Football finer(2011)} & -0.00$\pm$0.01 & \red{-0.19$\pm$0.01} & -0.04$\pm$0.01 & -0.01$\pm$0.01 & -0.02$\pm$0.01 & \blue{-0.05$\pm$0.01} & -0.00$\pm$0.01 \\
			{\it Football finer(2012)} & -0.01$\pm$0.01 & \red{-0.23$\pm$0.04} & -0.03$\pm$0.02 & -0.01$\pm$0.01 & -0.04$\pm$0.02 & \blue{-0.07$\pm$0.01} & -0.00$\pm$0.01 \\
			{\it Football finer(2013)} & -0.03$\pm$0.03 & \red{-0.25$\pm$0.00} & -0.03$\pm$0.02 & -0.02$\pm$0.02 & -0.05$\pm$0.03 & \blue{-0.06$\pm$0.02} & -0.02$\pm$0.01 \\
			{\it Football finer(2014)} & -0.00$\pm$0.01 & \red{-0.17$\pm$0.02} & \blue{-0.09$\pm$0.03} & -0.02$\pm$0.02 & -0.01$\pm$0.01 & -0.02$\pm$0.01 & -0.00$\pm$0.01 \\
			{\it Basketball(1985)} & -0.00$\pm$0.00 & \blue{-0.23$\pm$0.00} & \red{-0.32$\pm$0.01} & -0.03$\pm$0.02 & -0.01$\pm$0.00 & -0.02$\pm$0.01 & -0.02$\pm$0.00 \\
			{\it Basketball(1986)} & -0.00$\pm$0.00 & \red{-0.28$\pm$0.00} & \red{-0.28$\pm$0.01} & -0.03$\pm$0.01 & -0.02$\pm$0.00 & -0.01$\pm$0.00 & -0.01$\pm$0.00 \\
			{\it Basketball(1987)} & -0.00$\pm$0.00 & \red{-0.25$\pm$0.00} & \blue{-0.23$\pm$0.01} & -0.03$\pm$0.01 & -0.03$\pm$0.01 & -0.00$\pm$0.00 & -0.01$\pm$0.00 \\
			{\it Basketball(1988)} & 0.00$\pm$0.00 & \blue{-0.24$\pm$0.00} & \red{-0.27$\pm$0.01} & -0.03$\pm$0.01 & -0.01$\pm$0.00 & -0.01$\pm$0.00 & -0.01$\pm$0.00 \\
			{\it Basketball(1989)} & 0.00$\pm$0.00 & \red{-0.26$\pm$0.00} & \blue{-0.25$\pm$0.01} & -0.03$\pm$0.00 & 0.00$\pm$0.00 & -0.01$\pm$0.00 & -0.02$\pm$0.00 \\
			{\it Basketball(1990)} & -0.00$\pm$0.00 & \blue{-0.25$\pm$0.00} & \red{-0.27$\pm$0.01} & -0.03$\pm$0.03 & -0.02$\pm$0.00 & -0.02$\pm$0.00 & -0.01$\pm$0.00 \\
			{\it Basketball(1991)} & -0.00$\pm$0.00 & \blue{-0.27$\pm$0.00} & \red{-0.29$\pm$0.01} & -0.02$\pm$0.01 & -0.01$\pm$0.00 & -0.00$\pm$0.00 & -0.02$\pm$0.00 \\
			{\it Basketball(1992)} & -0.00$\pm$0.00 & \red{-0.27$\pm$0.00} & \blue{-0.26$\pm$0.00} & -0.01$\pm$0.03 & -0.01$\pm$0.00 & 0.00$\pm$0.00 & -0.02$\pm$0.00 \\
			{\it Basketball(1993)} & -0.00$\pm$0.00 & \blue{-0.25$\pm$0.00} & \red{-0.29$\pm$0.02} & -0.02$\pm$0.02 & -0.01$\pm$0.00 & -0.02$\pm$0.00 & -0.02$\pm$0.00 \\
			{\it Basketball(1994)} & 0.00$\pm$0.00 & \blue{-0.25$\pm$0.00} & \red{-0.29$\pm$0.00} & -0.01$\pm$0.02 & -0.01$\pm$0.00 & -0.01$\pm$0.00 & -0.01$\pm$0.00 \\
			{\it Basketball(1995)} & 0.00$\pm$0.00 & \red{-0.26$\pm$0.00} & \blue{-0.25$\pm$0.01} & -0.01$\pm$0.02 & -0.00$\pm$0.00 & -0.02$\pm$0.00 & -0.01$\pm$0.00 \\
			{\it Basketball(1996)} & -0.00$\pm$0.00 & \blue{-0.23$\pm$0.00} & \red{-0.28$\pm$0.01} & -0.02$\pm$0.02 & -0.01$\pm$0.00 & -0.01$\pm$0.00 & -0.02$\pm$0.00 \\
			{\it Basketball(1997)} & -0.00$\pm$0.00 & \blue{-0.24$\pm$0.00} & \red{-0.25$\pm$0.01} & -0.03$\pm$0.01 & -0.01$\pm$0.00 & -0.01$\pm$0.00 & -0.02$\pm$0.00 \\
			{\it Basketball(1998)} & -0.00$\pm$0.00 & \blue{-0.24$\pm$0.00} & \red{-0.25$\pm$0.01} & -0.03$\pm$0.01 & -0.02$\pm$0.00 & -0.01$\pm$0.00 & -0.02$\pm$0.00 \\
			{\it Basketball(1999)} & -0.00$\pm$0.00 & \blue{-0.21$\pm$0.00} & \red{-0.27$\pm$0.01} & -0.02$\pm$0.02 & -0.02$\pm$0.00 & -0.02$\pm$0.00 & -0.02$\pm$0.00 \\
			{\it Basketball(2000)} & -0.00$\pm$0.00 & \blue{-0.25$\pm$0.00} & \red{-0.26$\pm$0.01} & -0.02$\pm$0.01 & -0.01$\pm$0.00 & -0.01$\pm$0.00 & -0.01$\pm$0.00 \\
			{\it Basketball(2001)} & 0.00$\pm$0.00 & \blue{-0.26$\pm$0.00} & \red{-0.28$\pm$0.03} & -0.02$\pm$0.02 & -0.01$\pm$0.00 & -0.01$\pm$0.00 & -0.03$\pm$0.01 \\
			{\it Basketball(2002)} & -0.00$\pm$0.00 & \blue{-0.24$\pm$0.00} & \red{-0.26$\pm$0.00} & -0.02$\pm$0.01 & -0.00$\pm$0.00 & -0.02$\pm$0.00 & -0.01$\pm$0.00 \\
			{\it Basketball(2003)} & 0.00$\pm$0.00 & \blue{-0.27$\pm$0.00} & \red{-0.28$\pm$0.00} & -0.02$\pm$0.02 & -0.02$\pm$0.00 & -0.02$\pm$0.00 & -0.02$\pm$0.00 \\
			{\it Basketball(2004)} & -0.00$\pm$0.00 & \blue{-0.25$\pm$0.00} & \red{-0.28$\pm$0.03} & -0.02$\pm$0.01 & -0.01$\pm$0.00 & -0.02$\pm$0.00 & -0.02$\pm$0.00 \\
			{\it Basketball(2005)} & 0.00$\pm$0.00 & \red{-0.26$\pm$0.00} & \red{-0.26$\pm$0.02} & -0.04$\pm$0.01 & -0.00$\pm$0.00 & -0.01$\pm$0.00 & -0.02$\pm$0.00 \\
			{\it Basketball(2006)} & 0.00$\pm$0.00 & \blue{-0.26$\pm$0.00} & \red{-0.27$\pm$0.01} & 0.00$\pm$0.02 & 0.00$\pm$0.00 & -0.01$\pm$0.00 & -0.01$\pm$0.00 \\
			{\it Basketball(2007)} & -0.00$\pm$0.00 & \red{-0.28$\pm$0.00} & \blue{-0.26$\pm$0.02} & -0.02$\pm$0.01 & -0.01$\pm$0.00 & -0.01$\pm$0.00 & -0.01$\pm$0.00 \\
			{\it Basketball(2008)} & 0.00$\pm$0.00 & \blue{-0.24$\pm$0.00} & \red{-0.27$\pm$0.01} & -0.01$\pm$0.02 & -0.01$\pm$0.00 & -0.01$\pm$0.00 & -0.01$\pm$0.00 \\
			{\it Basketball(2009)} & -0.00$\pm$0.00 & \blue{-0.26$\pm$0.00} & \red{-0.29$\pm$0.01} & -0.03$\pm$0.01 & -0.01$\pm$0.00 & -0.01$\pm$0.00 & -0.01$\pm$0.00 \\
			{\it Basketball(2010)} & -0.00$\pm$0.00 & \blue{-0.21$\pm$0.00} & \red{-0.26$\pm$0.01} & -0.01$\pm$0.03 & -0.01$\pm$0.00 & -0.01$\pm$0.00 & -0.01$\pm$0.00 \\
			{\it Basketball(2011)} & -0.00$\pm$0.00 & \blue{-0.22$\pm$0.00} & \red{-0.27$\pm$0.01} & 0.00$\pm$0.02 & -0.01$\pm$0.00 & -0.01$\pm$0.00 & -0.00$\pm$0.00 \\
			{\it Basketball(2012)} & -0.00$\pm$0.00 & \red{-0.28$\pm$0.00} & \red{-0.28$\pm$0.01} & -0.02$\pm$0.02 & -0.01$\pm$0.00 & -0.01$\pm$0.00 & -0.01$\pm$0.00 \\
			{\it Basketball(2013)} & -0.00$\pm$0.00 & \blue{-0.24$\pm$0.00} & \red{-0.27$\pm$0.02} & -0.03$\pm$0.01 & -0.01$\pm$0.00 & -0.01$\pm$0.00 & -0.02$\pm$0.00 \\
			{\it Basketball(2014)} & -0.00$\pm$0.00 & \blue{-0.24$\pm$0.00} & \red{-0.28$\pm$0.01} & -0.00$\pm$0.03 & -0.00$\pm$0.00 & -0.01$\pm$0.00 & -0.00$\pm$0.00 \\
			{\it Basketball finer(1985)} & -0.00$\pm$0.00 & \blue{-0.23$\pm$0.00} & \red{-0.24$\pm$0.02} & -0.14$\pm$0.02 & -0.09$\pm$0.00 & -0.08$\pm$0.00 & -0.02$\pm$0.00 \\
			{\it Basketball finer(1986)} & 0.00$\pm$0.00 & \red{-0.28$\pm$0.00} & \blue{-0.25$\pm$0.01} & -0.13$\pm$0.02 & -0.08$\pm$0.00 & -0.07$\pm$0.00 & -0.01$\pm$0.00 \\
			{\it Basketball finer(1987)} & -0.00$\pm$0.00 & \red{-0.26$\pm$0.00} & \blue{-0.23$\pm$0.00} & -0.11$\pm$0.02 & -0.06$\pm$0.00 & -0.07$\pm$0.00 & -0.01$\pm$0.00 \\
			{\it Basketball finer(1988)} & -0.00$\pm$0.00 & \red{-0.27$\pm$0.00} & \blue{-0.25$\pm$0.00} & -0.14$\pm$0.03 & -0.09$\pm$0.00 & -0.09$\pm$0.00 & -0.01$\pm$0.00 \\
			{\it Basketball finer(1989)} & 0.00$\pm$0.00 & \red{-0.24$\pm$0.00} & \blue{-0.22$\pm$0.01} & -0.13$\pm$0.01 & -0.07$\pm$0.00 & -0.07$\pm$0.00 & -0.02$\pm$0.00 \\
			{\it Basketball finer(1990)} & -0.00$\pm$0.00 & \blue{-0.24$\pm$0.00} & \red{-0.25$\pm$0.00} & -0.14$\pm$0.01 & -0.10$\pm$0.00 & -0.08$\pm$0.00 & -0.01$\pm$0.00 \\
			{\it Basketball finer(1991)} & -0.00$\pm$0.00 & \red{-0.28$\pm$0.00} & \blue{-0.26$\pm$0.01} & -0.11$\pm$0.02 & -0.12$\pm$0.00 & -0.10$\pm$0.00 & -0.01$\pm$0.00 \\
			{\it Basketball finer(1992)} & -0.00$\pm$0.00 & \red{-0.26$\pm$0.00} & \blue{-0.25$\pm$0.00} & -0.12$\pm$0.01 & -0.10$\pm$0.00 & -0.09$\pm$0.00 & -0.02$\pm$0.00 \\
			{\it Basketball finer(1993)} & -0.00$\pm$0.00 & \blue{-0.24$\pm$0.00} & \red{-0.27$\pm$0.02} & -0.12$\pm$0.01 & -0.09$\pm$0.00 & -0.07$\pm$0.00 & -0.01$\pm$0.00 \\
			{\it Basketball finer(1994)} & 0.00$\pm$0.00 & \blue{-0.25$\pm$0.00} & \red{-0.28$\pm$0.01} & -0.12$\pm$0.02 & -0.10$\pm$0.00 & -0.07$\pm$0.00 & -0.01$\pm$0.00 \\
			{\it Basketball finer(1995)} & -0.00$\pm$0.00 & \red{-0.26$\pm$0.00} & \blue{-0.25$\pm$0.00} & -0.12$\pm$0.02 & -0.09$\pm$0.00 & -0.07$\pm$0.00 & -0.02$\pm$0.00 \\
			{\it Basketball finer(1996)} & -0.00$\pm$0.00 & \red{-0.22$\pm$0.00} & \red{-0.22$\pm$0.01} & -0.11$\pm$0.01 & -0.07$\pm$0.01 & -0.06$\pm$0.00 & -0.01$\pm$0.00 \\
			{\it Basketball finer(1997)} & -0.00$\pm$0.00 & \red{-0.26$\pm$0.00} & \blue{-0.24$\pm$0.00} & -0.12$\pm$0.01 & -0.07$\pm$0.00 & -0.06$\pm$0.00 & -0.02$\pm$0.00 \\
			{\it Basketball finer(1998)} & -0.00$\pm$0.00 & \red{-0.24$\pm$0.00} & \red{-0.24$\pm$0.00} & -0.11$\pm$0.02 & -0.06$\pm$0.00 & -0.06$\pm$0.00 & -0.02$\pm$0.00 \\
			{\it Basketball finer(1999)} & -0.00$\pm$0.00 & \blue{-0.23$\pm$0.00} & \red{-0.26$\pm$0.01} & -0.11$\pm$0.02 & -0.07$\pm$0.00 & -0.05$\pm$0.00 & -0.02$\pm$0.00 \\
			{\it Basketball finer(2000)} & 0.00$\pm$0.00 & \blue{-0.24$\pm$0.00} & \red{-0.25$\pm$0.00} & -0.10$\pm$0.01 & -0.08$\pm$0.00 & -0.06$\pm$0.00 & -0.01$\pm$0.00 \\
			{\it Basketball finer(2001)} & -0.00$\pm$0.00 & \blue{-0.24$\pm$0.00} & \red{-0.25$\pm$0.01} & -0.11$\pm$0.01 & -0.08$\pm$0.00 & -0.06$\pm$0.00 & -0.02$\pm$0.00 \\
			{\it Basketball finer(2002)} & -0.01$\pm$0.00 & \red{-0.25$\pm$0.00} & \blue{-0.23$\pm$0.01} & -0.13$\pm$0.02 & -0.08$\pm$0.00 & -0.06$\pm$0.00 & -0.01$\pm$0.00 \\
			{\it Basketball finer(2003)} & 0.00$\pm$0.00 & \red{-0.25$\pm$0.00} & \blue{-0.22$\pm$0.02} & -0.12$\pm$0.02 & -0.06$\pm$0.00 & -0.05$\pm$0.00 & -0.01$\pm$0.00 \\
			{\it Basketball finer(2004)} & -0.00$\pm$0.00 & \red{-0.25$\pm$0.01} & \blue{-0.23$\pm$0.00} & -0.12$\pm$0.02 & -0.07$\pm$0.00 & -0.06$\pm$0.00 & -0.02$\pm$0.00 \\
			{\it Basketball finer(2005)} & -0.00$\pm$0.00 & \red{-0.27$\pm$0.00} & \blue{-0.22$\pm$0.00} & -0.10$\pm$0.01 & -0.05$\pm$0.00 & -0.04$\pm$0.00 & -0.02$\pm$0.00 \\
			{\it Basketball finer(2006)} & -0.00$\pm$0.00 & \red{-0.25$\pm$0.00} & \blue{-0.24$\pm$0.00} & -0.10$\pm$0.02 & -0.06$\pm$0.00 & -0.04$\pm$0.00 & -0.01$\pm$0.00 \\
			{\it Basketball finer(2007)} & -0.00$\pm$0.00 & \red{-0.26$\pm$0.00} & \blue{-0.23$\pm$0.03} & -0.12$\pm$0.01 & -0.06$\pm$0.00 & -0.04$\pm$0.00 & -0.01$\pm$0.00 \\
			{\it Basketball finer(2008)} & -0.00$\pm$0.00 & \blue{-0.23$\pm$0.00} & \red{-0.24$\pm$0.01} & -0.11$\pm$0.02 & -0.07$\pm$0.00 & -0.05$\pm$0.00 & -0.01$\pm$0.00 \\
			{\it Basketball finer(2009)} & -0.00$\pm$0.00 & \blue{-0.24$\pm$0.00} & \red{-0.25$\pm$0.01} & -0.11$\pm$0.01 & -0.07$\pm$0.00 & -0.05$\pm$0.00 & -0.01$\pm$0.00 \\
			{\it Basketball finer(2010)} & -0.00$\pm$0.00 & \blue{-0.23$\pm$0.00} & \red{-0.25$\pm$0.00} & -0.10$\pm$0.02 & -0.05$\pm$0.00 & -0.05$\pm$0.00 & -0.01$\pm$0.00 \\
			{\it Basketball finer(2011)} & -0.00$\pm$0.00 & \blue{-0.22$\pm$0.00} & \red{-0.25$\pm$0.00} & -0.11$\pm$0.01 & -0.05$\pm$0.00 & -0.06$\pm$0.00 & -0.01$\pm$0.00 \\
			{\it Basketball finer(2012)} & 0.00$\pm$0.00 & \red{-0.24$\pm$0.00} & \red{-0.24$\pm$0.01} & -0.11$\pm$0.02 & -0.07$\pm$0.00 & -0.05$\pm$0.00 & -0.01$\pm$0.00 \\
			{\it Basketball finer(2013)} & -0.00$\pm$0.00 & \blue{-0.24$\pm$0.00} & \red{-0.26$\pm$0.02} & -0.11$\pm$0.02 & -0.07$\pm$0.00 & -0.06$\pm$0.00 & -0.02$\pm$0.00 \\
			{\it Basketball finer(2014)} & -0.00$\pm$0.00 & \blue{-0.25$\pm$0.05} & \red{-0.26$\pm$0.01} & -0.12$\pm$0.02 & -0.06$\pm$0.01 & -0.05$\pm$0.00 & -0.01$\pm$0.00 \\
			{\it ERO(p=0.05, style=uniform,$\eta$=0)} & -0.01$\pm$0.01 & \red{-0.43$\pm$0.00} & \blue{-0.40$\pm$0.01} & -0.04$\pm$0.03 & 0.02$\pm$0.00 & -0.00$\pm$0.00 & -0.01$\pm$0.00 \\
			{\it ERO(p=0.05, style=gamma,$\eta$=0)} & -0.03$\pm$0.00 & \red{-0.44$\pm$0.01} & \blue{-0.40$\pm$0.10} & 0.06$\pm$0.01 & 0.08$\pm$0.01 & 0.05$\pm$0.01 & 0.02$\pm$0.00 \\
			{\it ERO(p=0.05, style=uniform,$\eta$=0.1)} & -0.00$\pm$0.00 & \red{-0.37$\pm$0.00} & \blue{-0.32$\pm$0.00} & -0.01$\pm$0.01 & -0.06$\pm$0.00 & -0.06$\pm$0.00 & -0.00$\pm$0.00 \\
			{\it ERO(p=0.05, style=gamma,$\eta$=0.1)} & -0.00$\pm$0.00 & \red{-0.40$\pm$0.00} & \blue{-0.26$\pm$0.02} & 0.01$\pm$0.00 & -0.01$\pm$0.00 & 0.01$\pm$0.00 & 0.01$\pm$0.00 \\
			{\it ERO(p=0.05, style=uniform,$\eta$=0.2)} & -0.00$\pm$0.01 & \red{-0.34$\pm$0.00} & \blue{-0.30$\pm$0.01} & -0.00$\pm$0.01 & -0.05$\pm$0.01 & -0.08$\pm$0.01 & -0.00$\pm$0.01 \\
			{\it ERO(p=0.05, style=gamma,$\eta$=0.2)} & -0.00$\pm$0.00 & \red{-0.36$\pm$0.00} & \blue{-0.21$\pm$0.01} & 0.00$\pm$0.00 & -0.02$\pm$0.00 & -0.02$\pm$0.00 & -0.00$\pm$0.00 \\
			{\it ERO(p=0.05, style=uniform,$\eta$=0.3)} & -0.00$\pm$0.00 & \red{-0.28$\pm$0.01} & \blue{-0.25$\pm$0.03} & 0.01$\pm$0.01 & -0.03$\pm$0.01 & -0.05$\pm$0.02 & -0.00$\pm$0.01 \\
			{\it ERO(p=0.05, style=gamma,$\eta$=0.3)} & -0.01$\pm$0.00 & \red{-0.29$\pm$0.00} & \blue{-0.16$\pm$0.02} & -0.00$\pm$0.00 & -0.01$\pm$0.01 & -0.01$\pm$0.00 & -0.00$\pm$0.00 \\
			{\it ERO(p=0.05, style=uniform,$\eta$=0.4)} & -0.00$\pm$0.00 & \red{-0.25$\pm$0.01} & \blue{-0.18$\pm$0.01} & 0.01$\pm$0.01 & -0.00$\pm$0.00 & -0.02$\pm$0.01 & 0.00$\pm$0.00 \\
			{\it ERO(p=0.05, style=gamma,$\eta$=0.4)} & -0.00$\pm$0.01 & \red{-0.24$\pm$0.01} & \blue{-0.11$\pm$0.03} & 0.00$\pm$0.01 & -0.01$\pm$0.00 & -0.01$\pm$0.00 & -0.00$\pm$0.00 \\
			{\it ERO(p=0.05, style=uniform,$\eta$=0.5)} & -0.00$\pm$0.00 & \red{-0.21$\pm$0.00} & \blue{-0.16$\pm$0.00} & 0.01$\pm$0.04 & -0.00$\pm$0.01 & -0.01$\pm$0.01 & -0.00$\pm$0.01 \\
			{\it ERO(p=0.05, style=gamma,$\eta$=0.5)} & -0.01$\pm$0.01 & \red{-0.21$\pm$0.01} & \blue{-0.06$\pm$0.04} & -0.01$\pm$0.01 & -0.01$\pm$0.00 & -0.00$\pm$0.00 & -0.03$\pm$0.03 \\
			{\it ERO(p=0.05, style=uniform,$\eta$=0.6)} & -0.00$\pm$0.00 & \red{-0.17$\pm$0.00} & \blue{-0.14$\pm$0.00} & 0.00$\pm$0.01 & -0.00$\pm$0.00 & -0.01$\pm$0.01 & -0.04$\pm$0.01 \\
			{\it ERO(p=0.05, style=gamma,$\eta$=0.6)} & -0.01$\pm$0.01 & \red{-0.18$\pm$0.01} & \blue{-0.04$\pm$0.00} & -0.01$\pm$0.01 & -0.01$\pm$0.01 & -0.01$\pm$0.01 & -0.01$\pm$0.01 \\
			{\it ERO(p=0.05, style=uniform,$\eta$=0.7)} & -0.00$\pm$0.01 & \red{-0.16$\pm$0.00} & \blue{-0.09$\pm$0.00} & -0.00$\pm$0.01 & 0.00$\pm$0.00 & -0.00$\pm$0.00 & -0.06$\pm$0.00 \\
			{\it ERO(p=0.05, style=gamma,$\eta$=0.7)} & -0.00$\pm$0.01 & \red{-0.16$\pm$0.01} & \blue{-0.05$\pm$0.01} & -0.01$\pm$0.01 & -0.01$\pm$0.01 & 0.00$\pm$0.00 & -0.02$\pm$0.01 \\
			{\it ERO(p=0.05, style=uniform,$\eta$=0.8)} & -0.00$\pm$0.00 & \red{-0.14$\pm$0.00} & \blue{-0.07$\pm$0.00} & 0.00$\pm$0.01 & -0.00$\pm$0.00 & 0.00$\pm$0.00 & \blue{-0.07$\pm$0.00} \\
			{\it ERO(p=0.05, style=gamma,$\eta$=0.8)} & -0.01$\pm$0.01 & \red{-0.14$\pm$0.01} & \blue{-0.04$\pm$0.01} & -0.01$\pm$0.01 & -0.01$\pm$0.01 & 0.00$\pm$0.01 & -0.03$\pm$0.02 \\
			{\it ERO(p=1, style=uniform,$\eta$=0)} & 0.00$\pm$0.00 & \red{-0.46$\pm$0.00} & 0.00$\pm$0.00 & \blue{-0.06$\pm$0.03} & 0.00$\pm$0.00 & 0.00$\pm$0.00 & 0.00$\pm$0.00 \\
			{\it ERO(p=1, style=gamma,$\eta$=0)} & 0.00$\pm$0.00 & \red{-0.46$\pm$0.00} & 0.00$\pm$0.00 & \blue{-0.02$\pm$0.00} & 0.00$\pm$0.00 & 0.00$\pm$0.00 & -0.00$\pm$0.00 \\
			{\it ERO(p=1, style=uniform,$\eta$=0.1)} & -0.00$\pm$0.00 & \red{-0.42$\pm$0.00} & 0.00$\pm$0.00 & -0.03$\pm$0.02 & -0.02$\pm$0.01 & \blue{-0.12$\pm$0.01} & -0.00$\pm$0.00 \\
			{\it ERO(p=1, style=gamma,$\eta$=0.1)} & -0.00$\pm$0.00 & \red{-0.39$\pm$0.00} & 0.00$\pm$0.00 & -0.01$\pm$0.00 & -0.05$\pm$0.02 & \blue{-0.11$\pm$0.01} & -0.00$\pm$0.00 \\
			{\it ERO(p=1, style=uniform,$\eta$=0.2)} & -0.00$\pm$0.00 & \red{-0.35$\pm$0.00} & -0.00$\pm$0.00 & -0.03$\pm$0.01 & -0.03$\pm$0.00 & \blue{-0.09$\pm$0.00} & -0.00$\pm$0.00 \\
			{\it ERO(p=1, style=gamma,$\eta$=0.2)} & -0.00$\pm$0.00 & \red{-0.36$\pm$0.00} & -0.00$\pm$0.00 & -0.00$\pm$0.00 & -0.04$\pm$0.01 & \blue{-0.08$\pm$0.01} & -0.00$\pm$0.00 \\
			{\it ERO(p=1, style=uniform,$\eta$=0.3)} & -0.00$\pm$0.00 & \red{-0.30$\pm$0.00} & 0.00$\pm$0.00 & -0.02$\pm$0.01 & -0.02$\pm$0.00 & \blue{-0.06$\pm$0.00} & -0.00$\pm$0.00 \\
			{\it ERO(p=1, style=gamma,$\eta$=0.3)} & -0.00$\pm$0.00 & \red{-0.31$\pm$0.00} & -0.00$\pm$0.00 & -0.00$\pm$0.00 & -0.03$\pm$0.01 & \blue{-0.05$\pm$0.01} & 0.00$\pm$0.00 \\
			{\it ERO(p=1, style=uniform,$\eta$=0.4)} & -0.00$\pm$0.00 & \red{-0.26$\pm$0.00} & -0.00$\pm$0.00 & -0.01$\pm$0.01 & -0.01$\pm$0.00 & \blue{-0.04$\pm$0.00} & -0.00$\pm$0.00 \\
			{\it ERO(p=1, style=gamma,$\eta$=0.4)} & -0.00$\pm$0.00 & \red{-0.27$\pm$0.00} & 0.00$\pm$0.00 & -0.00$\pm$0.00 & -0.02$\pm$0.01 & \blue{-0.03$\pm$0.01} & 0.00$\pm$0.00 \\
			{\it ERO(p=1, style=uniform,$\eta$=0.5)} & -0.00$\pm$0.00 & \red{-0.22$\pm$0.00} & -0.00$\pm$0.00 & -0.01$\pm$0.01 & -0.01$\pm$0.00 & \blue{-0.03$\pm$0.00} & -0.00$\pm$0.00 \\
			{\it ERO(p=1, style=gamma,$\eta$=0.5)} & -0.00$\pm$0.00 & \red{-0.22$\pm$0.00} & -0.00$\pm$0.00 & 0.00$\pm$0.01 & \blue{-0.01$\pm$0.01} & \blue{-0.01$\pm$0.01} & 0.00$\pm$0.00 \\
			{\it ERO(p=1, style=uniform,$\eta$=0.6)} & -0.00$\pm$0.00 & \red{-0.17$\pm$0.00} & -0.00$\pm$0.00 & -0.01$\pm$0.01 & -0.01$\pm$0.00 & \blue{-0.02$\pm$0.00} & -0.00$\pm$0.00 \\
			{\it ERO(p=1, style=gamma,$\eta$=0.6)} & -0.00$\pm$0.00 & \red{-0.16$\pm$0.00} & -0.00$\pm$0.00 & -0.00$\pm$0.00 & \blue{-0.01$\pm$0.00} & \blue{-0.01$\pm$0.00} & -0.00$\pm$0.00 \\
			{\it ERO(p=1, style=uniform,$\eta$=0.7)} & \blue{-0.01$\pm$0.00} & \red{-0.12$\pm$0.00} & -0.00$\pm$0.00 & -0.00$\pm$0.01 & -0.00$\pm$0.00 & -0.00$\pm$0.00 & -0.00$\pm$0.00 \\
			{\it ERO(p=1, style=gamma,$\eta$=0.7)} & \blue{-0.01$\pm$0.00} & \red{-0.13$\pm$0.00} & 0.00$\pm$0.00 & -0.00$\pm$0.00 & -0.00$\pm$0.00 & -0.00$\pm$0.00 & -0.00$\pm$0.00 \\
			{\it ERO(p=1, style=uniform,$\eta$=0.8)} & -0.00$\pm$0.00 & \red{-0.08$\pm$0.00} & \blue{-0.05$\pm$0.00} & -0.00$\pm$0.01 & -0.00$\pm$0.00 & -0.00$\pm$0.00 & -0.00$\pm$0.00 \\
			{\it ERO(p=1, style=gamma,$\eta$=0.8)} & -0.00$\pm$0.00 & \red{-0.08$\pm$0.00} & \blue{-0.05$\pm$0.00} & -0.00$\pm$0.00 & -0.00$\pm$0.00 & 0.00$\pm$0.00 & -0.00$\pm$0.00 \\
\bottomrule
\end{tabular}}
\end{table*}

\section{Variant and Hyperparameter Selection}
\label{appendix_sec:variant_hyper}
The results reported in the main text are selected within either non-proximal or proximal categories depending on whether they have proximal gradient steps within the architecture. All selections are carried out via the lowest $\mathcal{L}_\text{upset, simple}$ or the lowest $\mathcal{L}_\text{upset, naive}$ or the lowest $\mathcal{L}_\text{upset, ratio}$ depending on our objective. ``--" in the tables in this section means ``not applicable". There are a total of two variants for non-proximal variants and three for proximal variants for score generation, and each can be coupled with one of DIMPA \cite{he2021digrac} and the inception block model (ib) \cite{tong2020digraph} (or others like MagNet \cite{zhang2021magnet}, not tested here) for digraph embedding learning. We fix the learning rate to be 0.01, and vary the the method of pretraining, the coefficient for $\mathcal{L}_\text{upset, margin}$ in training, the coefficient for $\mathcal{L}_\text{upset, ratio}$ in training, 
and the baseline selected for ``proximal baseline" variant. Choices are SyncRank, SpringRank, SerialRank,  BTL, Eig.Cent., PageRank and SVD\_NRS for real-world data sets, and SpringRank, BTL, SerialRank for synthetic data. We consider  these choices to demonstrate that our ``proximal baseline" has the ability to improve over initial guess vectors coming from different types of baselines. See  Sec.~\ref{appendix_sec:improvement_over_baselines} for more details on improvements over baselines. 
\begin{table*}[!ht]
\centering
\caption{GNN selection among GNNRank-N methods for the lowest $\mathcal{L}_\text{upset, simple}$.}
\label{tab:GNN_selection_non_proximal_upset_simple}
\resizebox{0.85\linewidth}{!}{
\begin{tabular}{lrrrrrr}
\toprule
Data  & Variant &train with & pretrain with & $\mathcal{L}_\text{upset,margin}$ coefficient & baseline&$\mathcal{L}_\text{upset,ratio}$ coefficient\\
\midrule
{\it HeadToHead} & ib dist & proximal baseline & innerproduct & 1 & SVD\_NRS & 1 \\
			{\it Finance} & ib dist & innerproduct & -- & 0 & -- & 1 \\
			{\it Animal} & DIMPA dist & dist & -- & 0 & -- & 1 \\
			{\it Faculty: Business} & DIMPA innerproduct & innerproduct & -- & 0 & -- & 1 \\
			{\it Faculty: CS} & DIMPA dist & dist & -- & 0 & -- & 1 \\
			{\it Faculty: History} & DIMPA innerproduct & innerproduct & -- & 0 & -- & 1 \\
			{\it Basketball(1985)} & DIMPA dist & dist & -- & 0 & -- & 1 \\
\bottomrule
\end{tabular}}
\end{table*}

\begin{table*}[!ht]
\centering
\caption{GNN selection among GNNRank-P methods for the lowest $\mathcal{L}_\text{upset, simple}$.}
\label{tab:GNN_selection_proximal_upset_simple}
\resizebox{0.85\linewidth}{!}{\begin{tabular}{lrrrrrrr}
\toprule
Data  & Variant & train with & pretrain with & $\mathcal{L}_\text{upset,margin}$ coefficient & baseline&$\mathcal{L}_\text{upset,ratio}$ coefficient\\
\midrule
{\it HeadToHead} & ib proximal baseline & proximal baseline & SerialRank similarity & 1 & SyncRank & 1 \\
			{\it Finance} & ib proximal dist & innerproduct & -- & 0 & -- & 1 \\
			{\it Animal} & ib proximal baseline & proximal baseline & innerproduct & 1 & SyncRank & 0 \\
			{\it Faculty: Business} & ib proximal baseline & proximal baseline & innerproduct & 0 & SyncRank & 1 \\
			{\it Faculty: CS} & ib proximal baseline & proximal baseline & dist & 0 & SyncRank & 1 \\
			{\it Faculty: History} & DIMPA proximal dist & dist & -- & 1 & -- & 1 \\
			{\it Basketball(1985)} & ib proximal baseline & proximal baseline & dist & 0 & SyncRank & 1 \\
\bottomrule
    \end{tabular}}
\end{table*}

\begin{table*}[!ht]
\centering
\caption{GNN selection among GNNRank-N methods for the lowest $\mathcal{L}_\text{upset, naive}$.}
\label{tab:GNN_selection_non_proximal_upset_naive}
\resizebox{0.63\linewidth}{!}{\begin{tabular}{lrrrrrr}
\toprule
Data  & Variant  & train with & pretrain with & $\mathcal{L}_\text{upset,margin}$ coefficient & baseline&$\mathcal{L}_\text{upset,ratio}$ coefficient\\
\midrule
{\it HeadToHead} & DIMPA dist & dist & -- & 0 & -- & 1 \\
			{\it Finance} & DIMPA innerproduct & innerproduct & -- & 1 & -- & 0 \\
			{\it Animal} & DIMPA dist & dist & -- & 0 & -- & 1 \\
			{\it Faculty: Business} & DIMPA innerproduct & innerproduct & -- & 0 & -- & 1 \\
			{\it Faculty: CS} & DIMPA dist & dist & -- & 0 & -- & 1 \\
			{\it Faculty: History} & DIMPA innerproduct & innerproduct & -- & 0 & -- & 1 \\
			{\it Basketball(1985)} & DIMPA dist & dist & -- & 0 & -- & 1 \\
\bottomrule
\end{tabular}}
\end{table*}

\begin{table*}[!ht]
\centering
\caption{GNN selection among GNNRank-P methods for the lowest $\mathcal{L}_\text{upset, naive}$.}
\label{tab:GNN_selection_proximal_upset_naive}
\resizebox{0.65\linewidth}{!}{\begin{tabular}{lrrrrrr}
\toprule
Data  & Variant  & train with & pretrain with & $\mathcal{L}_\text{upset,margin}$ coefficient & baseline&$\mathcal{L}_\text{upset,ratio}$ coefficient\\
\midrule
{\it HeadToHead} & ib proximal baseline & proximal baseline & SerialRank similarity & 1 & SyncRank & 1 \\
			{\it Finance} & ib proximal baseline & proximal baseline & SerialRank similarity & 0 & SerialRank & 1 \\
			{\it Animal} & ib proximal baseline & proximal baseline & innerproduct & 1 & SyncRank & 0 \\
			{\it Faculty: Business} & ib proximal baseline & proximal baseline & innerproduct & 0 & SyncRank & 1 \\
			{\it Faculty: CS} & ib proximal baseline & proximal baseline & dist & 0 & SyncRank & 1 \\
			{\it Faculty: History} & DIMPA proximal dist & dist & -- & 1 & -- & 1 \\
			{\it Basketball(1985)} & ib proximal baseline & proximal baseline & dist & 0 & SyncRank & 1 \\
\bottomrule
\end{tabular}}
\end{table*}

\begin{table*}[!ht]
\centering
\caption{GNN selection among GNNRank-N methods for the lowest $\mathcal{L}_\text{upset, ratio}$.}
\label{tab:GNN_selection_non_proximal_upset_ratio}
\resizebox{0.75\linewidth}{!}{\begin{tabular}{lrrrrrr}
\toprule
Data  & Variant& train with & pretrain with & $\mathcal{L}_\text{upset,margin}$ coefficient   & baseline&$\mathcal{L}_\text{upset,ratio}$ coefficient\\
\midrule
{\it HeadToHead} & DIMPA dist & dist & -- & 0 & -- & 1 \\
			{\it Finance} & DIMPA innerproduct & innerproduct & -- & 0 & -- & 1 \\
			{\it Animal} & DIMPA dist & dist & -- & 0 & -- & 1 \\
			{\it Faculty: Business} & DIMPA innerproduct & innerproduct & -- & 0 & -- & 1 \\
			{\it Faculty: CS} & DIMPA dist & dist & -- & 0 & -- & 1 \\
			{\it Faculty: History} & DIMPA innerproduct & innerproduct & -- & 0 & -- & 1 \\
			{\it Basketball(1985)} & DIMPA dist & dist & -- & 0 & -- & 1 \\
\bottomrule
\end{tabular}}
\end{table*}

\begin{table*}[!ht]
\centering
\caption{GNN selection among GNNRank-P methods for the lowest $\mathcal{L}_\text{upset, ratio}$.}
\label{tab:GNN_selection_proximal_upset_ratio}
\resizebox{0.85\linewidth}{!}{\begin{tabular}{lrrrrrrr}
\toprule
Data  & Variant & learning rate & train with & pretrain with & $\mathcal{L}_\text{upset,margin}$ coefficient & baseline&$\mathcal{L}_\text{upset,ratio}$ coefficient\\
\midrule
{\it HeadToHead} & DIMPA proximal baseline & proximal baseline & dist & 1 & SyncRank & 1 \\
			{\it Finance} & DIMPA proximal baseline & proximal innerproduct & dist & 1 & SpringRank & 1 \\
			{\it Animal} & ib proximal baseline & proximal baseline & SerialRank similarity & 0 & SpringRank & 1 \\
			{\it Faculty: Business} & DIMPA proximal dist & proximal baseline & SerialRank similarity & 0 & SerialRank & 1 \\
			{\it Faculty: CS} & DIMPA proximal baseline & proximal baseline & dist & 1 & PageRank & 0 \\
			{\it Faculty: History} & DIMPA proximal baseline & proximal baseline & dist & 1 & BTL & 0 \\
			{\it Basketball(1985)} & ib proximal baseline & proximal baseline & SerialRank similarity & 0 & SyncRank & 1 \\
\bottomrule
\end{tabular}}
\end{table*}

\end{document}